\documentclass[11pt,letterpaper]{article} 
\usepackage{amsmath,amsthm,amssymb,graphicx,mathtools, algorithmicx,setspace,mdframed,verbatim, dsfont,pifont}
\usepackage{algpseudocode}
\usepackage{algorithm}

\usepackage{tikz}
\usepackage{caption}
\usepackage{subcaption}

\usepackage[square,numbers]{natbib}
\usepackage{geometry}
\usepackage{bold-extra}
\usepackage{nicefrac}
\allowdisplaybreaks[4]
\usepackage[skins]{tcolorbox}

\usepackage[colorlinks=true, linkcolor=black, urlcolor=black, citecolor=blue]{hyperref}

\usetikzlibrary{positioning}

\newenvironment{Msg}[1]
  {\mdfsetup{
    frametitle={\colorbox{white}{\space \large #1\space}},
    innertopmargin=-3pt,
    innerbottommargin=7pt,
    innerrightmargin=7pt,
    innerleftmargin=7pt,
    frametitleaboveskip=-\ht\strutbox,
    frametitlealignment=\center,
    linewidth=1pt
    }
  \begin{mdframed}%
  }
{\end{mdframed}}

\newcommand{\R}{\mathbb{R}}
\newcommand{\E}{\mathbb{E}}
\newcommand{\N}{\mathcal{N}}

\newcommand{\vbrack}[1]{\langle #1\rangle}

\renewcommand{\S}{\mathcal{S}}
\newcommand{\D}{\mathcal{D}}

\newcommand{\B}{\mathcal{B}}

\newcommand{\Ecal}{\mathcal{E}}

\newcommand{\F}{\mathfrak{F}}
\newcommand{\1}{\mathds{1}}
\newcommand{\sign}{\mathrm{sign}}
\newcommand{\DD}{\mathbf{D}}
\newcommand{\Id}{\mathbf{I}}

\newcommand{\M}{\mathbf{M}}
\newcommand{\Mperp}{\mathbf{M}^{\perp}}

\newcommand{\poly}{\mathsf{poly}}
\newcommand{\polylog}{\mathsf{polylog}}

\newcommand{\Mcal}{\mathcal{M}}

\newcommand{\pos}{\mathsf{pos}}

\newcommand{\Nfr}{\mathfrak{N}}
\newcommand{\Bfr}{\mathfrak{B}}

\newcommand{\Sim}{\mathsf{Sim}}

\newcommand{\ReLU}{\mathsf{ReLU}}

\newcommand{\Obj}{\mathbf{Obj}}

\newcommand{\diag}{\mathbf{diag}}
\renewcommand{\Pr}{\mathbf{Pr}}

\newcommand{\StopGrad}{\mathsf{StopGrad}}

\newcommand{\RandomMask}{\mathsf{RandomMask}}

\newcommand{\myref}[2]{\hyperref[#1]{#2 \ref*{#1}}}

\algnewcommand{\Input}{\textbf{Input:} }

\newcounter{main}
\numberwithin{main}{section}
\newtheorem{theorem}[main]{Theorem}
\newtheorem{lemma}[main]{Lemma}
\newtheorem{induct}[main]{Induction Hypothesis}
\newtheorem{corollary}[main]{Corollary}
\theoremstyle{definition}

\newtheorem{definition}[main]{Definition}

\newtheorem{fact}[main]{Fact}

\theoremstyle{remark}
\newtheorem{remark}[main]{Remark}

\numberwithin{equation}{section}

\tikzstyle{arrow} = [thick,->,>=stealth]

\begin{document}

\title{Toward Understanding the Feature Learning Process of Self-supervised Contrastive Learning}
\author{
    Zixin Wen \\
    \url{zixinw@andrew.cmu.edu} \\
    UIBE Beijing
\and 
    Yuanzhi Li \\
    \url{yuanzhil@andrew.cmu.edu} \\
    Carnegie Mellon University 
}
\date{June 1, 2021\footnotetext{V1 appeared on June 1, 2021. V2 polished writing and added citations, V3 corrected related works. We would like to thank Zeyuan Allen-Zhu for many helpful suggestions on the experiments, and thank Qi Lei, Jason D. Lee for clarifying results of their paper.}}

\maketitle

\begin{abstract}
    How can neural networks trained by contrastive learning extract features from the unlabeled data? Why does contrastive learning usually need much stronger data augmentations than supervised learning to ensure good representations? These questions involve both the optimization and statistical aspects of deep learning, but can hardly be answered by the analysis of supervised learning, where the target functions are the highest pursuit. Indeed, in self-supervised learning, it is inevitable to relate to the optimization/generalization of neural networks to how they can encode the latent structures in the data, which we refer to as the \textit{feature learning process}.

    In this work, we formally study how contrastive learning learns the feature representations for neural networks by analyzing its feature learning process. We consider the case where our data are comprised of two types of features: the more semantically aligned sparse features which we want to learn from, and the other dense features we want to avoid. Theoretically, we prove that contrastive learning using \textbf{ReLU} networks provably learns the desired sparse features if proper augmentations are adopted. We present an underlying principle called \textbf{feature decoupling} to explain the effects of augmentations, where we theoretically characterize how augmentations can reduce the correlations of dense features between positive samples while keeping the correlations of sparse features intact, thereby forcing the neural networks to learn from the self-supervision of sparse features. Empirically, we verified that the feature decoupling principle matches the underlying mechanism of contrastive learning in practice.
\end{abstract}

%\tableofcontents

\newpage

\section{Introduction}

Self-supervised learning \cite{devlin2019bert,mikolov2013efficient,sutskever2014sequence,jing2020self} has demonstrated its immense power in different areas of machine learning (e.g. BERT \citep{devlin2019bert} in natural language processing). Recently, it has been discovered that contrastive learning \cite{tian2019contrastive,he2020momentum,chen2020a,chen_big_2020,grill_bootstrap_2020,chen_exploring_2020}, one of the most typical forms of self-supervised learning, can indeed learn representations of image data that achieve superior performance in many downstream vision tasks. Moreover, as shown by the seminal work \cite{he2020momentum}, the learned feature representations can even outperform those learned by supervised learning in several downstream tasks. The remakable potential of contrastive learning methods poses challenges for researchers to understand and improve upon such simple but effective algorithms.

Contrastive learning in vision learns the feature representations by minimizing pretext task objectives similar to the cross-entropy loss used in supervised learning, where both the inputs and ``labels'' are derived from the unlabeled data, especially by using augmentations to create multiple views of the same image. The seminal paper \cite{Chen2020} has demonstrated the effects of stronger augmentations (comparing to supervised learning) for the improvement of feature quality. \cite{tian2020what} showed that as the augmentations become stronger, the quality of representations displayed a U-shaped curve. Such observations provided insights into the inner-workings of contrastive learning. But it remains unclear \emph{what has happened in the learning process} that renders augmentations necessary for successful contrastive learning.

Some recent works have been done to understand contrastive learning from theoretical perspective \citep{arora2019theoretical,wang2020understanding,tsai2020demystifying}. However, these works have not analyzed how \textbf{data augmentations} affect the \textbf{feature learning process} of \textbf{neural networks}, which we deem as crucial to understand how contrastive learning works in practice. We state the fundamental questions we want to address below, and provide tentative answers to all the questions by building theory on a simplified model that shares similar structures with real scenarios, and we provide some empirical evidence through experiments to verify the validity of our models.

\begin{Msg}{Fundamental Questions}
    1. How do \textbf{neural networks} trained by contrastive learning learn their feature representations \textbf{efficiently}, and are the representations similar to those learned in supervised learning?\\
    2. Why does contrastive learning {in deep learning} collapse in practice when no augmentation is used, and how do standard augmentations on the data help contrastive learning?
\end{Msg}

\subsection{Our Contributions}

\begin{figure*}[t!]
    \includegraphics[width=\textwidth]{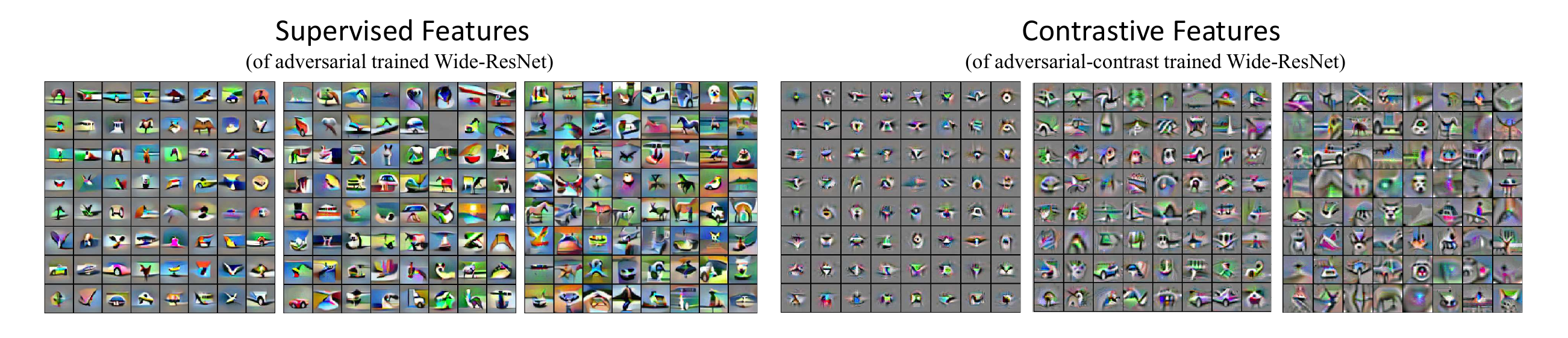}\caption{The difference between supervised features and contrastive features (in the higher layers of Wide-ResNet 34x5 over CIFAR10). While both features contain shapes of objects, the supervised features are more colorful than the contrastive features. (here both crop-resize and color distortion were used in contrastive learning, while no color distortion was used in supervised learning. The adversarial-contrast learning follows \cite{kim2020adversarial}). And we use the visualization technique in~\cite{allen-zhu2020feature}.} 
    \label{fig:contrastive vs supervised}
\end{figure*}

In this paper we directly analyze the \textbf{feature learning process} of contrastive learning for neural networks (i.e. learning the hidden layers of the neural network). Our results hold for certain data distributions based on \emph{sparse coding model}. Mathematically, we assume our input data are of the form \(x = \M z+\xi\) , where \(\M z\) is called the sparse signal such that \(\|z\|_0 = \widetilde{O}(1)\), and \(\xi\) is \textbf{the spurious dense noise}, where we simply assume that \(\xi\) follows from certain dense distributions (such that \(\mathrm{span}(\xi) \equiv \mathrm{span}(x)\)) with large norm (e.g., \(\|\xi\|_2 = \poly(d) \gg \|\M z\|_2 \approx \widetilde{O}(1)\)). Formal definition will be presented in \myref{sec:setup}{Section}, as we argue that sparse coding model is indeed a proper \emph{provisional} model to study the feature learning process of contrastive learning.

\paragraph{Theoretical results.}
Over our data distributions based on sparse coding model, when we perform contrastive learning by using stochastic gradient descent (SGD) to train a one-hidden-layer neural networks with ReLU activations:
\begin{enumerate}
    \item If no augmentation is applied to the data inputs, \textbf{the neural networks will learn feature representations that emphasize the spurious dense noise}, which can easily overwhelm the sparse signals.
    \item If \emph{natural} augmentation techniques  (in particular, the \(\RandomMask\) defined in \myref{def:random-mask}{Definition}) are applied to the training data, \textbf{the neural networks will avoid learning the features associated with dense noise but pick up the features on the sparse signals}. Such a difference of features brought by data augmentation is due to a principle we refer to as {\bfseries``feature decoupling''}. Moreover, these features can be learned \emph{efficiently} simply by doing a variant of Stochastic Gradient Descent (SGD) over the contrastive training objective (after data augmentations). 
    \item The features learned by neural networks via contrastive learning (with augmentations) is similar to the features learned via supervised learning (under sparse coding model). This claim holds as long as two requirements are satisfied: (1) The sparse signals in the data have not been corrupted by augmentations in contrastive learning; (2) The labels in supervised learning mostly depends on the sparse signals. %This suggests an advantage of contrastive learning over supervised learning: \textbf{contrastive learning could possibly pick up features from the data that are overlooked by supervised labels}.
\end{enumerate}
Therefore, our theory indicates that in our model, the success of contrastive learning of neural networks relies essentially on the data augmentations to remove the features associated with the spurious dense noise. We abstract this process into a principle below, which we show to hold in neural networks used in real-world settings as well. 

\begin{Msg}{Feature Decoupling}
    Augmentations in contrastive learning serve to \textbf{decouple the correlations of spurious features} between the representations of positive samples. Moreover, after the augmentations, the neural networks will ignore the decoupled features and learn from the similarities of features that are \emph{more resistant to data augmentations}.
\end{Msg}

We will prove that contrastive learning can successfully learn the desired sparse features using this principle. The intuitions of our proof will be present in \myref{sec:proof-intuition}{Section}.

\begin{figure*}[t!] 
    \centering
    \includegraphics[width=1.0\textwidth]{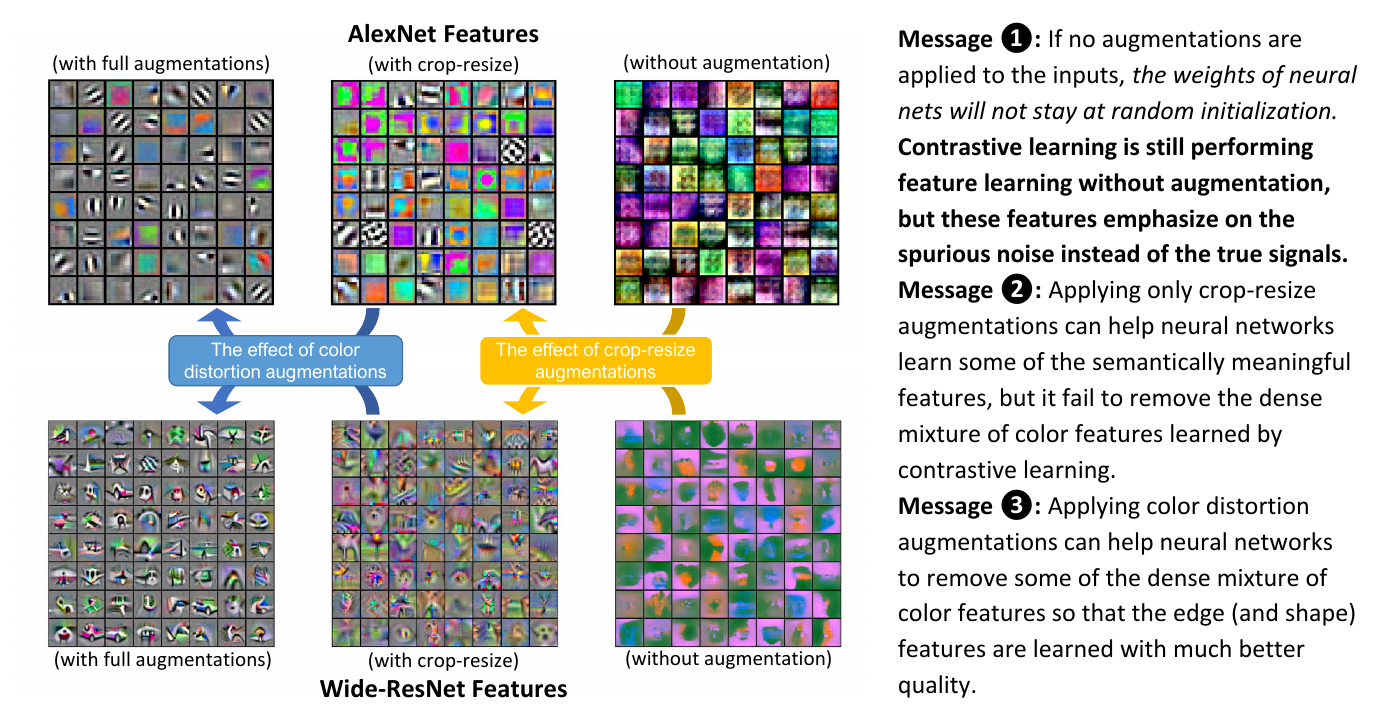}\caption{Evidence of \textbf{feature decoupling}: how do augmentations affect the features learned by neural networks in contrastive learning. The two different augmentations we have conducted here are \emph{color distortions} and \emph{crop-resize}.The color distortions we used consist of color jittering and random grayscale.}
    \label{fig:feature-decouple}
\end{figure*}
\paragraph{Empirical evidence of our theory.} Empirically, we conduct multiple experiments to justify our theoretical results, and the results indeed matches our theory. We show in contrastive learning: 
\begin{itemize}
    \item \textbf{When no proper augmentation is applied to the data, the neural network will learn features with dense patterns.} As shown in \myref{fig:feature-decouple}{Figure}, \myref{fig:tsne}{Figure} and \myref{fig:sparsity}{Figure}: If no augmentations are used, the learned features are completely meaningless and the representations are dense; If only crop-resize augmentations are used, then the mixture of color features (which also generate dense firing patterns) will remain in the neural network and prevent further separation of clusters.
    \item {\textbf{Standard augmentations removes features associated with dense patterns, and the remaining features do exhibit sparse firing pattern.}} As shown in \myref{fig:tsne}{Figure} and \myref{fig:sparsity}{Figure}, if no (suitable) augmentations are applied, the neural networks will learn dense representations of image data. After the augmentations, neural networks will successfully form separable clusters of representations for image data, and the \emph{learned features  indeed emphasizes sparse signals}.
    \item \textbf{The features learned in contrastive learning resemble the features learned in supervised learning}. As shown in \myref{fig:contrastive vs supervised}{Figure}, the shape features (filters that exhibit shape images) of the higher layer of Wide-ResNet via supervised learning are similar to those learned in contrastive learning. However, color features learned in supervised learning are much more than those in contrastive learning. This verifies our theoretical results that features preserved under augmentations will be learned by both contrastive and supervised learning.
    
\end{itemize}

\subsection{Related Work}

\begin{figure*}[t!] 
    \centering
    \includegraphics[width=0.9\textwidth]{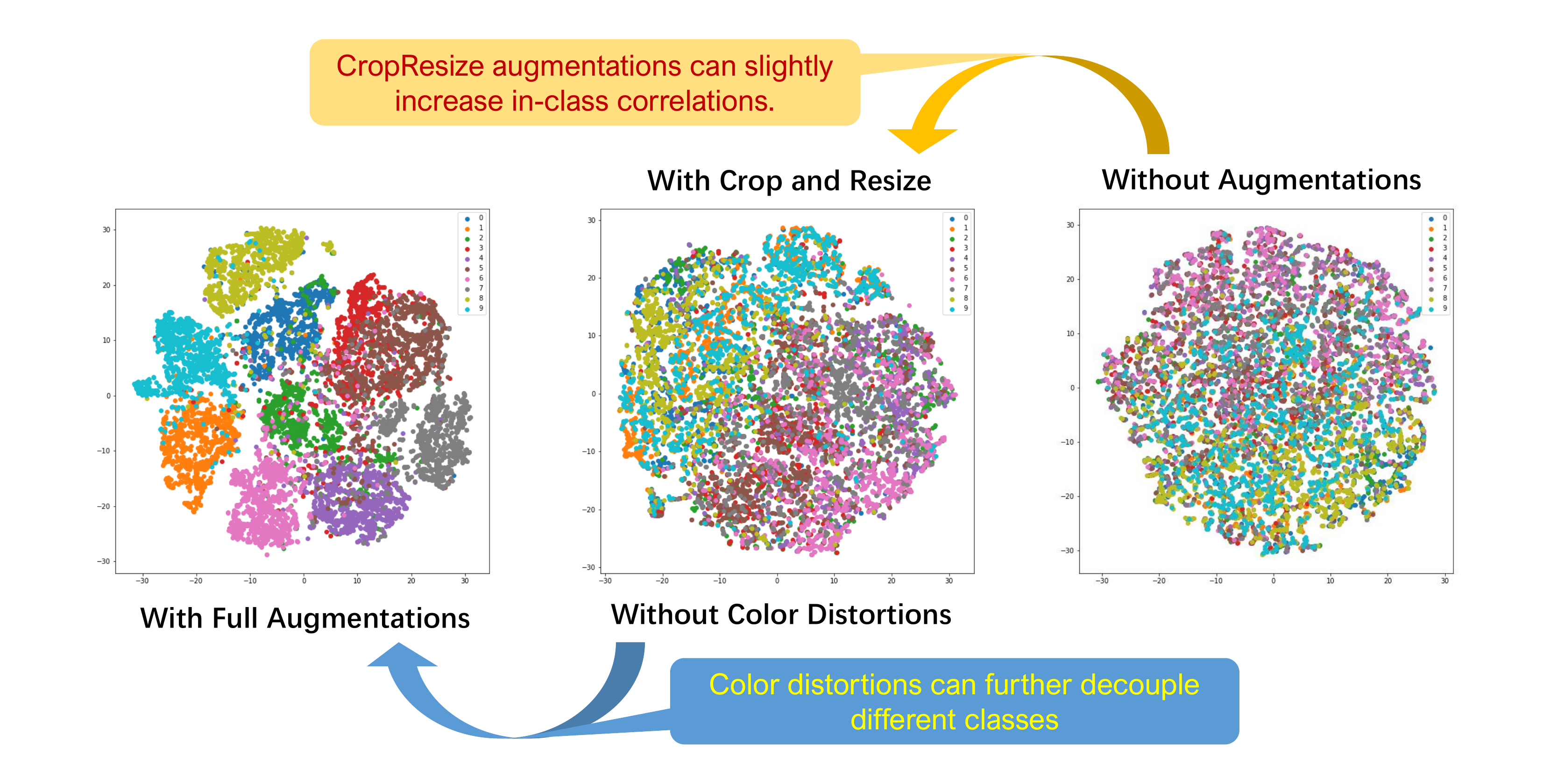}\caption{Evidence supporting our theoretical framework: the effects of augmentations on the learned representations of Wide-ResNet 34x5 over CIFAR10 visualized via t-SNE. The differences bewteen features learned under different augmentations shows that the neural networks will indeed learn \textbf{dense representations} if augmentation is not powerful enough.}
    \label{fig:tsne}
\end{figure*}

\paragraph{Self-supervised learning.} Self/un-supervised representation learning has a long history in the literature. In natural language processing (NLP), self-supervised learning has been the major approach \cite{mikolov2013efficient,devlin2019bert}. The initial works \cite{Carreira-Perpinan2005,Smith2005,gutmann2012noise} of contrastive learning focus on learning the hidden latent variables of the data. Later the attempts to use self-supervised to help pretraining brought the contrastive learning to visual feature learning \citep{oord2018representation,tian2019contrastive,he2020momentum,chen2020a,chen_big_2020,grill_bootstrap_2020,chen_exploring_2020}. On the theoretical side, there has been a lot of papers trying to understand un/self-supervised learning \citep{Coates2011AnAO,radhakrishnan2018memorization,arora2019theoretical,nguyen2019benefits,lee_predicting_2020,wang2020understanding,tsai2020demystifying,tian2020what,tosh2020contrastive,tosh2021contrastive,haochen2021provable,haochen2021provable}. For contrastive learning, \cite{arora2019theoretical} assume that different positive samples are independently drawn from the same latent class, which can be deemed as supervised learning. \cite{wang2020understanding} pointed out the tradeoff between alignment and uniformity. \cite{tsai2020demystifying,tian2020what} proposed to analyze contrastive learning via information-theoretic techniques. \cite{lee_predicting_2020} analyzed the optimal solution of a generative self-supervised pretext task, and \cite{tosh2021contrastive} analyzed contrastive loss from the same perspective. \cite{haochen2021provable} analyzed a spectral version of contrastive loss and analyzed its statistical behaviors. However, the above theoretical works do not study \emph{how features are learned} by \textbf{neural networks} and \emph{how augmentations affect the learned features}, which are essential to understand contrastive learning in practice. \cite{tian2020understanding} tries to analyze the learning process, but their augmentation can fix the class-related node and resample all latent nodes in their generative models, reducing the problem to supervised learning. 

\paragraph{Optimization theory of neural networks.}
There are many prior works on the supervised learning of neural networks. The works \cite{li2017convergence,brutzkus2017globally,ge2018learning,soltanolkotabi2017learning,li2018algorithmic} focus on the scenarios where data inputs are sampled from Gaussian distributions. \emph{We consider in our paper the Gaussian part of the data to be spurious} and use augmentation to prevent learning from them. Our approach is also fundamentally different from the \emph{neural tangent kernel} (NTK) point of view \cite{jacot2018neural,Li2018,du2019gradient,allen-zhu2019a,allen-zhu2019learning,allen-zhu2019on,chen2019much}. The NTK approach relies first order taylor-expansion with extreme over-parameterization, and cannot explain the \textbf{feature learning process} of neural networks, because it is merely linear regression over \emph{prescribed feature map}. Some works consider the regimes beyond NTK \citep{allen-zhu2019what,allen-zhu2020backward,allen-zhu2020feature,allen2020towards,li2020learning,bai_beyond_2020,allen2021forward}, which shedded insights to the innerworkings of neural networks in practice.

\section{Problem Setup}\label{sec:setup}

\begin{figure*}[t!]\centering
    {\includegraphics[width = \textwidth]{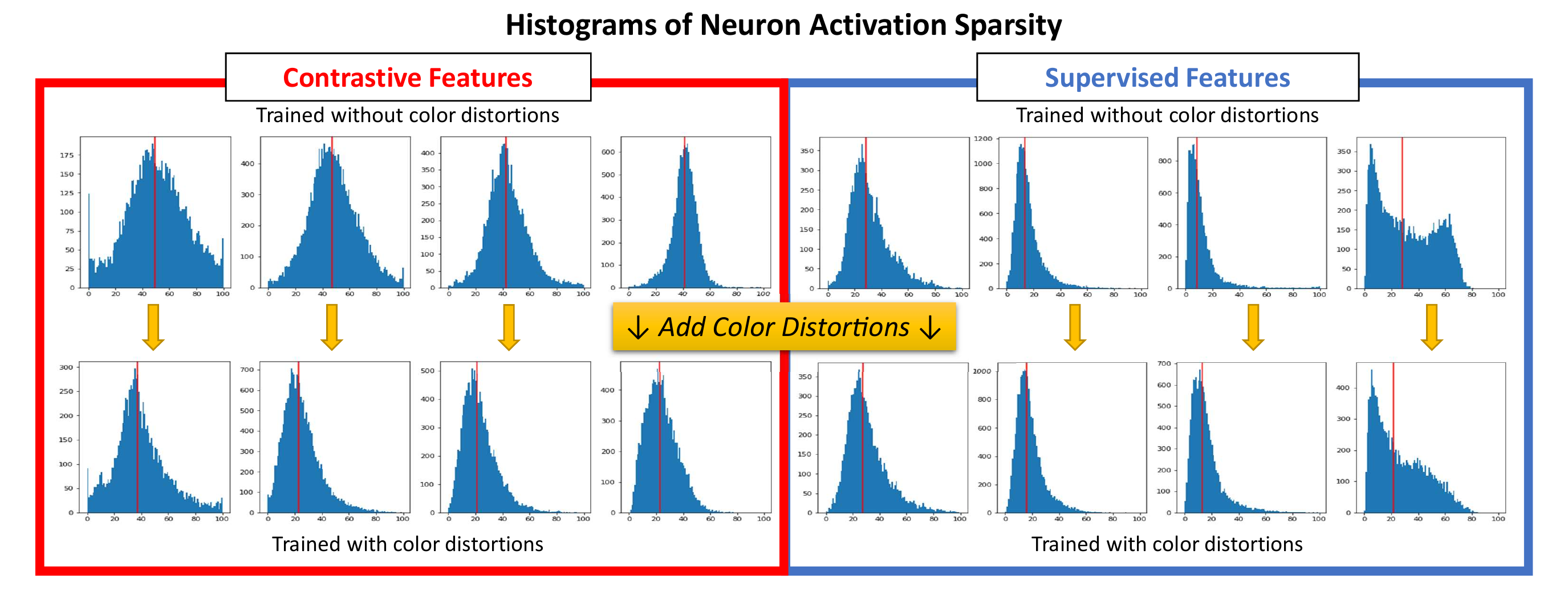}}\caption{Another evidence supporting our theoretical framework. After adding the color distortion to augmentation, the  neurons of AlexNet (2nd to 5th layer) exhibit sparser firing patterns over input images of CIFAR10. Meanwhile the networks obtained from supervised learning always have sparse activations regardless of augmentations. These observations indicate that (1). In contrastive learning, augmentations can indeed help neural nets focus on the sparse signals. (2). Sparse signals are indeed more important for the downstream tasks (such as supervised classification).}
    \label{fig:sparsity}
\end{figure*}

\paragraph{Notations.} We use \(O,\Omega,\Theta\) notations to hide universal constants with respect to \(d\) and \(\widetilde{O},\widetilde{\Omega},\widetilde{\Theta}\) notations to hide polylogarithmic factors of \(d\). We use the notations \(\poly(d),\ \polylog(d)\) to represent constant degree polynomials of \(d\) or \(\log d\). We use \([d]\) as a shorthand for the index set \(\{1,\dots, d\}\). For a matrix \(\M\in\R^{d'\times d}\), we use \(\M_j\), where \(j\in[d]\), to denote its \(j\)-th column. We say an event happens with high probability (or w.h.p. for short) if the event happens with probability at least \(1 - e^{-\Omega(\log^2 d)}\). We use \(\N(\mu,\Sigma)\) to denote standard normal distribution in with mean \(\mu\) and covariance matrix \(\Sigma\).

\subsection{Data Distribution.} 
We present our sparse coding model below, which form the basis of our analysis.

\begin{definition}[sparse coding model (\(\D_x, \D_z, \D_{\xi}\))]\label{def:sparse-coding}
    We assume our raw data samples $x \in \mathbb{R}^{d_1}$ are generated i.i.d. from distribution \(\D_x\) in the following form:
    \begin{displaymath}
        x = \M z + \xi \sim \D_x,\quad z \sim \D_z,\quad \xi \sim \D_{\xi} = \N(\mathbf{0},\sigma_{\xi}^2\Id_{d_1})
    \end{displaymath}
    Where $z \in \mathbb{R}^d$. We refer to \( z\) as the \textbf{sparse signal} and \(\xi\) as the \textbf{spurious dense noise}.  We assume \(d_1 = \poly(d)\) for simplicity. We have the following assumptions on \(\M, z, \xi\) respectively:\footnote{ The choice of \(\Pr(|z_j|=1) = \Theta(\frac{\log\log d}{d})\) instead of \(\Theta(\frac{1}{d})\) here is to avoid the scenario where \(z \) could be zero with probability \(\geq \Omega(1)\). One can also assume the noise vector \(\xi\) to be non-spherical Gaussian or has certain directions with larger variance than \(\Theta(\sqrt{\log d}/d)\). Although our theory tolerates a wider range of these parameters, we choose to present the simplest setting.}
    
    \begin{itemize}
        \item The dictionary matrix \(\M = [\M_1,\dots,\M_d] \in \R^{d_1\times d}\) is a column-orthonormal matrix, and satisfies \(\|\M_j\|_{\infty} \leq \widetilde{O}\big(\frac{1}{\sqrt{d_1}}\big)\) for all \(j \in [d]\).
        \item The sparse latent variable \(z = (z_1,\dots, z_d)^{\top} \in \{-1,0,1\}^d \) is sampled from \(\D_z\), we assume all \(z_j\)'s are symmetric around zero, satisfying \(\Pr(|z_j| = 1) = \Theta\big(\frac{\log\log d}{d}\big)\), and are identically distributed and independent across all \(j \in [d]\).
        \item For the spurious dense noise \(\xi\sim \N(\mathbf{0},\sigma_{\xi}^2\Id_{d_1})\), we assume its variance \(\sigma_{\xi}^2 = \Theta(\frac{\sqrt{\log d}}{d})\).
    \end{itemize}

\end{definition}

\paragraph{Why sparse coding model.} Sparse coding model was first proposed by neuroscientists to model human visual systems \citep{olshausen1997sparse,olshausen2004sparse}, where they provided experimental evidence that sparse codes can produce coding matrices for image patches that resemble known features in certain portion of the visual cortex. It has been further studied by \cite{Foeldiak1998,vinje2000sparse,olshausen2004sparse,protter2009image,yang2009linear,mairal2014sparse} to model images based on the sparse occurences of objects. For the natural language data, sparse code is also found to be helpful in modelling the polysemy of words \citep{arora2018linear}. Thus we believe our setting share some similar structures with practical scenarios. 

\paragraph{Why sparse signals are more favorable than the dense signal.} Theoretically, we argue that sparse signals are more favorable as we can see from the properties of our sparse signals \(\M z\) and dense signals \(\xi\):
\begin{enumerate}
    \item \textbf{The significance of sparse signal.} Since \(\sigma_{\xi}^2 = \Theta\big(\frac{\sqrt{\log d}}{d}\big)\), the \(\ell_2\)-norm of \(\xi\) becomes \(\|\xi\|_2^2 \geq \Omega(\poly(d)) \gg \|\M z\|_2\) w.h.p. However, whenever there is one \(z_j\neq 0\), we have \(|\vbrack{\M z,\M_j}| \geq \Omega(1)\) while \(|\vbrack{\xi,\M_j}| \leq \widetilde{O}(\frac{1}{\sqrt{d}})\) with high probability. {This indicates that even if the dense signal is extremely large in norm, it cannot corrupt the sparse signal.} 
    \item \textbf{The individuality of dense signal.} For each \(j \in [d]\), the sparse feature \(\pm\M_j\) are shared by at least \(\widetilde{\Omega}(\frac{1}{d})\) of the population. However, for polynomially many independent dense signal \(\xi_i\), with high probability we have \(|\left\langle\frac{\xi_i}{\|\xi_i\|_2},\frac{\xi_j}{\|\xi_j\|_2}\right\rangle| \leq \widetilde{O}(\frac{1}{\poly(d)})\) for any \(i \neq j\), which shows that the dense signal \(\xi\) is in some sense \textbf{``individual to each sample"}. \textbf{This also suggests that any representations of the dense signal can hardly form separable clusters other than isolated points.}
\end{enumerate}

\subsection{Learner Network and Contrastive Learning Algorithm}

We use a single-layer neural net \(f:\R^{d_1}\to \R^{m}\) with ReLU activation as our contrastive learner, where \(m\) is the number of neurons. More precisely, it is defined as follows:
\begin{align*}
    f(x) &= (h_{1}(x),\dots,h_{m}(x))^\top \in \R^m,&
    h_{i}(x) &= \ReLU(\vbrack{w_i,x} - b_i) - \ReLU(-\vbrack{w_i,x} - b_i)
\end{align*}
Such activation function \(h_i\) is a symmetrized version of \(\ReLU\) activation. We initialize the parameters by \(w_i^{(0)}\sim\N(0,\sigma_0^2\Id_{d_1})\) and \(b_i^{(0)} = 0\), where \(\sigma_0^2 = \Theta(\frac{1}{d_1 \poly(d)})\) is small (and also theoretically friendly). Corresponding to the two types of signals in \myref{def:sparse-coding}{Definition}, \textbf{we call the learned weights of neural networks \(\{w_i\}_{i \in [m]} \) ``features'', and we expand the weight of a neuron as} \[ \textstyle w_i = \sum_{j\in[d]} \vbrack{w_i,\M_j}\M_j + \sum_{j\in[d_1]\setminus [d]}\vbrack{w_i,\M^{\perp}_j}\M^{\perp}_j \] 
where we name the (unit-norm) directions \(\M_j\) and \(\M^{\perp}_j\) as follows:
\begin{itemize}
    \item We call \(\M = [\M_j]_{j\in[d]}\) the \textbf{sparse features}, which is the features associated with our sparse signals \(\M z\). These are the desired features we want our learner network to learn.
    \item We call \(\M^{\perp} = [\M^{\perp}_j]_{j\in[d_1]\setminus [d]} \) (the orthogonal complement of \(\M\)) the \textbf{spurious dense features}, which is associated with the dense signal \(\xi\) only. These are the undesired features for our learner.
\end{itemize}

Our contrastive loss function is based on the similarity measure defined as follows: let \(x\) and \(x'\) be two samples in \(\R^{d_1}\), and \(f:\R^{d_1}\to \R^d\) be a feature map, the similarity of the representations of \(x\) and \(x'\) is defined as 
\begin{equation}\label{eqdef:sim-measure}
    \Sim_{f}(x,x') := \vbrack{f(x), \StopGrad(f(x'))}
\end{equation}
The \(\StopGrad(\cdot)\) operator here means that we do not compute its gradient in optimization, which is inspired by recent works \cite{grill_bootstrap_2020,chen_exploring_2020}. Below we present the definition of contrastive loss.

\begin{definition}[Contrastive loss function]
    Given a pair of positive data samples \(x_p,\, x_p'\) and a batch of negative data samples \(\Nfr = \{x_{n,s}\}_{s\in[\Nfr]}\), letting \(\tau\) be the temperature parameter, and denoting \(\mathfrak{B} = \{x_p'\}\cup\Nfr\), the contrastive loss is defined as\footnote{Our contrastive loss \eqref{eqdef:contrastive-loss} here uses the unnormalized representations instead of the normalized ones, which is simpler to analyze theoretically. As shown in \cite{chen2020a}, contrastive learning using unnormalized representation can also achieve meaningful (more than 57\%) ImageNet top-1 accuracy in linear evaluation of the learned representations.} 
\begin{align}\label{eqdef:contrastive-loss}
    \mathcal{L}(f, x_p,x_p',\Nfr) :=  - \tau\log\left( \frac{e^{\Sim_f(x_p,x_p')/\tau}}{\sum_{x \in \Bfr} e^{\Sim_f(x_p,x)/\tau}}\right)
\end{align}
\end{definition}

Nevertheless, as shown by our experiments (see \myref{fig:feature-decouple}{Figure} or \myref{fig:tsne}{Figure}), the success of contrastive learning rely on the data augmentations adopted in generating the positive samples. We present our augmentation method \(\mathsf{RandomMask}\) below, which is an analog of the random cropping data augmentation used in practice. 

\begin{definition}[\(\RandomMask\) and \(\D_{\DD}\)] \label{def:random-mask} 
We first define a distribution \(\D_{\DD}\) over the space \(\R^{d_1\times d_1}\) of diagonal matrices as follows: let \(\DD = \diag(\DD_{\ell,\ell})_{\ell\in[d_1]} \sim \D_{\DD}\) be a diagonal matrix with $\{0, 1\}$ entries, its diagonal entries \(\DD_{\ell,\ell}\) are sampled from \(\mathrm{Bernoulli}\big(\frac{1}{2}\big)\) independently. Now given a positive sample \(x_p \sim \D_x\), we generate \(\DD \sim \D_{\DD}\), and then apply \(\DD\) to generate \(x_p^+\) and \(x_p^{++}\) as follows:
\begin{displaymath}
    x_p^+ := 2\DD x_p, \qquad x_p^{++} := 2(\Id - \DD)x_p
\end{displaymath}
\end{definition}

\begin{remark}
    We do not apply any augmentation to our negative samples for simplicity of theory. And also we point out that adding such augmentations do not reveal any further insights, since we do not expect the augmentation to decouple any correlations other than that between positive samples. Nevertheless our theory can easily adapt to the setting where augmentations are applied to all input data.
\end{remark}

\paragraph{Intuitions behind the \(\RandomMask\) augmentation.} Intuitively, the \(\RandomMask\) data augmentation simply masks out roughly a half of the coordinates in the data. The contrastive learning objective asks to learn features that can match \emph{two disjoint set of the coordinates} of given data points. Suppose we can maintain the correlations of desired signals between the disjoint coordinates and remove the undesired correlations, then we can force the algorithm to learn from the desired signals. We will discuss the effects of augmentations with more detail in \myref{sec:proof-intuition}{Section}.

\paragraph{Significance of our analysis on the data augmentations.} Our  analysis on the data augmentation are fundamentally different from those in \cite{tsai2020demystifying,tian2020understanding,wei_theoretical_2020,lee_predicting_2020}. In \cite{tsai2020demystifying,tian2020understanding}, they argued their data augmentations can change the latent variables unretaled to the downstream tasks, while real-life augmentations can only affect the observables, and cannot identify which latents are the task-specific ones. \cite{wei_theoretical_2020} assumed their augmentations are only picking data points inside a small neighborhood of the original data (in the observable space), which is also untrue in practice. Indeed, common augmentations such as crop-resize and color distortions can considerably change the data, making it very distant to the original data in the observable space. Our analysis of \(\RandomMask\) makes a step toward understanding realistic data augmentations in deep learning.

\paragraph{Training algorithm using SGD.} We consider two cases: training with augmentation and without augmentation: 
\begin{itemize}
    \item \textbf{With augmentations.} We perform stochastic gradient descent on the following objectives: let \(f_t\) be the contrastive learner at each iterations \(t \geq 0\), the objectives is defined as follows:
    \begin{align*}
        &L(f_t) := \underset{x_p^+, x_p^{++},\Nfr}\E\left[\mathcal{L}(f_t,x_p^+,x_{p}^{++},\Nfr)\right], \\
        &\Obj(f_t) := L(f_t) + \frac{\lambda}{2} \sum_{i\in[m]}\|w_i^{(t)}\|_2^2
    \end{align*}
    where \(\lambda \in [\frac{1}{d^{1.001}}, \frac{1}{d^{1.499}}]\) is the regularization parameter, \(L(f_t)\) is the population loss and \(x_p,\{x_{n,s}\}_{\Nfr}\) are sampled from \(\D_x\), \(x_p^+,x_p^{++}\) are obtained by applying \(\RandomMask\) to \(x_p\). At each iteration \(t\), let \(\eta = \frac{1}{\poly(d)}\) be the learning rate, we update as:
    \begin{displaymath}
        w_i^{(t+1)} \gets w_i^{(t)} - \eta \nabla_{w_i}\Obj(f_t)
    \end{displaymath}
    \item \textbf{Without augmentations.} We perform stochastic gradient descent on the following modified objectives \(\Obj_{\textsf{NA}}(f_t)\):
    \begin{align*}
        &L_{\mathsf{NA}}(f_t) := \underset{x_p,{\Nfr}}\E\left[\mathcal{L}(f_t,x_p,x_{p},\Nfr)\right],\\
        & \Obj_{\mathsf{NA}}(f_t) := L_{\mathsf{NA}}(f_t) + \frac{\lambda}{2} \sum_{i\in[m]}\|w_i^{(t)}\|_2^2
    \end{align*} 
    where \(\lambda \leq O(1/d)\) can be arbitrary. The learning rate \(\eta \leq o(1)\) can also be arbitrary. We update as:
    \begin{displaymath}
        w_i^{(t+1)} \gets w_i^{(t)} - \eta \nabla_{w_i}\Obj_{\mathsf{NA}} (f_t)
    \end{displaymath}
\end{itemize}

We manually tune bias\footnote{In fact, when trained without augmentations, the biases can be tuned arbitrarily as long as the neurons are not killed. It will not affect our results.} \(b_i^{(t)}\) during the training process as follows: let \(T_1 = \Theta\left(\frac{d \log d_1}{\eta \log\log d}\right)\) be the iteration when all \(\|w_i^{(0)}\|_2 \leq \frac{c_0}{1000}\|w_i^{(t)}\|_2 \). At \(t = T_1\), we reset the bias \(b_i^{(t)} = \sqrt{\frac{2\log d}{d}}\|w_i^{(t)}\|_2\) and update by \(b_i^{(t+1)} = b_i^{(t)} \left(1 + \eta_{b,t}\right)\), where \(\eta_{b,t} = \max\{\frac{\eta}{d}, \frac{\|w_i^{(t+1)}\|_2}{\|w_i^{(t)}\|_2}-1\}\) if \( b_i^{(t)} \leq \frac{\polylog(d)}{\sqrt{d}}\).\footnote{We manually increase the bias after the weights are updated in order to simplify the proof. It can be verified that the biases will indeed increase over synthetic sparse coding data. More importantly, in synthetic experiments, the bias will decrease if no augmentation is used.}

\section{Main Results} \label{sec:main-results}

We now state the main theorems of this paper in our setting. We argue that contrastive learning objective learns completely different features with/without data augmentation. \textbf{Moreover, to further illustrate the how these learned features are different with/without data augmentation, we also consider two simple downstream tasks to evaluate the performance of contrastive learning}. We argue that using a linear function taking the learned representation as input to perform these tasks can be more efficient than using raw inputs, it should be considered as successful representation learning. 

\begin{definition}[downstream tasks]\label{def:downstream-task}
    We consider two simple supervised tasks, regression and classification, based on the label functions defined below:
    \begin{itemize}
        \item Regression: For each \(x = \M z+ \xi \sim \D_x\), we define its label \(y = \vbrack{w^{\star}, z}\), where \(w^{\star} \in \R^{d}\).
        \item Classification: For each \(x = \M z + \xi \sim \D_x\), we define \(y = \sign(\vbrack{w^{\star},z})\), where \(w^{\star} \in \R^{d}\).
    \end{itemize}
    where in both cases we assume \(w^{\star}\) satisfies \(|w^{\star}_j| = \Theta(1) \) for all \(j \in [d]\).
\end{definition}

Given these downstream tasks, our goal of representations learning is to obtain suitable feature representations and train a linear classifier over them. Specifically, let \(f(\cdot)\) be the obtained representation map, we use optimization tool\footnote{Since the downstream learning tasks only involve linear learners on convex objectives, for simplicity, we directly argue the properties of the minimizers for these downstream training objectives. } to find \(w^{*}\) such that 
\begin{displaymath}
    w^* = \mathrm{arg}\min_{w \in \R^{m}} \E [ \widetilde{\mathcal{L}}(w^{\top}f(x),y)]
\end{displaymath}
where \(\widetilde{\mathcal{L}}(\cdot,\cdot)\) is the loss function for the downstream tasks considered: For regression, it is the \(\ell_2\) loss \(\widetilde{\mathcal{L}}(\hat{y},y) = (\hat{y}-y)^2\); For classification, it is the logistic loss \(\widetilde{\mathcal{L}}(\hat{y},y) = \log(1 + e^{-\hat{y}y})\). It should be noted that these tasks can be done by neural networks via supervised learning as shown in \cite{allen-zhu2020feature}, where the sample complexities have not been calculated exactly. However, using linear regression over the input \(x\) to find \(\M\) requires sample complexity at least \(\sqrt{d_1}\sigma_{\xi} \gg \poly(d)\), which under our setting can be as large as \(d^{100}\), much larger than those of linear regression over contrastive features from our results. Furthermore, even if one can locate the desired features \(\M\), the noise level \(\sigma_{\xi}^2 = \Theta(\frac{\sqrt{\log d}}{d})\) is still much larger than the signal size \(\E[z_j^2] = \Theta(\frac{\log\log d}{d})\), thus linear models will fail with constant probability.

\subsection{Contrastive Learning Without Augmentations}

We present our theorem for the learned features without using any augmentations.

\begin{theorem}[Contrastive features learned without augmentation]
    Let \(f_t^{\mathsf{NA}} \) be the neural network trained by conrtastive learning without any data augmentations, and using \(|\Nfr| = \poly(d)\) many negative samples, we have objective guarantees \(L_{\mathsf{NA}}(f_t^{\mathsf{NA}}) = o(1)\) for any \(t \geq \frac{\poly(d)}{\eta}\). Moreover, given a data sample \(x = \M z +\xi \sim \D_x\), with high probability it holds:
    \begin{displaymath}
         \left\vbrack{\frac{f_t^{\mathsf{NA}}(x)}{\|f_t^{\mathsf{NA}}(x)\|_2},\frac{f_t^{\mathsf{NA}}(\xi)}{\|f_t^{\mathsf{NA}}(\xi)\|_2} \right} \geq 1 - \widetilde{O}\left(\frac{1}{\poly(d)}\right)
    \end{displaymath}
\end{theorem}

This results means that in the representations of \(f_t\), the sparse signal \(\M z\) are completely overwhelmed by the spurious dense signal \(\xi\). It would be easy to verify the following corollary:

\begin{corollary}[Downstream task performance]
    The learned network \(f_t^{\mathsf{NA}}\), where \(t \geq 0\), fail to achieve meaningful \(\ell_2\)-loss/accuracy in the downstream tasks in \myref{def:downstream-task}{Definition}. More specifically, no matter how many labeled data we have for downstream linear evaluation (where \(f_t^{\mathsf{NA}}\) is frozen):
    \begin{itemize}
        \item For regression, we have 
        \begin{displaymath}
            \underset{x\sim\D_x}{\E} |y - \vbrack{w^{*}, f_t^{\mathsf{NA}}}(x) |^2  \geq  \Omega(1)
        \end{displaymath}
        \item For classification, we have
        \begin{displaymath}
            \underset{x\sim\D_x}{\Pr} [ y = \sign(\vbrack{w^{*},f_t^{\mathsf{NA}}(x) }) ] = o(1)
        \end{displaymath}
    \end{itemize}
\end{corollary}

\subsection{Contrastive Learning With Augmentation}

We present our results of the learned features after successful training with augmentations.

\begin{theorem}[Contrastive features learned with augmentation]
    Let \(m = d^{1.01}\) be the number of neurons, \(\tau=\polylog(d)\), and \(|\Nfr| = \poly(d)\) be the number of negative samples. Suppose we train the neural net \(f_t\) via contrastive learning with augmentation, then for some small constant \(c<\frac{1}{1000}\), and some iterations \(T \in [T_3,T_4]\), where \(T_3 = \frac{d^{1.01}}{\eta}, T_4 = \frac{d^{1.99}}{\eta}\), we have objective guarantees
    \begin{align*}
        \frac{1}{T}\sum_{t\in [T]} \Obj(f_t)\leq o(1),\qquad \frac{1}{T}\sum_{t\in [T]}L(f_t) \leq o(1)
    \end{align*}
    Moreover, for each neuron \(i \in [m]\) and \(t \in [T_3,T_4]\), contrastive learning will learn the following set of features:
    \begin{displaymath}
        w_i^{(t)} = \sum_{j\in \N_i}\alpha_{i,j}\M_j +\sum_{j\notin \N_i } \alpha'_{i,j}\M_j + \sum_{j\in[d_1]\setminus[d]} \beta_{i,j}\M^{\perp}_{j}
    \end{displaymath}
    where \(\alpha_{i,j} \in [\frac{\tau}{d^c},\tau]\), \(|\N_i|=O(1)\), \(\alpha'_{i,j} \leq o(\frac{1}{\sqrt{d}})\|w_i^{(t)}\|_2\) and \(|\beta_{i,j}| \leq o(\frac{1}{\sqrt{d_1}})\|w_i^{(t)}\|_2\). Furthermore, for each dictionary atom \(\M_j\), there are at most \(o(m/d)\) many \(i \in [m]\) such that \( j\in \N_i \), and at least \(\Omega(1)\) many \(i \in [m]\) such that \( \N_i = \{j\}\).
\end{theorem}

This result indicates the following: let \(x = \M z +\xi \sim \D_x\) be a data sample and \(f_t, t\in [T_3,T_4]\) be the trained nerwork, then \(\|f_t(x) - f_t(\M z)\|_2 \leq \widetilde{O}(\frac{1}{\sqrt{d}}) \) with high probability, while \(\|f(\M z)\|_2 \geq \Omega(1)\) with probability at least \(1 - \frac{1}{\polylog(d)}\). Thus the learned feature map has successfully removed the spurious dense noise \(\xi\) from the model/representation. We have a direct corollary following this theorem.

\begin{corollary}[Downstream task performance]
    The learned feature map \(f_t\), \(t \in [T_3,T_4]\) obtained by contrastive learning perform well in all the downstream tasks defined in \myref{def:downstream-task}{Definition}. Specifically, we have 
    \begin{enumerate}
        \item For the regression task, with sample complexity at most \(\widetilde{O}(d^{1.001})\), we can obtain \(w^{*} \in \R^{m}\) such that
        \begin{displaymath}
            \underset{x\sim\D_x}{\E} |y - \vbrack{w^{*}, f_t(x)} |^2  = o(1)
        \end{displaymath}
        \item For the classification task, again by using logistic regression over feature map \(f_t\), with sample complexity at most \(\widetilde{O}(d^{1.001})\), we can find \(w^{*} \in \R^{d}\) such that
        \begin{displaymath}
            \underset{x\sim\D_x}{\Pr} [ y = \sign(\vbrack{w^{*},f_t(x)}) ] = 1 - o(1)
        \end{displaymath}
    \end{enumerate}
\end{corollary}

\section{Proof Intuition: The Feature Decoupling Principle} \label{sec:proof-intuition}

Theoretically speaking, contrastive learning objectives can be view as two parts, as is also observed in \cite{wang2020understanding}:
\begin{displaymath}
    \mathcal{L}  = -\Sim_{f}(x_p ,x'_{p}) + \tau\log\left(\sum_{x \in \Bfr}e^{\Sim_f(x_p,x)/\tau} \right)
\end{displaymath}
where the first part \(-\Sim_{f}(x_p ,x'_{p})\) emphasize similarity between positive samples, and the second part \(\tau\log\left\{\sum_{x \in \Bfr}e^{\Sim_f(x_p,x)/\tau} \right\}\) emphasize dissimilarities between the positive and negative samples. To understand what happens in the learning process, we separately discuss the cases of learning with/without augmentations below:

\paragraph{Why does contrastive learning prefer spurious dense noise without augmentation?}
Without data augmentation, we simply have $x_p = x_p'$. In this case, contrastive learning will learn to emphasize the signals that simultaneously maximize the correlation \(\vbrack{f(x_p^{++}),f(x_p^+)} = \|f_t(x_p)\|_2^2\) and minimize \(\vbrack{f(x_{n,s}), f(x_p^+)}\) by learning from all the available signals. \textbf{However, in our sparse coding model \(x = \M z+ \xi\), the spurious dense features \(\xi\) has much larger \(\ell_2\)-norm and the least correlations between different samples} (see \myref{sec:setup}{Section} for discussion). In contrast, the sparse signals \(\M z = \sum_{j}\M_j z_{j}\) display larger correlations between different samples because of possible co-occurences of features \(\M_j\) (i.e., at least \(\widetilde{\Omega}(\frac{1}{d})\) portion of the data contain feature \(\M_j\)). Thus the our contrastive learner will focus on learning the features associated with the dense noise \(\xi\), and fail to emphasize sparse features.

\paragraph{Feature Decoupling: How does augmentation remove the spurious dense noise:} Theoretically, we show how data augmentations help contrastive learning, which demonstrate the principle of \textbf{feature decoupling}. The spirit is that the augmentation should be able to making the dense signals completely different between the positive samples while preserve the correlations of sparse signals. 

Specifically, under our data model, if no augmentations are applied to the two positive samples \(x_p^+,x_p^{++}\) generated from \(x_p = \M z_p + \xi_p \sim \D_x\), their correlations will mostly come from the inner product of noise \(\vbrack{\xi_p,\xi_p}\), which can easily overwhelm those from the sparse signals \(\vbrack{\M z_p, \M z_p}\). Nevertheless, we have a simple observation: different coordinate \(\xi_{p,j}\) of our dense noise \(\xi_p\) are independent to each other, which enables a simple method to decorrelate the dense noise: by randomly applying two completely opposite masks \(\DD\) and \(\Id - \DD\) to the data \(x_p\) to generate two positive samples \(x_p^+ = 2\DD x_p\) and \(x_p^{++} = 2(\Id-\DD) x_p\). From our observation, such data augmentations can make the dense signals \(\DD\xi_p\) and \((\Id -\DD)\xi_p\) of \(x_p^+\) and \(x_p^{++}\) independent to each other. This independence will decouple the dense features between positive samples, which substantially reduces the gradients of the dense features.

\textbf{However, the sparse signals are more resistant to data augmentation.} As long as the sparse signals \(\M z = \sum_{j\in[d]}\M_j z_j\) span across the space, they will show up in both \(x_p^+\) and \(x_p^{++}\), so that their correlations will remain in the representations. More precisely, whenever a sparse signal \(\M_j\) is present (meaning its latent variable \(z_j \neq 0\)), it can be recovered both from \(2\DD \M z\) and from \( 2(\Id - \DD)\M  z\) with the correct decoding: e.g. we have $T_b (\langle \M_j, 2\DD x \rangle) \approx  T_b ( \langle \M_j, 2(\Id - \DD)x \rangle) \approx z_j $, where \(T_b(x) = x\1_{|x|\geq b}\) is a threshold operator with a proper bias $b > 0$. Unless in very rare case \(\M_j\) is completely masked by augmentations (that is \(\DD\M_j = 0\) or \((\Id-\DD)\M_j = 0\)), the sparse signals will remain their correlations in the feature representations, which will later be reinforced by neural networks following the SGD trajectory.

\section{Conclusion and Discussion}

In this work, we show a theoretical result toward understanding how contrastive learning method learns the feature representations in deep learning. We present the feature decoupling principle to tentatively explain how augmentations work in contrastive learning. We also provide empirical evidence supporting our theory, which suggest that augmentations are necessary if we want to learn the desired features and remove the undesired ones. We hope our theory could shed light on the innerworkings of how neural networks perform representation learning in self-supervised setting.

However, we also believe that our results can be significantly improved if we can build on more realistic data distributions. For example, real life image data  should be more suitably modeled as ``hierachical sparse coding model'' instead of the current simple linear sparse coding model. We believe that deeper network would be needed in the new model. Studying contrastive learning over those data models and deep networks is an important open direction.

\newpage

\appendix

\begin{center}
    \LARGE \textsc{Appendix: Complete Proofs}
\end{center}

\section{Proof Overview}

In this section we present an overview of our full proof. Before going into the proof, we describe some preliminaries.

\subsection{Preliminaries and Notations}
At every iteration \(t \geq 0\), we denote the weights of the neurons as \(w^{(t)} = \{w_i^{(t)}\}_{i\in [m]}\), and given \(x\in \R^{d_1}\) as input, the output of the network is denoted as
\begin{align*}
    f_t(x) &= (h_{1,t}(x),\dots,h_{m,t}(x))^\top \in \R^m,\\
    h_{i,t}(x) &= \ReLU(\vbrack{w_i^{(t)},x} - b_i^{(t)}) - \ReLU(-\vbrack{w_i^{(t)},x} - b_i^{(t)})
\end{align*}

\paragraph{Data Preparation and Loss Objective.}
Given a positive sample \(x_p \sim \D_x\), the augmented data are defined as follows: we generate random mask \(\DD \sim \D_{\DD}\) (defined in \myref{def:random-mask}{Def.}) and apply to \(x_p\) as:
\begin{displaymath}
    x_p^+ \gets 2\DD x_p,\qquad x_p^{++} \gets 2(\Id - \DD)x_p
\end{displaymath}
where the \(2\)-factor is to renormalize the data. Now recall our similarity measure is defined as \(\Sim_{f}(x_1,x_2) = \vbrack{f(x_1),\StopGrad(f(x_2))} \) for inputs \(x_1,x_2\). Our population contrastive loss objective is defined as follows: suppose in addition to \(x_p^+,\ x_p^{++}\), we are given a batch of negative samples \(\Nfr = \{x_{n,s}\}_{s\in[|\Nfr|]}\), where each \(x_{n,s} \sim \D_x\) independently (for short), we write \(\Bfr = \{x_p^{++}\}\cup\Nfr\) and define
\begin{align*}
    L(f_t) := \underset{x_p^+, \Bfr}\E\left[ - \Sim_{f_t}(x_p^+,x_p^{++}) + \tau \log\left(\sum_{x \in \Bfr} e^{\Sim_{f_t}(x_p^+,x) /\tau}\right) \right]
\end{align*}

\paragraph{The Gradient of Weights.}
We perform stochastic gradient descent on the objective \(\Obj(f_t) = L(f_t) + \lambda \|w^{(t)}\|_F^2\) as follows: At iteration \(t\geq 0\), we first sample \( \{ x_{p,\ell}, \Nfr^{\ell} = \{x_{n,s,\ell}\}_{s\in[\mathbf{N}]} \}_{\ell \in [K]} \) independently for \(K = \poly(d)\) many batches of data.\footnote{Note that such an assumption of sampling from populations without resorting to a finite dataset is very close to reality, in that the unlabeled data are much cheaper to obtain as opposed to labeled data used in supervised learning. Indeed, \cite{he2020momentum} have used an unlabeled dataset of one billion images, which is larger than any labeled datasets in vision.} We augmented all the positives as \(x_{p,\ell}^{+}, x_{p,\ell}^{++} \gets \RandomMask(x_{p,\ell})\) as defined in \myref{def:random-mask}{Def.}. And regroup the data into \(K\) batches of data of the form \(\{\Bfr^{\ell}\}_{\ell\in[K]} = \{\{x_{p,\ell}^{++}\}\cup\Nfr^{\ell}\}_{\ell\in[K]}\). Now we evaluate the empirical loss and gradient as 
\begin{itemize}
    \item \textbf{empirical objective}: \(\widehat{\Obj}(f_t) = \frac{1}{K}\sum_{\ell\in[K]} \mathcal{L}(f_t,x_{p,\ell}^+,\Bfr^{\ell}) + \lambda \|w^{(t)}\|_F^2\);
    \item \textbf{empirical gradient of weight \(w_i\)}: \(\nabla_{w_i}\widehat{\Obj}(f_t) = \frac{1}{K}\sum_{\ell\in[K]} \nabla_{w_i}\mathcal{L}(f_t,x_{p,\ell}^+,\Bfr^{\ell}) + \lambda w_i^{(t)} \)
\end{itemize}
and we update weights \(\{w_i^{(t)} \}_{i\in[m]}\) at each iteration \(t \geq 0\) as follows:
\begin{displaymath}
    w_i^{(t+1)} \gets w_i^{(t)} - \nabla_{w_i}\widehat{\Obj}(f_t) = (1 - \lambda)w_i^{(t)} - \frac{1}{K}\sum_{\ell\in[K]} \nabla_{w_i}\mathcal{L}(f_t,x_{p,\ell}^+,\Bfr^{\ell})
\end{displaymath}
Note that as long as \(\|w^{(t)}\|_F^2 = \sum_{i\in[m]}\|w_i^{(t)}\|_2^2 \leq \poly(d)\), the following fact always holds (which can be easily obtained by Bernstein concentration, and note that we do not need to use uniform convergence):
\begin{fact}[approximation of populatiion gradients by empirical gradients]
    As long as \(\|w^{(t)}\|_F^2 \leq \poly(d)\), there exist \(K = \poly(d)\) such that the following inequality holds with high probability for all iteration \(t\):
    \begin{displaymath}
        \left\|\nabla_{w_i}\widehat{\Obj}(f_t) - \nabla_{w_i}\Obj(f_t)\right\|_2 \leq \frac{\|w_i^{(t)}\|_2}{\poly(d_1)}\quad \text{ for all } i\in[m]
    \end{displaymath}
\end{fact}
To compute the gradient of our loss function \(\mathcal{L}(f_t,x_{p,\ell}^+ ,\Bfr^{\ell})\) with respect to the weights \(\{w_i^{(t)}\}_{i\in[m]}\), we define the following notations: positive logit \(\ell'_{p,t}(x_{p}^+,\Bfr) \) and negative logits \(\ell'_{s,t}(x_{p}^+,\Bfr)\):
\begin{displaymath}
    \ell'_{p,t}(x_{p}^+,\Bfr)  := \frac{e^{ \Sim_{f_t}(x_p^+,x_{p}^{++})/\tau}}{\sum_{x \in\Bfr} e^{ \Sim_{f_t}(x_p^+,x)/\tau}}\qquad \ell'_{s,t}(x_{p}^+,\Bfr)  := \frac{e^{ \Sim_{f_t}(x_{p}^+,x_{n,s})/\tau}}{\sum_{x \in\Bfr} e^{\Sim_{f_t}(x_{p}^+,x)/\tau}}
\end{displaymath}
Then the empirical gradient of \(L(f_t)\) with respect to weight \(w_i^{(t)}\) at iteration \(t\) can be expressed as (recall that we have used \(\StopGrad\) operation in our similarity measure \(\Sim_{f_t}\)):
\begin{align*}
    \nabla_{w_i}L(f_t) = \E\left[ (1 - \ell'_{p,t})\cdot h_{i,t}(x_p^{++})\1_{|\vbrack{w_i^{(t)},x_p^{+}}|\geq b_i^{(t)}}x_p^{+} + \sum_{x_{n,s}\in\Nfr}\ell'_{s,t}\cdot h_{i,t}(x_{n,s})\1_{|\vbrack{w_i^{(t)},x_p^+}|\geq b_i^{(t)} }x_p^+ \right]
\end{align*}
These notations will be frequently used in our proof in later sections.

\paragraph{Global Notations.} We define some specific notations we will use throughout the proof. 
\begin{itemize}
    \item We let \(C_z\) be the constant inside the \(\Theta\) notation of \(\Pr(|z_j|=1) = \Theta(\frac{\log \log d}{d}) = \frac{C_z\log\log d}{d}\) defined in \myref{def:sparse-coding}{Definition}.
    \item We assume in our paper all the \(\polylog(d)\) explicitly written to be much bigger than the \(\polylog(d)\) factors in our \(\widetilde{O}(\cdot)\) notations.
    \item For any \(j \in [d]\) or set \(\S \subset [d]\), let \(x = \sum_{j\in[d]}\M z_j + \xi \sim \D_x\) be an input, we denote the superscript \(^{\setminus j}\) (or \(^{\setminus \S}\)) as an operation to subtract features \(\M_j\) (or \(\{\M_j\}_{j\in\S}\)) in the data as follows:
    \begin{displaymath}
        x^{\setminus j}: = x - \M_j z_j \quad\text{or}\quad x^{\setminus \S}: = x - \sum_{j\in \S}\M_j z_j
    \end{displaymath}
    Furthermore, for augmented input \(x^+ = 2\DD x\) (or \(x^{++} = 2(\Id-\DD)x\)), we also define:
    \begin{align*}
        x^{,\setminus j}: = x^+ - \M_j z_j \quad\text{or}\quad x^{+,\setminus \S}: = x - \sum_{j\in \S}\M_j z_j, \tag{and similarly for \(x^{++}\)}
    \end{align*}
\end{itemize}

\subsection{The Initial Stage of Training: Initial Feature Decoupling}

The initial stage of our training process is defined as the training iterations \(t \leq T_1\), where \(T_1 = \Theta\left(\frac{d\log d}{\eta \log\log d}\right)\) is the iteration when all \(\|w_i^{(t)}\|_2^2\geq \frac{2(1+c_0)+c_1}{c_1} \|w_i^{(0)}\|_2^2\). Before \(T_1\), the learning of our neurons focus on emphasizing the entire subspace of sparse features (i.e., focus on learning \(\M\)), which is enabled by our augmentations \(\RandomMask\) and \textbf{feature decoupling} principle. 

More formally, we will investigate for each neuron \(i \in [m]\), how the features (the weights) grow at each directions. For the sparse features \(\{\M_j\}_{j \in [d]}\), when there is no bias, we shall prove that:
\begin{displaymath}
     \vbrack{\nabla_{w_i}\Obj(f_t), \M_j} \approx \vbrack{w_i^{(t)},\M_j}\cdot \E[z_j^2]
\end{displaymath}
which give exponential rate of growth for our (subspace of) sparse features. Indeed, such an exponential growth will continue to hold until the bias have been pushed up by the negative samples\footnote{We decide to leave the analysis of how the biases are trained open and based our analysis on manually growing biases. The exact mechanism of bias growth depend on the effects of positive-negative contrast in the second stage, which would significantly complicated our analysis.} or until the end of training where the gradient is cancelled by positive-negative contrast. Meanwhile, for the spurious dense features \(\{\M^{\perp}_j\}_{[d_1]\setminus [d]}\), we will prove for each \(j \in [d_1]\setminus [d]\), at iterations \(t \leq T_1\) (also note that \(\lambda = \frac{\polylog(d)}{\sqrt{d_1}}\)):
\begin{displaymath}
    \vbrack{\nabla_{w_i}\Obj(f_t), \M^{\perp} _j} \approx -\lambda \vbrack{w_i^{(t)},\M^{\perp}_j} + \widetilde{O}\left(\frac{\|w_i^{(t)}\|_2}{d_1}\right)
\end{displaymath}
which is only possible because we have used the augmentation \(\RandomMask\). Without such augmentation, we shall expect the growth rate of \(\vbrack{w_i^{(t)},\M^{\perp}_j}\) to be approximately the same with \(\vbrack{w_i^{(t)},\M_j}\), which would collapse the featrues learned in our neural nets. As the training proceeds, we will prove that:
\begin{align*}
    \|\M \M^{\top}w_i^{(t)}\|_2 &\approx \|\M \M^{\top}w_i^{(0)}\|_2 \cdot\left(1 + O\left(\frac{\eta\log\log d}{d}\right)\right)^{t} \\
    \|\M^{\perp}(\M^{\perp})^{\top}w_i^{(t)}\|_2 &\approx \|\M^{\perp}(\M^{\perp})^{\top}w_i^{(0)}\|_2 + o\left(\frac{\|\M^{\perp}(\M^{\perp})^{\top}w_i^{(0)}\|_2}{\poly(d)}\right)
\end{align*}
Therefore, after sufficient iterations (not many comparing to the total training time), the weights \(\{w_i^{(t)}\}_{i\in[m]}\) of neurons will mostly consist of the sparse features \(\{\M_j\}_{j\in[d]}\) rather than the dense features \(\{\Mperp_j\}_{j\in[d_1]\setminus [d]}\). We can then move into the second stage of training, where we tune the bias to simulate the sparsification process.

\subsection{The Second Stage of Training: Singletons Emerge}

After the initial training stage \(t\leq T_1\), we enter the second stage of training, where we will analyze how the growth of bias drive the neurons to become singletons. However, the crucial challenge here is that as soon as the bias \(b_i^{(t)}\) start to grow above zero. The correlations between the sparse features and dense features will emerge to obfuscate our analysis of gradients. Indeed, mathematically we can formulate the problem as follows: Let \(b_i^{(t)} \geq 0\), how can we obtain a bound of the following term, which cannot exceed \(\vbrack{\nabla_{w_i}\Obj(f_t),\M_j}\) as calulated above:
\begin{align*}
    \vbrack{\nabla_{w_i}\Obj(f_t), \M^{\perp}_j} &\approx \E\left[  h_{i,t}(x_p^{++})\vbrack{\nabla_{w_i}h_{i,t}(x_p^{+})\M^{\perp}_j} \right]\\
    &\approx \E\left[  h_{i,t}(2(\Id-\DD) x_p)\1_{|\vbrack{w_i^{(t)},2\DD(\M z_p +\xi_p) }|\geq b_i^{(t)} } \vbrack{2\DD\xi_p,\M^{\perp}_j} \right]
\end{align*}
The difficulty, as opposed to what we saw in the initial stage, is that now there is a chain of correlations transmitted through the following line, when \(b_{i}^{(t)}> 0\):
\begin{itemize}
    \item the term with the masked sparse signals \(2 (\Id-\DD)\M z_p\) and the term with the masked dense signals \(2(\Id - \DD)\xi_p\) in the activation \(h_{i,t}(x_p^{+})\) are positively correlated;
    \item the term with the masked sparse signals \(2 (\Id-\DD)\M z_p\) in the activation \(h_{i,t}(x_p^{+})\) and the term with the masked sparse signals \(2(\Id - \DD)\M z_p\) in the gradient \(\nabla_{w_i}h_{i,t}(x_{p}^{++})\) are positively correlated;
    \item the term with the masked sparse signals \(2 \DD \M z_p\) and the term with the masked dense signals \(2\DD\xi_p\) in the gradient \(\nabla_{w_i} h_{i,t}(x_p^{+})\) are positively correlated.
\end{itemize}
This chain of correlations will significantly complicated our analysis, we will prove the following lemma that have taken into considerations all the factors affecting the gradients:
\begin{lemma}[sketched]
For neuron \(i \in [m]\) and spurious dense feature \(\M^{\perp}_j\), we have
\begin{displaymath}
    \E[(1 - \ell'_{p,t})h_{i,t}(x_p^{++})\vbrack{\nabla_{w_i}h_{i,t}(x_p^+), \M^{\perp}_j}] \approx \left(\vbrack{w_i^{(t)}, \M^{\perp}_j} \pm \widetilde{O}\left(\frac{\|w_i^{(t)}\|_2}{\sqrt{d_1}}\right)\right)\E[\vbrack{\xi,\Mperp_j}^2\1_{|\vbrack{w_i^{(t)},x_p}|\geq b_i^{(t)}}]
\end{displaymath}    
\end{lemma}
However, we can prove that after the initial stage, both \(\vbrack{w_i^{(t)}, 2\DD\M^{\perp}_j}\) and \(\widetilde{O}\left(\|w_i^{(t)}\|_2/\sqrt{d_1}\right)\) are very small compared to the sparse features \(\vbrack{w_i^{(t)},\M_j}\). Thus \(\E[\vbrack{\xi, \M^{\perp}_j}^2\1_{|\vbrack{w_i^{(t)},x_p}|\geq b_i^{(t)}} ]\) shall be somehow small since the correlation between \(\vbrack{\xi, \M^{\perp}_j}^2\) and \(\1_{|\vbrack{w_i^{(t)},x_p}|\geq b_i^{(t)}}\) is small.

On the contrary, for some of the features \(\vbrack{w_i,\M_j}\) that is ``lucky'' in the sense that at initialization \(\vbrack{w_i^{(0)} ,\M_j}^2 \geq (2.02)\sigma_0^2\log d\), we can maintain such ``luckiness'' till the stage II and obtain similar gradient approximation as:
\begin{lemma}[sketched]
    For neuron \(i \in [m]\) and the ``lucky'' sparse feature \(\M_j\), we have
    \begin{displaymath}
        \E[h_{i,t}(x_p^{++})\vbrack{\nabla_{w_i}h_{i,t}(x_p^+), \M_j}] \approx \left(\vbrack{w_i^{(t)}, \M_j} \pm \widetilde{O}\left(\frac{\|w_i^{(t)}\|_2}{\sqrt{d_1}}\right)\right)\E[z_j^2\1_{|\vbrack{w_i^{(t)},x_p}|\geq b_i^{(t)}} ]
    \end{displaymath}
\end{lemma}
And we also have two observations: (1) \(\vbrack{w_i^{(t)},\M_j}\) is almost as large as \(\Theta(\|w_i^{(t)}\|_2/\sqrt{d}) \), which is much larger than \(\widetilde{O}(\|w_i^{(t)}\|_2/\sqrt{d_1})\); (2) we know \(\E[z_j^2\1_{|\vbrack{w_i^{(t)},x_p}|\geq b_i^{(t)}} ] \gg o(1/d)\) since when the sparse feature \(\vbrack{w_i^{(t)},\M_j}z_j\) is active, it would be much larger than the dense feature \(\vbrack{w_i^{(t)},\M^{\perp}_{j'}}\). This pave the way for our feature growth till stage III. Our theory indeed matches what happens in practice, where one can observe the slow emergence of features (in the first layer of AlexNet) during the training process comparing to supervised learning.

\subsection{The Final Stage of Training: Convergence to Sparse Features}

We assume our training proceeds until we reach at least \(T_3 = \frac{\poly(d)}{\eta}\), but the stage III start at some \(T_2 = \Theta\left(\frac{d\log d}{\eta \log\log d}\right) \) when there exist a neuron \(i \in [m]\) such that \(\|w_i^{(T_2)} \|_2 \geq d\|w_i^{(t)}\|_2\). At iterations \(t \geq T_2\), the negative term will begin to cancel the positive gradient, which drives the learning process to converge. We now sketch the proof here.

When the training process reach \(t \geq T_2\), we have the following properties for all the neurons:
\begin{itemize}
    \item For each \(j \in [d]\), there is a set \(\mathcal{M}_j = \{i \in [m], |\vbrack{w_i^{(t)},\M_j}| \geq \Theta(1) \max_{j'\in[d]} |\vbrack{w_i^{(t)},\M_{j'}}| \}\) such that this set has cardinality \(|\mathcal{M}_j| \geq \omega(1)\). 
    \item For all the neurons \(i\in[m]\) such that \(i \notin \mathcal{M}_j\), we have \(|\vbrack{w_i^{(t)},\M_j}| \leq o(\frac{\|w_i^{(t)}\|_2}{\sqrt{d}})\);
    \item the neuron activations can be written as follows with high probability, which is because now the neurons are truly sparse and can be written as decompositions of \(O(1)\) many signals (plus some small mixture): \(h_{i,t}(x) \approx \vbrack{w_i^{(t)},x}\1_{z_j\neq 0 \text{ for } j\in[d]: i\in\mathcal{M}_j} + \widetilde{O}(\|w_i^{(t)}\|/d^2)\).
\end{itemize}

At this stage, for each \(j\in[d]\), the gradient \(\vbrack{\nabla_{w_i}L_{\pos}(f_t),\M_j}\) does not change much compared to the previous stage, but the negative term changed essentially, which we elaborate as follows:
\begin{align*}
    \sum_{x_{n,s} \in \Nfr}\ell'_{s,t} h_{i,t}(x_{n,s}) \approx \sum_{x_{n,s} \in \Nfr}\ell'_{s,t}\sum_{j\in[d]: i\in\mathcal{M}_j}\vbrack{w_i^{(t)},\M_j}z_{n,s,j}
\end{align*}
For the simplest case where there is only one \(j\in[d]\) such that \(i \in \Mcal_j\), we can see that 
\begin{displaymath}
    \sum_{x_{n,s} \in \Nfr}\ell'_{s,t} h_{i,t}(x_{n,s}) \approx \sum_{x_{n,s} \in \Nfr}\ell'_{s,t} \vbrack{w_i^{(t)},\M_j}z_{n,s,j}\1_{z_{n,s,j}\neq 0}
\end{displaymath}
The critical question here is that: The problem at head is extremely non-convex, how does our algorithm find the minimal of the loss without being trapped in some undesired solutions. We argue that as long as the trajectory of weights following SGD is good in the sense that only ``good'' features are picked up. The SGD in the final stage will point to the desired solution, then the singletons of our sparse feature \(\M_J\) will converge as follows:
\begin{displaymath}
    \vbrack{\nabla_{w_i}L(f_t),\M_j } \approx \E\left[ \vbrack{w_i^{(t)},\M_j}z_{p,j}^2 - \sum_{x_{n,s} \in \Nfr}\ell'_{s,t} \vbrack{w_i^{(t)},\M_j}z_{n,s,j}z_{p,j}\1_{z_{p,j}\neq 0, z_{n,s,j}\neq 0} \right] \approx 0
\end{displaymath}
While the graident of other features (including sparse features not favored by the specific neuron \(i\in[m]\), and the spurious dense features \(\Mperp_j\)) will be smaller to ensure sparse representations. More formal arguments will be presented in the later sections.

\subsection{Without Augmentations, Dense Features Are Preferred}

Now we turn to the case where no augmentations are used. In this scenario, for any ense feature \(\M^{\perp}_j\), we always have:
\begin{displaymath}
    \vbrack{\nabla_{w_i}\Obj(f_t), \M^{\perp}_j} \approx \vbrack{w_i^{(t)},\Mperp_j}\cdot \Theta\left(\frac{\sqrt{\log d}}{d}\cdot \Pr(|h_{i,t}(x_p)|\geq 0)\right)
\end{displaymath}
where \(\Pr(|h_{i,t}(x_p)|\geq 0) \approx \Pr(|h_{i,t}(\xi_p)|\geq 0)\)
which is approximately equal to the growth rate of sparse signals when no augmentations are used. In this case the sparse signals will not be emphasized during any stage of training. And more over, we can easily verify the following condition: at each iteration \(t \geq 0\), we have \(\|\Mperp(\Mperp)^{\top}w_i^{(t)}\|_2^2 \approx (1 - \frac{1}{\poly(d)})\|w_i^{(t)}\|_2^2\) based on similar careful characterization of the learning process. Moreover, such learning process can easily converge to low loss: when \(\sum_{i\in[m]}\|w_i^{(t)}\|_2^2 \geq \Omega(\tau \log d)\), we simply have 
\begin{align*}
    \vbrack{f_t(x_p^+),f_t(x_p^{++})} &\approx \Omega(\tau \log d)  \quad \text{and}\quad \vbrack{f_t(x_p^+),f_t(x_{n,s})} &\approx \widetilde{O}(\frac{1}{d})\tag{with high probability}
\end{align*}
Using this characterization, we immediately obtain the loss (and gradient) convergence.

\section{Some Technical Lemmas}

\subsection{Characterization of Neurons}

In this section we give some definitions and lemmas that characterize the neurons at initialization and during the training process. We choose \(c_1=2 + 2(1-\gamma)c_0 , c_2=c_1 - \gamma c_0\) be two constants. (which we choose \(\gamma \in (0,\frac{1}{100})\), similar to the choice in \cite{allen-zhu2020feature}).

\begin{definition}\label{def:neuron-sets}
    We define several sets of neurons that will be useful for the characterization of the stochastic gradient descent trajectory in later sections.
    \begin{itemize}
        \item For each \(j \in [d]\), we define the set \(\Mcal_j \subseteq [m] \) of neurons as:
        \begin{displaymath}
            \Mcal_j := \left\{ i \in [m] : \vbrack{w_i^{(0)},\M_j}^2 \geq \frac{c_2\log d}{d}\| \M\M^{\top}w_i^{(0)}\|_2^2  \right\}
        \end{displaymath}
        \item For each \(j \in [d]\), we define the set \(\Mcal^{\star}_{j} \subseteq [m] \) of neurons as:
        \begin{align*}
            \Mcal^{\star}_{j}  := &\left\{ i \in [m] : \vbrack{w_i^{(0)},\M_j}^2 \geq \frac{c_1\log d}{d}\| \M\M^{\top}w_i^{(0)}\|_2^2,  \right. \\
            &\left. \qquad\qquad\ \vbrack{w_i^{(0)},\M_{j'}}^2 \leq \frac{c_2\log d}{d}\| \M\M^{\top}w_i^{(0)}\|_2^2,\quad  \forall j' \in [d], j' \neq j \right\}
        \end{align*}
    \end{itemize}

\end{definition}

\paragraph{Properties at initialization:} At initialization, where \(t = 0\), we need to give several facts concerning our neurons, which will later be useful for the analysis of SGD trajectory.

\begin{lemma}\label{lem:property-init}
    At iteration \(t = 0\), the following properties hold:
    \begin{enumerate}
        \item[(a)] With high probability, for every \(i \in [m]\), we have
        \begin{displaymath}
            \|w_i^{(0)}\|_2^2 \in \left[ \sigma_0^2 d_1\left(1 - \widetilde{O}(\frac{1}{\sqrt{d_1}}) \right) , \sigma_0^2 d_1\left(1 + \widetilde{O}(\frac{1}{\sqrt{d_1}})  \right)\right]
        \end{displaymath}
        \item[(b)] With high probability, for every \(i \in [m]\), we have
        \begin{displaymath}
            \|\M\M^{\top} w_i^{(0)}\|_2^2 \in \left[ \sigma_0^2 \left(1 - \widetilde{O}(\frac{1}{\sqrt{d}}) \right) , \sigma_0^2d \left(1 + \widetilde{O}(\frac{1}{\sqrt{d}}) \right)\right]
        \end{displaymath}
        \item[(c)] With probability at least \(1 - o(\frac{1}{d^4})\), we have for each \(j \in [d]\):
        \begin{displaymath}
            |\Mcal^\star_{j}| \geq \Omega(d^{\gamma c_0/4}) =: \Xi_1,\qquad |\Mcal_j| \leq O(d^{2\gamma c_0}) =: \Xi_2;
        \end{displaymath}
        \item[(d)] For each \(i \in [m]\), let  \(\Lambda_i := \left\{ j \in [d]:\ |\vbrack{w_i^{(0)},\M_j}| \leq \sigma_0/d \right\}\subseteq [d]\), then \(|\Lambda_i| = O(\frac{d}{\polylog(d)})\);
        \item[(e)] For any \(j' \neq j\), \(|\Mcal_{j'}\cap \Mcal_j| \leq O(\log d) \), with probability at least \(1 - o(1/d^4)\).
        \item[(f)] For each \(i \in [m]\), there are at most \(O(1)\) many \(j \in [d]\) such that \(i \in \Mcal_j\), and at most \(O(2^{-\sqrt{\log d}}d)\) many \(j \in [d]\) such that \(|\vbrack{w_i^{(0)},\M_j}| \geq \Omega(\sigma_0\log^{1/4} d)\).
    \end{enumerate}
\end{lemma}

\begin{proof}
    The proof of (a)--(b) can be derived from simple concentration of chi-squared concentration. The proof of (c) and (f) follow from \cite[Lemma B.2]{allen-zhu2020feature}.  For (d) it suffices to use basic Gaussian anti-concentration around the mean. For (e) it suffices to use a simple Bernoulli concentration.
\end{proof}

\subsection{Activation Size and Probability}

\begin{lemma}[correlation from augmentation]\label{lem:corr-aug}
    Let \(\Mperp \in \R^{d_1\times(d_1-d)}\) be an orthonormal complement of \(\M\), \(\DD \sim\D_{\DD}\), it holds:
    \begin{enumerate}
        \item for each \(j,j'\in[d]\), with high probability we have
        \begin{displaymath}
            \vbrack{\M_j,(\Id - 2\DD)\M_{j'}} \lesssim \widetilde{O}\left(\frac{1}{\sqrt{d_1}}\right)
        \end{displaymath}
        \item for each \(j\in[d] ,j'\in[d_1]\setminus [d]\), with high probability we have 
        \begin{displaymath}
            \vbrack{\M_j,(\Id - 2\DD)\Mperp_{j'}} \lesssim \widetilde{O}\left(\frac{1}{\sqrt{d_1}}\right)
        \end{displaymath}
        Since \(\vbrack{\M_j,\Mperp_{j'}} = 0\), this bound also hold for variables \(\vbrack{\M_j,(\Id - \DD)\Mperp_{j'}}\) and \(\vbrack{\M_j,\DD\Mperp_{j'}}\).
    \end{enumerate}
\end{lemma}

\begin{proof}
    \begin{enumerate}
        \item For \(j\in[d]\) and \(r \in [d_1]\), we denote \((\M_j)_r\) to be the \(r\)-th coordinate of \(\M_j\), Now we expand 
        \begin{displaymath}
            \vbrack{\M_j,(\Id-2\DD)\M_{j'}} = \sum_{r\in[d_1]}(\M_{j})_r(\M_{j'})_r(1 - 2\DD_{r,r})
        \end{displaymath}
        which can be view as a sub-Gaussian variables with variance parameter \(\widetilde{O}(1/d_1)\). Applying Chernoff bound concludes the proof.
        \item Proof is similar to (1). Since the mean of \(\vbrack{\M_j,(\Id - 2\DD)\Mperp_{j'}}\) is zero, we can compute its variance as
        \begin{align*}
            \E\left[\vbrack{\M_j,(\Id - 2\DD)\Mperp_{j'}}^2\right] &= \sum_{r\in[d_1]}(\M_{j})^2_r(\Mperp_{j'})^2_r\E[(1 - 2\DD_{r,r})^2] \\
            &\leq \widetilde{O}\left(\frac{1}{d_1}\right) \sum_{r\in[d_1]}(\Mperp_{j'})^2_r \leq \widetilde{O}\left(\frac{1}{d_1}\right)
        \end{align*}
        then again from Chernoff bound we conclude the proof.
    \end{enumerate}
\end{proof}

we will give several probability tail bounds for the so defined variables, which will be used in the computations of the training process throughout our analysis.
\begin{lemma}[pre-activation size]\label{lem:activation-size}
    Let \(x  = \M z + \xi \sim \D_x\), \(w_i \in \R^{d_1}\) and \(\DD\sim\D_{\DD}\). Denoting \(x^{\setminus j} = \sum_{j'\neq j, j'\in[d]}\M_{j'}z_{j'} + \xi\), we have the following results:
    \begin{enumerate}
        \item For any \(\lambda > 0\), we have
        \begin{displaymath}
            \Pr_{\DD\sim\D_{\DD}}\left(|\vbrack{w_i,(\Id-2\DD)\M_j} | > \lambda\|w_i\|_2\|\M_j\|_{\infty}\right) \leq 2e^{-\Omega(\lambda^2)}
        \end{displaymath}
        \item (naive Chebychev bound) For any \(\lambda > 0\) and \(z \in [-1,1]\), we have 
        \begin{displaymath}
            \Pr_{z^{\setminus j},\xi,\DD} \left((\vbrack{w_i,(\Id-2\DD)x^{\setminus j}} + \vbrack{w_i,(\Id-2\DD)\M_j} |z| )^2 > \frac{\lambda\|w_i\|_2^2\sqrt{\log d}}{d} \right) \leq O\left(\frac{1}{\lambda}\right) 
        \end{displaymath}
        The same tail bound holds for variables \(\vbrack{w_i,x}\), \(\vbrack{w_i,(\Id-2\DD)x}\) and \(\vbrack{w_i,\xi}\) as well.
        \item (high probability bound for sparse signal)
        \begin{align*}
            \Pr\Big( \vbrack{w_i,(\Id-2\DD)\M z}^2 >   \|w_i\|_2^2\cdot \max_{j\in[d]}\|\M_j\|_{\infty}^2 \log^4 d \Big) \lesssim e^{-\Omega(\log^2 d)}
        \end{align*}
        \item (high probability bound for dense signal) Let \(Z = \vbrack{w_i,(\Id-2\DD)\xi}\) or \(Z = \vbrack{w_i,\xi}\), we have
        \begin{displaymath}
            \Pr\left( Z^2 \geq \frac{\|w_i\|_2^2\log^4 d}{d} \right) \lesssim e^{-\Omega(\log^2 d)}
        \end{displaymath}
    \end{enumerate}
\end{lemma}

\begin{proof}
    \begin{enumerate}
        \item From the fact that \(\DD = (\DD_{k,k})_{k=1}^{d_1}\), where \(\DD_{k,k}\sim\mathrm{Bernoilli}(\frac{1}{2}) \) are subgaussian variables, we can use the subgaussian tail coupled with Hoeffding's bound to conclude.
        \item Since the mean of \(\vbrack{w_i,(\Id-2\DD)x}\) is zero, we can simply compute the variance as
        \begin{align*}
            &\E\left[(\vbrack{w_i,(\Id-2\DD)x^{\setminus j}}+\vbrack{w_i,(\Id-2\DD)\M_j}|z|)^2\right] \\
            \leq \ & \sum_{s = 1}^{d_1}(w_i)_s^2(\M_{j'})_s^2 \E\left[\left(\sum\nolimits_{ j'\neq j}z_{j'}^2+\xi_s^2\right)(2\DD_{s,s} - 1)^2\right] + \sum_{s = 1}^{d_1}(w_i)_s^2(\M_j)_s^2 \E[z^2(2\DD_{s,s} - 1)^2]\\
            \leq \ & \widetilde{O}\left(\|w_i\|_2^2/d\right)
        \end{align*}
        Now we can use Chebychev's inequality to conclude. As to the tail bounds for other variables, it suffices to go through some similar calculations.
        \item First we consider for each \(j \in [d]\), the variable \(Z_j := \vbrack{w_i,(2\DD-\Id)\M_j }\). Note that 
        \begin{displaymath}
            Z_j = \sum_{s = 1}^{d_1}(w_i)_s(\M_j)_s(2\DD_{s,s}-1)
        \end{displaymath}
        is a sum of subgaussian variables, each with variances \(O((w_i)_s^2/d_1)\), therefore by using Hoeffding's bound we have with prob \(\geq 1-e^{-\Omega(\log^2 d)}\), we have \(|Z_j|\leq \widetilde{O} \left(\frac{\|w_i\|_2}{\sqrt{d_1}}\right)\), Now by using a union bound, we conditioned on \(\{|Z_j|\leq \widetilde{O} \left(\frac{\|w_i\|_2}{\sqrt{d_1}}\right) \text{ for all } j \in [d]\}\) happending, which is still with high probability. We use Bernstein's inequality to show that with high probability over \(z\) and \(\DD\) it holds that \(\left|\sum_{j\in[d]}Z_jz_j\right|^2 \lesssim \widetilde{O} \left(\frac{\|w_i\|_2^2}{d_1}\right) \). 
        \item We can first obtain for high probability bounds for each coordinates \(\xi_j, j\in [d_1]\) based on the concentration of Gaussian variables, and then use Chernoff bound via the randomness of \(\DD\) to conclude (when there is no \(\DD\) involved, the claim is obvious). 
    \end{enumerate}
\end{proof}

\begin{lemma}[pre-activation size, II]\label{lem:activation-size-2}
    Let \(i \in [m]\). Suppose the following holds: 
    \begin{itemize}
        \item \(\vbrack{w_i^{(t)},\M_j}^2 \geq \Omega((b_i^{(t)})^2)\) for no more than \(O(1)\) many \(j \in [d]\);
        \item \(\vbrack{w_i^{(t)},\M_j}^2 \geq \Omega(\frac{(b_i^{(t)})^2}{\sqrt{\log d}})\) for no more than \(O(e^{-\Omega(\sqrt{\log d})}d)\) many \(j \in [d]\);
        \item \(\|w_i^{(t)}\|_2^2\leq O(\frac{d(b_i^{(t)})^2}{\log d})\).
    \end{itemize} 
    Then for any \(\lambda \geq 0.0001\):
    \begin{align*}
        \Pr( |\vbrack{w_i^{(t)},x_p^+}| \geq \lambda  b_i^{(t)} ) \lesssim e^{-\Omega(\log^{1/4}d)},\qquad \Pr( |\vbrack{w_i^{(t)},x_p}| \geq \lambda  b_i^{(t)} ) \lesssim e^{-\Omega(\log^{1/4}d)}
    \end{align*}
\end{lemma}

\begin{proof}
    Our proof follows from similar arguments in \cite{allen-zhu2020feature}. The only difference here is that we have applied \(\RandomMask\) augmentation to our data \(x_p^+\) and \(x_p^{++}\). We only need to consider two terms:
    \begin{itemize}
        \item The augmented noise \(\vbrack{w_i^{(t)},2\DD\xi_p}\), which follows from Gaussian distribution with variance \(O(\|w_i^{(t)}\|_2^2\sigma_{\xi}^2)\). We have that for some small constant \(c \ll 0.0001\): \(|\vbrack{w_i^{(t)},2\DD\xi_p}| \leq c b_i^{(t)}\);
        \item The augmented sparse signals \(\vbrack{w_i^{(t)},2\DD\M z_p} = \vbrack{w_i^{(t)},\M z_p} + \vbrack{w_i^{(t)},(2\DD-\Id)\M z_p} \). Here the bound for the first term on RHS can be obtained via similar approach in \cite{allen-zhu2020feature}, the bound for the second term follows from \myref{lem:activation-size}{Lemma}.
    \end{itemize}
    They conclude the proof.
\end{proof}

\begin{lemma}[pre-activation size, III]\label{lem:activation-size-3}
    Let \(i \in [m]\). Suppose the following holds: there exist a set \(\N_i \subseteq [d]\) such that \(|\N_i| = O(1)\), and
    \begin{itemize}
        \item \(\vbrack{w_i^{(t)},\M_j}^2 \leq O(\frac{(b_i^{(t)})^2}{\polylog(d)})\) for \(j \notin \N_i\);
        \item \(\|w_i^{(t)}\|_2^2\leq O(\frac{d(b_i^{(t)})^2}{\polylog (d)})\).
    \end{itemize} 
    Then for any \(\lambda \in [0.01,0.99]\), 
    \begin{align*}
        \Pr\left[ | \sum_{j\notin\N_i}\vbrack{w_i^{(t)},2\DD\M_j}z_j + \vbrack{w_i,2\DD\xi}| \geq \lambda  b_i^{(t)} \right] \lesssim e^{-\Omega(\log^2d)}
    \end{align*}
\end{lemma}

\begin{proof}
    The proof is similar to those of \myref{lem:activation-size-2}{Lemma} above and the proof of \cite[Lemma C.3]{allen-zhu2020feature}.
\end{proof} 

\section{Stage I: Initial Feature Growth}

In this section we analyze the training process at the initial stage. Here we define the stage transition time \(T_1 = \Theta(\frac{d\log d}{\eta \log\log d})\) to be the iteration when \(\|w_i^{(t)}\|_2^2\geq \frac{2(1+\gamma c_0)+c_1}{c_1} \|w_i^{(0)}\|_2^2 \) for all the neurons \(i \in [m]\) (where \(c_0\) is a small constant defined in \myref{lem:property-init}{Lemma}). Indeed, we will characterize the trajectory of weights \(\{w_i^{(t)}\}_{i\in[m]}\) by calculating the growth of \(w_i^{(t)}\)s for all the features \(\{\M_j\}_{j\in[d]}\cup\{\Mperp_j\}_{j\in[d_1]\setminus [d]}\). And also, we keep the bias \(b_i^{(t)} = 0\) at this stage to simplify our analysis. 

We present our theorem of the initial stage below:

\begin{theorem}[Initial feature decoupling]\label{thm:initial-stage}
    At iteration \(t = T_1\), we have the following results:
    \begin{enumerate}
        \item[(a)] \(\|\M\M^{\top}w_i^{(T_1)}\|_2^2 \geq \|w_i^{(T_1)}\|_2^2/2\) for all \(i \in [m]\);
        \item[(b)] For each \(j \in [d]\), and each \(i\in\Mcal^{\star}_j\), we have \(|\vbrack{w_i^{(T_1)},\M_j}| \geq (1 + \gamma c_0)\frac{\sqrt{2\log d}}{\sqrt{d}}\|w_i^{(T_1)}\|_2\);
        \item[(c)] For each \(j \in [d]\), and each \(i\notin\Mcal_j\), we have \(|\vbrack{w_i^{(T_1)},\M_j}| \leq (1 - \gamma c_0)\frac{\sqrt{2\log d}}{\sqrt{d}}\|w_i^{(T_1)}\|_2\);
        \item[(d)] For each \(i \in [m]\), \(|\vbrack{w_i^{(T_1)},\M_j}| \geq \frac{\log^{1/4} d}{\sqrt{d}}\|w_i^{(T_1)}\|_2\), for at most \(O(\frac{d}{2^{\sqrt{\log d}}})\) many \(j \in [d]\).
        \item[(e)] For each \(i \in [m]\) and \(j \in [d_1]\setminus [d]\), we have \(|\vbrack{w_i^{(T_1)},\Mperp_j}| \leq O(\sqrt{\frac{\log d}{d_1}})\|w_i^{(T_1)}\|_2\).
    \end{enumerate}
\end{theorem}

\subsection{Gradient Computations}

Since the each bias \(b_i^{(t)}\) remains at zero during this stage, it is easy to compute the positive gradient for each \(t \leq T_1\) as the following: 

\begin{lemma}[positive gradient, stage I]\label{lem:grad-positive-1}
    Let \(h_{i,t}(\cdot)\) be the \(i\)-th neuron at iteration \(t \leq T_1\) (so that \(b_i^{(t)} = 0\)), then
    \begin{itemize}
        \item[(a)] For each \(j\in[d]\), we have 
        \begin{displaymath}
            \E\left[h_{i,t}(x_p^{++}) \vbrack{\nabla_{w_i} h_{i,t}(x_p^+),\M_j} \right] = \vbrack{w_i^{(t)},\M_j}\E[z_j^2] \pm \widetilde{O}\left(\frac{\|w_i^{(t)} \|_2}{d_1}\right)
        \end{displaymath}
        \item[(b)] For each \(j \in [d_1]\setminus [d]\), we have 
        \begin{displaymath}
            \E\left[h_{i,t}(x_p^{++}) \vbrack{\nabla_{w_i} h_{i,t}(x_p^+),\Mperp_j} \right] = \pm \widetilde{O}\left(\frac{\|w_i^{(t)} \|_2}{d_1}\right)
        \end{displaymath}
    \end{itemize}

\end{lemma}

\begin{proof}
    \begin{enumerate}
        \item[(a)] For each \(j \in [d]\) and \(t \leq T_1\), since \(b_i^{(t)} = 0\) for all \(i \in [m]\), we can calculate
        \begin{align*}
            & \E\left[h_i(x_p^{++})\1_{|\vbrack{w_i^{(t)},x_p^+}|\geq 0}\vbrack{x_p^+,\M_j} \right] \\
            = \ &\E\left[(\vbrack{w_i^{(t)},2(\Id-\DD)(\M z_p +\xi_p)})\1_{|\vbrack{w_i^{(t)},x_p^{++}}|\geq 0}\1_{|\vbrack{w_i^{(t)},x_p^+}|\geq 0}\vbrack{x_p^+,\M_j}\right]
        \end{align*}
        Conditioned on each fixed \(z_p \sim \D_z\) and \(\DD\sim\D_{\DD}\) and use the randomness of \(\xi_p\), we know that events \(\{\vbrack{w_i^{(t)},x_p^+} = 0\}\) and \(\vbrack{w_i^{(t)},x_p^{++}}\) has probability zero. Thus we can get rid of the indicator functions and compute as follows:
        \begin{align*}
            &\E\left[(\vbrack{w_i^{(t)},2(\Id-\DD)(\M z_p +\xi_p)})\1_{|\vbrack{w_i^{(t)},x_p^{++}}|\geq 0}\1_{|\vbrack{w_i^{(t)},x_p^+}|\geq 0}\vbrack{x_p^+,\M_j}\right]\\
            = \ & \E\left[\vbrack{w_i^{(t)},2(\Id-\DD)(\M z_p +\xi_p)} \vbrack{x_p^+,\M_j}\right] \\
            = \ &\E\left[\vbrack{w_i^{(t)},2(\Id-\DD)(\M z_p +\xi_p)} \vbrack{2\DD(\M z_p +\xi_p),\M_j}\right]
        \end{align*}
        From simple observation, conditioned on fixed \(\DD\), we know that \(\DD\xi_p\) and \((\Id-\DD)\xi_p\) are independent and both mean zero, and also \(z_p\) is independent w.r.t. \(\xi_p\) and is mean zero, so we can proceed to compute as
        \begin{align*}
            &\E\left[\vbrack{w_i^{(t)},2(\Id-\DD)(\M z_p +\xi_p)} \vbrack{2\DD(\M z_p +\xi_p),\M_j}\right] \\
            = \ & \E\left[\vbrack{w_i^{(t)},2(\Id-\DD)\M z_p } \vbrack{\DD\M z_p ,\M_j}\right]\\
            = \ &  \E\left[\vbrack{w_i^{(t)},2(\Id-\DD)\M z_p }\left(z_j + \sum_{j'\in[d]}\vbrack{(\Id-2\DD)\M_{j'} ,\M_j}z_{j'}\right)\right]\\
            = \ & \E\left[\sum_{j''\in[d]}\vbrack{w_i^{(t)},2(\Id-\DD)\M_{j''}}z_{p,j''}z_{p,j} \right] \\
            & + \E\left[\vbrack{w_i^{(t)},2(\Id-\DD)\M z_p }\cdot\sum_{j'\in[d]}\vbrack{(\Id-2\DD)\M_{j'} ,\M_j}z_{p,j'}\right]
        \end{align*}
        Now notice that \(z_{p,j''}\) and \(z_{p,j}\) are independent to each other if \(j''\neq j\), we have 
        \begin{align*}
            \E\left[\sum_{j''\in[d]}\vbrack{w_i^{(t)},2(\Id-\DD)\M_{j''}}z_{p,j''}z_{p,j} \right] &= \E\left[\vbrack{w_i,2(\Id-\DD)\M_j}z_{p,j}^2\right]\\
            & = \vbrack{w_i^{(t)},\M_j}\E\left[ z_j^2\right] + \E\left[\vbrack{w_i,(\Id-2\DD)\M_j}z_{p,j}^2\right]\\
            & = \vbrack{w_i^{(t)},\M_j}\E\left[ z_j^2\right]
        \end{align*}
        where in the last equality we have used the fact that \(\Id-2\DD\) is independent to \(z_{p,j}\) and has mean zero. Next by using \myref{lem:corr-aug}{Lemma} and \myref{lem:activation-size}{Lemma} (3), we have the bound
        \begin{align*}
            &\left|\E\left[\vbrack{w_i^{(t)},2(\Id-\DD)\M z_p }\cdot\sum_{j'\in[d]}\vbrack{(\Id-2\DD)\M_{j'} ,\M_j}z_{p,j'}\right]\right|\leq \widetilde{O}\left(\frac{\|w_i^{(t)}\|_2}{d_1}\right)
        \end{align*}
        Combining all results above, we obtain the desired approximation.
        \item[(b)] It is easy to notice that the only difference of this proof with that of (a) is we have
        \begin{displaymath}
            \vbrack{x_p^+,\Mperp_j} = \sum_{j'\in[d]}\vbrack{(\Id - 2\DD)\M_{j'} z_{p,j'}, \Mperp_j} + \vbrack{(\Id-2\DD)\xi_p,\Mperp_j}
        \end{displaymath}
        Now following the same argument as in (1), we can obtain the desired bound.
    \end{enumerate}
    
\end{proof}

In order to analyze the 

\begin{lemma}[logits near initialization]\label{lem:grad-negative-1}
    Letting \(w_i \in \R^{d_1}\) for each \(i\in[m]\), suppose we have \(\sum_{i\in[m]}\|w_i^{(t)}\|_2^2 \leq o(\tau/d)\), then with high probability over the randomness of \(x_p^+,x_p^{++},\Nfr\), it holds:
    \begin{displaymath}
        \left|\ell'_{p,t}(x_p^+,\Bfr) - \frac{1}{|\Bfr|} \right|, \ \left|\ell'_{s,t}(x_p^+,\Bfr) - \frac{1}{|\Bfr|} \right| \leq \widetilde{O}\left(\frac{\sum_{i\in[m]}\|w_i^{(t)}\|_2^2}{\tau|\Bfr|}\right)
    \end{displaymath}
\end{lemma}

\begin{proof}
    For the logit \(\ell_{s,t}(x_p^+,\Bfr)\) of negative sample \(x_{n,s}\), we can simply calculate
    \begin{align*}
        \left|\ell'_{s,t}(x_p^+,\Bfr) - \frac{1}{|\Nfr|} \right| & = \left|\frac{e^{ \Sim_{f_t}(x_p^+,x_{n,s})/\tau}}{\sum_{x \in \Bfr} e^{ \Sim_{f_t}(x_p^+,x)/\tau}} - \frac{1}{|\Nfr|} \right| \\ 
        & = \left| \left(\sum_{x \in\Bfr } e^{  \vbrack{f_t(x_p^+), f_t(x) - f_t(x_{n,s}) }/\tau }\right)^{-1} - \frac{1}{|\Bfr|}\right|\\
        & =  \left| |\Bfr| -  \sum_{x \in \Bfr} e^{ \vbrack{f_t(x_p^+), f_t(x) - f_t(x_{n,s}) }/\tau} \right|\cdot\left( |\Bfr| \cdot \sum_{x\in\Bfr} e^{ \vbrack{f_t(x_p^+), f_t(x) - f_t(x_{n,s}) }/\tau}\right)^{-1}\\
        & \leq \sum_{x \in \Bfr} \left| 1 - e^{ \vbrack{f_t(x_p^+), f_t(x) - f_t(x_{n,s}) }/\tau} \right| \cdot \left( |\Bfr| \cdot \sum_{x \in\Bfr} e^{  \vbrack{f_t(x_p^+), f_t(x) - f_t(x_{n,s}) }/\tau}\right)^{-1} \\
        & \lesssim \max_{x \in \Bfr}\left| \vbrack{f_t(x_p^+), f_t(x) - f_t(x_{n,s}) } \right| \cdot \left( \tau \sum_{x \in\Bfr} e^{  \vbrack{f_t(x_p^+), f_t(x) - f_t(x_{n,s}) }/\tau}\right)^{-1}\\
        & \stackrel{\text{\ding{172}}}{\leq}  \widetilde{O}\left(\frac{\sum_{i\in[m]}\|w_i^{(t)}\|_2^2}{\tau |\Bfr|}\cdot \exp\Bigg(\frac{\sum_{i\in[m]}\|w_i^{(t)}\|_2^2}{\tau}\Bigg) \right)\\
        & \stackrel{\text{\ding{173}}}{\leq} \widetilde{O}\left(\frac{\sum_{i\in[m]}\|w_i^{(t)}\|_2^2}{\tau |\Bfr|}\right)
    \end{align*}
    where \ding{172} is becausewe have \(|1 - e^a| \leq |a|\) for \(a\leq 0.1\), and also with high probability 
    \begin{displaymath}
        |\vbrack{f_t(x_p^+), f_t(x_{n, u}) - f_t(x_{n,s}) }| \leq \widetilde{O}\left(\sum_{i\in[m]}\|w_i^{(t)}\|_2^2\right)
    \end{displaymath}
    and \ding{173} is because \(e^x \leq O(1)\) for \(x\leq 1/2\). The approximation for logit \(\ell'_{p,t}\) of positive sample can be similarly obtained.
\end{proof}

\subsection{The Learning Process at Initial Stage}

In this subsection we will prove, for every neuron \(i \in [m]\), the weights \(w_i\) will mostly ignore the spurious features \(\M^{\perp}\) and learn to emphasize the features \(\M\). Recall that \(T_1 = \Theta(\frac{d\log d}{\eta\log\log d})\) is set to be the time when \(\|w_i^{(t)}\|_2^2\geq \frac{2(1+2\gamma c_0)}{c_1} \|w_i^{(0)}\|_2^2\) for all the neurons \(i \in [m]\), and that such a \(T_1\) is indeed of order $\Theta(\frac{d\log d}{\eta\log\log d})$. 

In order to prove the above theorem, we need the following :

\begin{induct}\label{induct-1}
    The following properties hold for all \(t \leq T_1\):
    \begin{itemize}
        \item[(a)] \(\|\M^{\perp}(\M^{\perp})^{\top}w_i^{(t)}\|_2^2 \leq (1 + O(1/\poly(d)))\|\M^{\perp}(\M^{\perp})^{\top}w_i^{(0)}\|_2^2 \);
        \item[(b)] \(\|\M\M^{\top}w_i^{(t)}\|_2^2 \leq \|\M\M^{\top}w_i^{(0)}\|_2^2\left(1 - \eta\lambda + \frac{\eta C_z\log\log d}{d}\right)^{2t} + O(\frac{1}{d})\|\M\M^{\top}w_i^{(0)}\|_2^2\), moreover, we have \(\|\M\M^{\top}w_i^{(t)}\|_2^2 \leq O(\|w_i^{(0)}\|_2^2)\);
        \item[(c)] \(\|\M\M^{\top}w_i^{(t)}\|_2^2 \geq \|\M\M^{\top}w_i^{(0)}\|_2^2\left(1 - \eta\lambda + \frac{\eta C_z\log\log d}{d}\right)^{2t} - O(\frac{1}{d})\|\M\M^{\top}w_i^{(0)}\|_2^2\)
    \end{itemize}
\end{induct}

\begin{proof}[Proof of \myref{induct-1}{Induction Hypothesis}]
    First we need to work out the exact form of gradient for each feature \(\M_j\) and \(\M^{\perp}_j\). Fix a neuron \(i \in [m]\), for the sparse feature \(\M_j\), \(j \in [d]\), we can write down the SGD iteration as follows:
    \begin{align*}
        \vbrack{w_i^{(t+1)},\M_j} & = \vbrack{w_i^{(t)},\M_j} -\vbrack{\nabla_{w_i}\Obj(f_t), \M_j} \pm \frac{\|w_i^{(t)}\|_2}{\poly(d_1)} \\
        &= (1 - \lambda)\vbrack{w_i^{(t)},\M_j} + \underset{x_p^+,x_p^{++}}\E \left[(1-\ell'_{p,t}(x_p^+,\Bfr))\cdot h_{i,t}(x_p^{++}) \vbrack{\nabla_{w_i} h_{i,t}(x_p^+),\M_j} \right] \\
        & \quad - \sum_{x_{n,s} \in \Nfr}\E\left[\ell'_{s,t}(x_p,\Bfr) h_{i,t}(x_{n,s})\vbrack{\nabla_{w_i}h(x_p^+),\M_j}\right] \pm \frac{\|w_i^{(t)}\|_2}{\poly(d_1)} 
    \end{align*}
    For the positive term \(\E \left[ (1 - \ell'_{p,t}(x_p^+,\Bfr)) h_{i,t}(x_p^{++}) \vbrack{\nabla_{w_i} h_{i,t}(x_p^+),\M_j} \right]\), we can use \myref{lem:grad-positive-1}{Lemma} and \myref{lem:grad-negative-1}{Lemma} to obtain that:
    \begin{displaymath}
        \E \left[(1 - \ell'_{p,t}(x_p^+,\Bfr))\cdot h_{i,t}(x_p^{++}) \vbrack{\nabla_{w_i} h_{i,t}(x_p^+),\M_j} \right] = \vbrack{w_i^{(t)},\M_j}\Pr(|z_j|=1) \pm \widetilde{O}\left(\frac{\|w_i^{(t)}\|_2}{d_1}\right)
    \end{displaymath}
    And for the negative term \(\E\left[\sum_{x_{n,s} \in \Nfr}\ell'_{s,t}(x_p^+,\Bfr) h_{i,t}(x_{n,s})\vbrack{\nabla_{w_i}h(x_p^+),\M_j}\right]\), we can use \myref{lem:grad-negative-1}{Lemma} to bound it as:
    \begin{align*}
        \E\left[\sum_{x_{n,s} \in \Nfr}\ell'_{s,t} h_{i,t}(x_{n,s})\vbrack{\nabla_{w_i}h(x_p^+),\M_j}\right]& \stackrel{\ding{172}}{=} \sum_{x_{n,s} \in \Nfr}\E\left[ \left(\ell'_{s,t} - \frac{1}{|\Bfr|} \right)h_{i,t}(x_{n,s})\vbrack{\nabla_{w_i}h(x_p^+),\M_j}\right]\\
        \leq&  \sum_{x_{n,s} \in \Nfr}\E\left[ \left|\ell'_{s,t} - \frac{1}{|\Bfr|} \right|\cdot|h_{i,t}(x_{n,s})|\cdot |\vbrack{\nabla_{w_i}h(x_p^+),\M_j}|\right]\\
        \stackrel{\ding{173}}\leq&  \widetilde{O}\left(\frac{\sum_{i\in[m]}\|w_i^{(t)}\|_2^2}{\tau d}\cdot \|w_i^{(t)}\|_2\right)
    \end{align*}
    where \ding{173} has applied \myref{lem:activation-size}{Lemma} to \(h_{i,t}(x_{n,s})\) and \(|\vbrack{\nabla_{w_i} h_{i,t}(x_{p}^+),\M_j}| = |\vbrack{\M_j,x_p^{+}}| \). Putting all the above calculations together, we have 
    \begin{align*}
        \vbrack{w_i^{(t+1)},\M_j} =  \left(1 - \lambda + \Pr(|z_j|=1)\right)\vbrack{w_i^{(t)},\M_j} \pm \widetilde{O}\left(\frac{\sum_{i\in[m]}\|w_i^{(t)}\|_2^2}{\tau d}\cdot \|w_i^{(t)}\|_2\right)\pm \widetilde{O}\left(\frac{\|w_i^{(t)}\|_2}{d_1}\right)
    \end{align*}
    Before we perform induction, we obtain from similar approach the (stochastic) gradient step of \(w_i\) toward the direction of dense feature \(\M^{\perp}_j\) as
    \begin{align*}
        \vbrack{w_i^{(t+1)},\M^{\perp}_j} & = \vbrack{w_i^{(t)},\M^{\perp}_j} -\vbrack{\nabla_{w_i}\Obj(f_t),\M^{\perp}_j} \\
        &= (1 - \lambda)\vbrack{w_i^{(t)},\M^{\perp}_j} + \E \left[(1 - \ell'_{p,t})h_{i,t}(x_p^{++}) \vbrack{\nabla_{w_i} h_{i,t}(x_p^+),\M^\perp_j} \right]\\
        & \quad - \sum_{x_{n,s} \in \Nfr}\E\left[\ell'_{s,t} h_{i,t}(x_{n,s})\vbrack{\nabla_{w_i}h(x_p^+),\Mperp_j}\right] + \frac{\|w_i^{(t)}\|_2}{\poly(d_1)} \\
        & = (1 - \lambda)\vbrack{w_i^{(t)},\M^{\perp}_j} \pm \widetilde{O}\left(\frac{\sum_{i\in[m]}\|w_i^{(t)}\|_2^2}{\tau d}\cdot \|w_i^{(t)}\|_2\right) \pm \widetilde{O}\left(\frac{\|w_i^{(t)}\|_2}{d_1}\right) 
    \end{align*}
    Then we can begin to perform our induction: at \(t=0\), our properties holds trivially. Now suppose before iteration \(t = t_1\), the claimed properties holds, then we can easily obtain that for all \(t \leq t_1\):
    \begin{align*}
        \|w_i^{(t)}\|_2^2 = \|\M\M^{\top}w_i^{(t)}\|_2^2 + \|\M^{\perp}(\Mperp)^{\top}w_i^{(t)}\|_2^2 \leq O(d) \|w_i^{(0)}\|_2^2 \leq \frac{1}{\poly(d_1)}
    \end{align*}
    Thus we have \(\sum_{i\in[m]}\|w_i^{(t)}\|_2^2 \leq \frac{1}{\poly(d_1)}\). We now begin to verify all the properties for \(t = t_1 + 1\), until \(t_1\) reaches \(T_1\).
    \begin{itemize}
        \item We first derive an upper bound for \(\|\M\M^{\top}w_i^{(t+1)}\|_2^2\) at iterations \(t\leq t_1\). For each \(j \in [d]\), as long as \(\vbrack{w_i^{(t)},\M_j}\geq \Omega(\|w_i^{(t)}\|_2/d\sqrt{d_1})\), then
        \begin{align*}
            |\vbrack{w_i^{(t)},\M_j}| & \leq \left(1 - \eta\lambda + \frac{\eta C_z\log\log d}{d}\right)\vbrack{w_i^{(t)},\M_j}+ \widetilde{O}\left(\frac{\eta\|w_i^{(t)}\|_2}{d_1}\right)  \\
            & \leq \left(1 + \frac{\eta C_z\log\log d}{d} + \widetilde{O}(\frac{\eta}{d^2})\right)|\vbrack{w_i^{(t)},\M_j}|
        \end{align*}
        Define set of features: \(\mathcal{E}^{(t)} := \{ j\in[d]:\,\vbrack{w_i^{(t)},\M_j}<O(\|w_i^{(t)}\|_2/d\sqrt{d_1}) \}\), note that \(\mathcal{E}^{(t+1)}\subseteq \mathcal{E}^{(t)} \subseteq \Lambda_i\) (where the set \(\Lambda_i\) is defined in \myref{lem:property-init}{Lemma}) in the sense that if \(j \notin \mathcal{E}^{(t)}\), then  but \(\frac{C_z\log\log d}{d}|\vbrack{w_i^{(t)},\M_j}| \geq \widetilde{O}\big(\frac{\|w_i^{(t)}\|_2}{d_1}\big)\) in the above calculations. Therefore:
        \begin{align*}
            \|\M\M^{\top} w_{i}^{(t+1)} \|_2^2 & = \sum_{j\in[d]}\left[\left(1 - \eta\lambda + \frac{\eta C_z\log\log d}{d}\right)\vbrack{w_i^{(t)},\M_j}\pm \widetilde{O}\left(\frac{\eta\|w_i^{(t)}\|_2}{d_1}\right)\right]^2 \\
            &\leq \sum_{j\in[d]}\vbrack{w_i^{(0)},\M_j}^2\left( 1 + \frac{\eta C_z\log\log d}{d} + \widetilde{O}(\frac{\eta}{d^2})\right)^{2t} \\
            &\quad + \sum_{j\in[d]: j\in\mathcal{E}^{(0)}}\widetilde{O}\left(\frac{ (t+1)^2\eta^2 \max_{t\leq t_1} \|w_i^{(t)}\|_2^2}{d_1^2}\right)\\
            &\leq  \|\M\M^{\top} w_{i}^{(0)} \|_2^2\left( 1 + \frac{\eta C_z\log\log d}{d} + \widetilde{O}(\frac{\eta}{d^2})\right)^{2t} + O(1/d)\|\M\M^{\top} w_{i}^{(0)} \|_2^2
        \end{align*}
        which holds for all \(t\leq t_1\leq T_1 = \Theta(\frac{d\log d}{\eta \log\log d})\), the last inequality is due to the following calculations:
        \begin{align*}
            \sum_{j\in[d]: j\in\mathcal{E}^{(0)}}\widetilde{O}\left(\frac{ (t+1)^2\eta^2 \max_{t\leq t_1} \|w_i^{(t)}\|_2^2}{d_1^2}\right) &\leq \widetilde{O} \left(\frac{d^3\max_{t\leq t_1}\|w_i^{(t)}\|_2^2 }{d_1^2}\right) \ll \frac{\max_{t\leq t_1}\|w_i^{(t)}\|_2^2}{d^2d_1} \\
            &\ll \frac{\|w_i^{(0)}\|_2^2}{dd_1} \ll O(1/d)\|\M\M^{\top}w_i^{(0)}\|_2^2
        \end{align*}
        \item Secondly we give an lower bound for \(\|\M\M^{\top}w_i^{(t+1)}\|_2^2\) for iterations \(t \leq t_1\). From the above calculations, we have 
        \begin{align*}
            \|\M\M^{\top} w_{i}^{(t+1)} \|_2^2 & = \sum_{j\in[d]}\left[\left(1 - \eta\lambda + \frac{\eta C_z\log\log d}{d}\right)\vbrack{w_i^{(t)},\M_j} \pm \widetilde{O}\left(\frac{\eta\|w_i^{(t)}\|_2}{d_1}\right)\right]^2 \\
            &\geq \sum_{j\in\mathcal{E}^{(0)}}\vbrack{w_i^{(0)},\M_j}^2\left( 1 - \eta\lambda + \frac{\eta C_z\log\log d}{d}\right)^{2t}  - \widetilde{O}\left(\frac{ (t+1)^2\eta^2d \max_{t\leq t_1} \|w_i^{(t)}\|_2^2}{d_1^2}\right)\\
            &\geq  \|\M\M^{\top} w_{i}^{(0)} \|_2^2\left( 1 - \eta\lambda + \frac{\eta C_z\log\log d}{d} \right)^{2t} - O(1/d)\|\M\M^{\top} w_{i}^{(0)} \|_2^2
        \end{align*}
        where the last inequality follows from our computations of the upper bound.
        \item Finally we give an upper bound of \(\|\M^{\perp}(\M^{\perp})^{\top}w_i^{(t+1)}\|_2^2\) for iterations \(t \leq t_1\). We can calculate similarly, by
        \begin{align*}
            &\qquad\|\Mperp (\Mperp)^{\top} w_{i}^{(t+1)} \|_2^2 \\
            & = \sum_{j\in[d_1]\setminus [d]} \left[(1 - \eta\lambda)\vbrack{w_i^{(t)},\M^{\perp}_j} \pm \widetilde{O}\left(\frac{\eta\|w_i^{(t)}\|_2}{d_1}\right)\right]^2\\
            & \leq \|\Mperp (\Mperp)^{\top} w_{i}^{(0)} \|_2^2 + \max_{j\in[d_1]\setminus [d]}|\vbrack{w_i^{(0)},\M^{\perp}_j}| \widetilde{O}(\max_{t\leq t_1}\|w_i^{(t)}\|_2) +\widetilde{O}\left(\frac{\eta^2 (t+1)^2\max_{t\leq t_1} \|w_i^{(t)}\|_2^2}{d_1^2}\right)\\
            & \stackrel{\text{\ding{172}}}{\leq}  (1 + \widetilde{O}(d/\sqrt{d_1}))\|\Mperp (\Mperp)^{\top} w_{i}^{(0)} \|_2^2 + O(d^4/d_1^2)\|w_i^{(0)}\|_2^2\\
            & \stackrel{\text{\ding{173}}}{\leq} \left(1 + \frac{1}{\poly(d)}\right)\|\Mperp (\Mperp)^{\top} w_{i}^{(0)} \|_2^2
        \end{align*}
        where \ding{172} and \ding{173} have used several facts: (1) at initialization, we have \(|\vbrack{w_i^{(0)},\M^{\perp}_j}| \leq \widetilde{O}(\|w_i^{(0)}/\sqrt{d_1})\) with high probability; (2) from our induction hypothesis, \(\widetilde{O}(\max_{t\leq t_1}\|w_i^{(t)}\|_2) \leq O(d)\|w_i^{(0)}\|_2^2\); (3) at initialization we have \(\|w_i^{(0)}\|_2^2 \leq O(\|\Mperp (\Mperp)^{\top} w_{i}^{(0)} \|_2^2)\) with high probability.
    \end{itemize}
    Note that for each neuron \(i \in [m]\), from \myref{lem:property-init}{Lemma} combined with our upper bound and lower bound, we know when all the weights \(\|w_i^{(t)}\|_2^2\) reach \(\Theta(1)\|w_i^{(0)}\|_2^2\), the maximum \(\max_{i\in[m]}\|w_i^{(t+1)}\|_2^2 \leq O(1)\|w_i^{(t+1)}\|_2^2\) for all $t\leq t_1$. Thus we have obtained all the results for \(t = t_1 + 1\), and are able to proceed induction.
\end{proof}

\begin{proof}[Proof of \myref{thm:initial-stage}{Theorem}]
    The result (a) is easy to verify using \myref{induct-1}{Induction Hypothesis}. This we only verify (b) and (c). Note that from similar gradient calculations to those in the proof of \myref{induct-1}{Induction Hypothesis}, we have, for \(j \in [d]\) and \(i \in \Mcal_j\):
    \begin{align*}
        |\vbrack{w_i^{(T_1)},\M_j}| &\geq |\vbrack{w_i^{(0)},\M_j}|  \left( 1 - \eta\lambda + \frac{\eta C_z\log\log d}{d} \right)^{T_1} - \widetilde{O}\left(\frac{\eta T_1\|w_i^{(T_1)}\|_2}{d_1}\right)\\
        & \stackrel{\text{\ding{172}}}{\geq} \frac{\sqrt{c_1\log d} }{\sqrt{d}} \|\M\M^{\top} w_{i}^{(0)} \|_2\left( 1 - \eta\lambda + \frac{\eta C_z\log\log d}{d} \right)^{T_1} - \frac{\|\M\M^{\top} w_{i}^{(0)} \|_2}{\poly(d)}\\
        & \stackrel{\text{\ding{172}}}{\geq}\frac{ (1 + c_0)\sqrt{2\log d} }{\sqrt{d}}\| w_{i}^{(T_1)} \|_2
    \end{align*}
    where in \ding{172} we have used \myref{lem:property-init}{Lemma} and the fact that \(\frac{\eta T_1}{d_1} \leq \frac{1}{\poly(d)}\). And in the last inequality \ding{173} we argue: when all \(\|w_{i'}^{(t)}\|_2, i' \in [m]\) reach \((1 + \frac{2(1+2c_0)}{c_1})\|w_{i'}^{(0)}\|_2\), by using \myref{induct-1}{Induction Hypothesis} and our definition of \(T_1\), combined with the concentrations of initial weight norm \(\|\M\M^{\top}w_i^{(0)}\|_2\) in \myref{lem:property-init}{Lemma} it holds that
    \begin{align*}
        |\vbrack{w_i^{T_1},\M_j}|^2&\geq \frac{c_1\log d}{d}\cdot\|\M\M^{\top}w_i^{(T_1)}\|_2^2 \\
        &\geq \frac{(2 + 4 c_0)^2\log d}{d} \cdot\|w_i^{(T_1)}\|_2^2 -\|\Mperp(\Mperp)^{\top}w_i^{(T_1)} \|_2^2 \\
        &\geq 2(1+c_0)^2\frac{\|w_i^{(T_1)}\|_2^2\log d}{d}
    \end{align*}
    since all neuron weights \(\{w_i^{(t)}\}_{i\in[m]}\) grow in the speed of \(\|w_i^{(t)}\|_2 = (1 + \frac{\eta C_z\log\log d}{d} + o(\frac{1}{\poly(d)}))^{t}\|w_i^{(0)}\|_2\). The property (c) and (d) can be verified via exactly the same approach, combined with \myref{lem:property-init}{Lemma}. For (e), noticing that at initialization \(|\vbrack{w_i^{(0)},\Mperp_j}| \leq O(\sqrt{\frac{\log d}{d_1}})\|w_i^{(0)}\|_2\), we have 
    \begin{align*}
        |\vbrack{w_i^{(T_1)} ,\Mperp_j}| \leq |\vbrack{w_i^{(0)} ,\Mperp_j}|  + O(T_1 \eta)\cdot\max_{t \leq T_1}O\left(\frac{\|w_i^{(t)}\|_2}{d_1}\right) \leq O(\sqrt{\frac{\log d}{d_1}})\|w_i^{(t)}\|_2
    \end{align*}

\end{proof}

\section{Stage II: Singleton Emerge}

In this section we will present an analysis of how each feature \(\M_j\) can be ``won'' by some subsets of the neurons, which depends on the randomness of random initialization. In this stage, we will prove that the following induction hypothesis holds for all iterations.

\begin{induct}\label{induct-2}
    For all iterations \(t \in (T_1,T_2]\), our neurons \(i\in[m]\) satisfies: 
    \begin{itemize}
        \item[(a)] For \(j \in [d]\), if \(i \in \Mcal_j^{\star}\), then \(|\vbrack{w_i^{(t)},\M_j}| \geq (1+c_0)b_i^{(t)}\);
        \item[(b)] For \(j \in [d]\), if \(i \notin \Mcal_j\), then \(|\vbrack{w_i^{(t)},\M_j}| \leq (1-c_0)b_i^{(t)}\), Furthermore, \(|\vbrack{w_i^{(t)},\M_j}| \leq \widetilde{O}(\frac{\|w_i^{(t)}\|_2}{\sqrt{d}})\); 
        \item[(c)] For each \(i \in [m]\), there are at most \(O(2^{-\sqrt{\log d}}d)\) many \(j \in [d]\) such that \(\vbrack{w_i^{(t)},\M_j}^2 \geq \frac{(b_i^{(t)})^2}{\sqrt{\log d}}\);
        \item[(d)] For each \(i\in[m]\), we have \(|\vbrack{w_i^{(t)},\Mperp_j}| \leq \widetilde{O}(\frac{\|w_i^{(t)}\|_2}{\sqrt{d_1}})\) for all \(j\in [d_1]\setminus [d]\);
        \item[(e)] \(\|w_i^{(t)}\|_2^2 \leq \frac{d(b_i^{(t)})^2}{\log d}\) for all \(i\in[m]\).
    \end{itemize}
\end{induct}

\subsection{Gradient Computations}

\begin{definition}[notations]\label{def:notations-proof}
    For simpler presentation, we define the following notations: given \(x = \M z + \xi \sim \D_x\) as in \myref{def:sparse-coding}{Definition}, and \(\DD\sim\D_{\DD}\) as in \myref{def:random-mask}{Definition}, we let (for each \(j \in [d]\))
    \begin{align}
        x^{\setminus j} &:= \sum_{j'\neq j,j'\in[d]}\M_{j'} z_{j'} + \xi & S^{\setminus j}_{i,t} &:= \vbrack{w_i^{(t)}, x^{\setminus j}} & \bar{S}^{\setminus j}_{i,t} &:= \vbrack{w_i^{(t)}, (\Id-2\DD)x^{\setminus j}} \\ 
        &  &\alpha_{i, j}^{(t)} &:= \vbrack{w_i^{(t)},\M_j} &\bar{\alpha}_{i,j}^{(t)} &:= \vbrack{w_i^{(t)},(\Id - 2\DD)\M_j} 
    \end{align}
    whenever the neuron index \(i \in [m]\) is clear from context, we omit the subscript of neuron index \(i\) and time \(t\) for simplicity. 
\end{definition}
First we present our lemma for the gradient of features associated with the sparse signals.

\begin{lemma}[Gradient for sparse features]\label{lem:positive-gd-sparse-2} 
    Suppose \myref{induct-2}{Induction Hypothesis} holds at iteration \(t \geq 0\), for \(j \in [d]\), we denote events 
    \begin{align*}
        A_1 &:=  \{S^{\setminus j}_{i,t}\geq b_i^{(t)} - \alpha_{i,j}^{(t)}\}, & A_2 &:= \{\bar{S}^{\setminus j}_{i,t}\geq b_i^{(t)} - \bar{\alpha}_{i,j}^{(t)}\}; \\
        A_3 &:= \{|\bar{S}^{\setminus j}_{i,t} + \bar{\alpha}_{i,j}^{(t)}| \geq (\alpha_{i,j}^{(t)} - b_i^{(t)})/2\}, & A_4 &:= \{S^{\setminus j}_{i,t}\geq (\alpha_{i,j}^{(t)} - b_i^{(t)})/2\};
    \end{align*}
    and quantities \(L_1,L_2,L_3,L_4\) as
    \begin{align*}
        L_1 &:= \sqrt{\frac{\E[|\bar{S}^{\setminus j}_{i,t}|^2(\1_{A_1} + \1_{A_2})]}{\E[\vbrack{w_i^{(t)},\xi_p}^2]}}, &L_2 &:= \Pr(A_1), &L_3 &:= \sqrt{\frac{\E[|\bar{S}^{\setminus j}_{i,t}|^2(\1_{A_1} + \1_{A_2})]}{\E[\vbrack{w_i^{(t)},\xi_p}^2]}}, & L_4 &:= \Pr(A_3)
    \end{align*}
    then we have the following results:
    \begin{itemize}
        \item[(a)] (all features) For all \(i \in [m]\), if \(\alpha_{i,j}^{(t)} \geq 0\), we have (when \(\alpha_{i,j}^{(t)} \leq 0\) the opposite inequality holds)
        \begin{align*}
            \E\left[h_{i,t}(x_p^{++})\1_{|\vbrack{w_i^{(t)},x_p^+}|\geq b_i^{(t)} }  z_j \right] &\leq \alpha_{i,j}^{(t)} \cdot \E[z_j^2\1_{|\vbrack{w_i^{(t)} ,x_p}|\geq b_i^{(t)} + |\vbrack{w_i^{(t)},x_p^+-x_p}|} ] \\
            &\quad \pm  (\alpha_{i,j}^{(t)} + O(\sqrt{\E|\bar{\alpha}_{i,j}^{(t)}|^2}))\cdot \E[z_j^2]\cdot O(L_1 + L_2) 
        \end{align*} 
        \item[(b)] (lucky features) If \(\alpha_{i,j}^{(t)} > b_i^{(t)} \), we have
        \begin{align*}
            \E\left[h_{i,t}(x_p^{++})\1_{|\vbrack{w_i^{(t)},x_p^+}|\geq b_i^{(t)} }  z_j \right] &= (\alpha_{i,j}^{(t)} - b_i^{(t)})\cdot\E[z_j^2\1_{|\vbrack{w_i^{(t)} ,x_p}|\geq b_i^{(t)} + |\vbrack{w_i^{(t)},x_p^+-x_p}|} ] \\
            &\quad \pm (\alpha_{i,j}^{(t)} + O(\sqrt{\E|\bar{\alpha}_{i,j}^{(t)}|^2}))\cdot \E[z_j^2]\cdot O(L_3 + L_4) 
        \end{align*} 
        If \(\alpha_{i,j}^{(t)}<-b_i^{(t)}\), then the opposite inequality holds with \( (\alpha_{i,j}^+ - b_i^{(t)})\) changing to \( (\alpha_{i,j}^+ + b_i^{(t)})\).
    \end{itemize}
\end{lemma}

\begin{proof}[Proof of \myref{lem:positive-gd-sparse-2}{Lemma} (a)]
    In the proof we will make the following simplification of notations: we drop the time superscript \(^{(t)}\), and also the subscript for neuron index \(i\) in \eqref{def:notations-proof}. We start with the case when \(0< \alpha_j < b_i^{(t)}\) and rewrite the expectation as follows:
    \begin{align*}
        & \E\left[h_i(x_p^{++})\1_{|\vbrack{w_i,x_p^+}|\geq b_i }  z_j\right] \\
        = \ &  \E\left[ \left(  \ReLU\left( \vbrack{w_i,x_p^{++}} - b_i \right) - \ReLU\left( -\vbrack{w_i,x_p^{++}} - b_i \right)\right)\1_{|\vbrack{w_i,x_p^+}|\geq b_i }z_j \right]\\
        = \ & \E\left[ \left(  \left( \vbrack{w_i,x_p^{++}} - b_i \right)\1_{\vbrack{w_i,x_p^{++}}\geq b_i} - \left( -\vbrack{w_i,x_p^{++}} - b_i \right)\1_{-\vbrack{w_i,x_p^{++}}\geq b_i}\right)\1_{|\vbrack{w_i,x_p^+}|\geq b_i }z_j \right]\\
        = \ & \E\left[ \vbrack{w_i,x_{p}^{++}}\1_{|\vbrack{w_i,x_{p}^{++}}|\geq b_i}\1_{|\vbrack{w_i,x_{p}^{+}}|\geq b_i}z_j\right] - \E\left[b_i\left(\1_{\vbrack{w_i,x_p^{++}}\geq b_i} - \1_{-\vbrack{w_i,x_p^{++}}\geq b_i}\right)\1_{|\vbrack{w_i,x_{p}^{+}}|\geq b_i}z_j\right]
    \end{align*}
    Notice that the first term on the RHS can be simplified as:
    \begin{align*}
        &\E\left[ \vbrack{w_i,x_{p}^{++}}\1_{|\vbrack{w_i,x_{p}^{++}}|\geq b_i}\1_{|\vbrack{w_i,x_{p}^{+}}|\geq b_i}z_j\right] \\
        = \ & \E\left[ (\vbrack{w_i,x_{p}} + \vbrack{w_i,(\Id-2\DD)x_{p}})\1_{|\vbrack{w_i,2(\Id-\DD)x_{p}}|\geq b_i}\1_{|\vbrack{w_i,2\DD x_{p}}|\geq b_i}z_j\right]\\
        \stackrel{\text{\ding{172}}}{=} \ & \E\left[ \vbrack{w_i,x_{p}}\1_{|\vbrack{w_i,x_{p}^{++}}|\geq b_i}\1_{|\vbrack{w_i,x_{p}^{+}}|\geq b_i}z_j\right]
    \end{align*}
    where \ding{172} is due the fact that \(\vbrack{w_i,(\Id-2\DD)x_{p}}\1_{|\vbrack{w_i,x_{p}^{++}}|\geq b_i}\1_{|\vbrack{w_i,x_{p}^{+}}|\geq b_i}\) is symmetric with respect to zero due to the randomness of \(\DD \sim \D_{\DD}\). Thus the expectation can be expanded as:
    \begin{align*}
        & \E\left[h_i(x_p^{++})\1_{|\vbrack{w_i,x_p^+}|\geq b_i }  z_j\right] \\
        = \ & \E\left[ \vbrack{w_i,x_{p}}\1_{|\vbrack{w_i,x_{p}^{++}}|\geq b_i}\1_{|\vbrack{w_i,x_{p}^{+}}|\geq b_i}z_j\right] - \E\left[b_i\left(\1_{\vbrack{w_i,x_p^{++}}\geq b_i} - \1_{-\vbrack{w_i,x_p^{++}}\geq b_i}\right)\1_{|\vbrack{w_i,x_{p}^{+}}|\geq b_i}z_j\right] \\
        = \ & \E\left[ \alpha_j z_j^2\1_{|\vbrack{w_i,x_{p}^{++}}|\geq b_i + |\vbrack{w_i,x_p^+ - x_p}|}\right] \\
        & + \E\left[ (S^{\setminus j}- b_i)\1_{\vbrack{w_i,x_{p}^{++}}\geq b_i}\1_{\vbrack{w_i,x_{p}^{+}}\geq b_i}z_j\right]+ \E\left[ (S^{\setminus j} + b_i)\1_{\vbrack{w_i,x_{p}^{++}}\leq -b_i}\1_{\vbrack{w_i,x_{p}^{+}}\leq -b_i}z_j\right]\\
        & + \E\left[ (\alpha_j z_j + S^{\setminus j} - b_i)\1_{\vbrack{w_i,x_{p}^{++}}\geq b_i}\1_{\vbrack{w_i,x_{p}^{+}}\leq -b_i}z_j\right]+ \E\left[ (\alpha_j z_j + S^{\setminus j} + b_i)\1_{\vbrack{w_i,x_{p}^{++}}\leq -b_i}\1_{\vbrack{w_i,x_{p}^{+}}\geq b_i}z_j\right]\\
        & = J_1 + J_2 + J_3
    \end{align*}
    Now we need to obtain absolute bounds for both \(J_2\) and \(J_3\). We start with \(J_2\), where 
    \begin{align*}
        J_2 & = \E\left[ (S^{\setminus j} - b_i)\1_{\vbrack{w_i,x_{p}^{++}}\geq b_i}\1_{\vbrack{w_i,x_{p}^{+}}\geq b_i}z_j\right]+ \E\left[ (S^{\setminus j} + b_i)\1_{\vbrack{w_i,x_{p}^{++}}\leq -b_i}\1_{\vbrack{w_i,x_{p}^{+}}\leq -b_i}z_j\right]\\
        & = \E\left[ (S^{\setminus j} - b_i)\1_{\vbrack{w_i,x_{p}}\geq b_i + |\vbrack{w_i,x_p^+-x_p}|}z_j\right]+ \E\left[ (S^{\setminus j} + b_i)\1_{\vbrack{w_i,x_{p}}\leq -b_i - |\vbrack{w_i,x_p^+-x_p}|}z_j\right]
    \end{align*}
    We proceed with the first term \(\E\left[ (S^{\setminus j} - b_i)\1_{\vbrack{w_i,x_{p}}\geq b_i + |\vbrack{w_i,x_p^+-x_p}|}z_j\right]\). First from a trivial calculation conditioned on the randomness of \(z_j\) we have:
    \begin{align*}
        &\quad \, \E\left[ (S^{\setminus j}- b_i)\1_{\vbrack{w_i,x_{p}}\geq b_i + |\vbrack{w_i,x_p^+-x_p}|}z_j\right]\\
        & = \E\left[ (S^{\setminus j}- b_i)|z_j|\left(\1_{ S^{\setminus j}\geq b_i - \alpha_jz_j  + |\bar{S}^{\setminus j} + \bar{\alpha}_jz_j |} -\1_{S^{\setminus j}\geq b_i + \alpha_jz_j  + |\bar{S}^{\setminus j} - \bar{\alpha}_jz_j|}\right)\right]
    \end{align*}
    Now define 
    \begin{displaymath}
        Z = \frac{1}{2}(|\bar{S}^{\setminus j} + \bar{\alpha}_jz_j | + |\bar{S}^{\setminus j} - \bar{\alpha}_jz_j |)\qquad  Z' = \frac{1}{2}(|\bar{S}^{\setminus j} + \bar{\alpha}_jz_j| - |\bar{S}^{\setminus j} - \bar{\alpha}_j z_j|)
    \end{displaymath}
    In this case, we always have \(|Z'|\leq |\bar{\alpha}_j||z_j|\), and 
    \begin{displaymath}
        \left| \1_{ S^{\setminus j}\geq b_i - \alpha_jz_j  + |\bar{S}^{\setminus j} + \bar{\alpha}_jz_j |} -\1_{S^{\setminus j}\geq b_i + \alpha_jz_j  + |\bar{S}^{\setminus j} - \bar{\alpha}_jz_j|}\right| = \1_{S^{\setminus j} - b_i - Z \in [ -|\alpha_jz_j - Z'|,  |\alpha_jz_j - Z'|]}
    \end{displaymath}
    which allows us to proceed as follows:
    \begin{align*}
        &\left|\E\left[ (S^{\setminus j}- b_i)\1_{\vbrack{w_i,x_{p}}\geq b_i + |\vbrack{w_i,x_p^+-x_p}|}z_j\right]\right|\\
        = \ & \left|\E\left[ (S^{\setminus j}- b_i - Z + Z)|z_j|\left(\1_{ S^{\setminus j}\geq b_i - \alpha_jz_j  + |\bar{S}^{\setminus j} + \bar{\alpha}_jz_j |} -\1_{S^{\setminus j}\geq b_i + \alpha_jz_j  + |\bar{S}^{\setminus j} - \bar{\alpha}_jz_j|}\right)\1_{S^{\setminus j}\geq b_i -\alpha_jz_j}\right]\right|\\
        \leq \ & \E\left[ |S^{\setminus j}- b_i - Z||z_j|\1_{S^{\setminus j} - b_i - Z \in [ -|\alpha_jz_j - Z'|,  |\alpha_jz_j - Z'|]}\1_{S^{\setminus j}\geq b_i -\alpha_j}\right] \\
        & + \E\left[ (|\bar{S}^{\setminus j}| + |\bar{\alpha}_j||z_j|)|z_j|\1_{S^{\setminus j} - b_i - Z \in [ -|\alpha_j - Z'|,  |\alpha_j - Z'|]}\1_{S^{\setminus j}\geq b_i -\alpha_j}\right]  \\
        \leq \ & \E\left[(\alpha_j + 2|\alpha'_j|)|z_j|^2\1_{S^{\setminus j}\geq b_i -\alpha_j} \right] +\E\left[ |\bar{S}^{\setminus j}||z_j|\1_{S^{\setminus j} - b_i - Z \in [ -|\alpha_j - Z'|,  |\alpha_j - Z'|]}\1_{S^{\setminus j}\geq b_i -\alpha_j}\right]\\
        = \ & \E\left[(\alpha_j + 2|\alpha'_j|)|z_j|^2\1_{S^{\setminus j}\geq b_i -\alpha_j} \right] +\sqrt{\E\left[|\bar{S}^{\setminus j}|^2|z_j^2| \1_{S^{\setminus j}\geq b_i - \alpha_j} \right]} \cdot \sqrt{\E\left[ \1_{S^{\setminus j} - b_i - Z \in [ -|\alpha_jz_j - Z'|,  |\alpha_jz_j - Z'|]}\right]} \\
        \stackrel{\text{\ding{172}}}{\leq} \ & \E\left[(\alpha_j + 2|\alpha'_j|)|z_j|^2\1_{S^{\setminus j}\geq b_i -\alpha_j} \right] + \sqrt{\E\left[|\bar{S}^{\setminus j}|^2|z_j^2| \1_{S^{\setminus j}\geq b_i-\alpha_jz_j}\right]}\cdot \sqrt{\frac{\E[(\alpha_j + |\bar{\alpha}_j|)^2 z_j^2]}{\E\left[ \vbrack{w_i,\xi_p}^2 \right]}}\\
        = \ & 2(\alpha_j + O(\E[|\bar{\alpha}_j|^2]^{1/2}) )\E[z_j^2] \left(\sqrt{\frac{\E[|\bar{S}^{\setminus j}|^2 \1_{S^{\setminus j}\geq b_i - \alpha_j}]}{\E\left[ \vbrack{w_i,\xi_p}^2 \right]}} + \Pr(S^{\setminus j} \geq b_i - \alpha_j)\right)
    \end{align*}
    where in \ding{172} we have used the randomness of \(\xi_p\) in the following manner: Fixing the randomness of \(z\) and \(\DD\), we have \(S^{\setminus j} - Z\) is a random variable depending solely on the randomness of \(\xi_p\), and thus we have 
    \begin{align*}
        \E\left[ \1_{S^{\setminus j} - b_i - Z \in [ -|\alpha_j - Z'|,  |\alpha_j - Z'|]}\right] & \leq \E\left[ \1_{\vbrack{w_i,\xi_p} - |\vbrack{2(\Id-\DD)w_i,\xi_p}|\in [ |\bar{\alpha}_j| -(\alpha_j + |\bar{\alpha}_j|) , \alpha_j + 2|\bar{\alpha}_j|]}\right] \\
        & = \E\left[ \1_{\vbrack{\DD w_i,\xi_p} + \vbrack{(\Id-\DD)w_i,\xi_p} - |\vbrack{2(\Id-\DD)w_i,\xi_p}|\in [ -O(\alpha_j + |\bar{\alpha}_j|) , O (\alpha_j + |\bar{\alpha}_j|)]}\right] \\
        & \leq \E\left[\frac{O(\alpha_j+ |\bar{\alpha}_j|)^2|z_j|}{\vbrack{w_i,\xi_p}^2}\right]  = \frac{\E[O(\alpha_j + |\bar{\alpha}_j|)^2|z_j|]}{\E[\vbrack{w_i,\xi_p}^2]}
    \end{align*}
    Simultaneously, from similar analysis as above, we have for the second term in \(J_2\):
    \begin{align*}
        \left|\E\left[ (S^{\setminus j} + b_i)\1_{\vbrack{w_i,x_{p}}\leq -b_i - |\vbrack{w_i,x_p^+-x_p}|}z_j\right]\right|
        & = \left|\E\left[ (S^{\setminus j}+ b_i)\1_{\vbrack{w_i,x_{p}}\leq -b_i - |\bar{S}^{\setminus j} + \bar{\alpha}_j z_j|}z_j\right] \right|\\
        &\leq O(\alpha_j + \E[|\bar{\alpha}_j|^2]^{1/2})\E[z_j^2]\sqrt{\frac{\E[|\bar{S}^{\setminus j}|^2 \1_{S^{\setminus j}\geq b_i - \alpha_j}]}{\E[ \vbrack{w_i,\xi_p}^2 ]}}
    \end{align*}
    Now we turn to \(J_3\), from the symmetry of \(x_p^+\) and \(x_p^{++}\) over the randomness of \(\DD\sim\D_{\DD}\), we observe
    \begin{align*}
        \E\left[b_i\1_{\vbrack{w_i,x_p^+}\geq b_i}\1_{\vbrack{w_i,x_p^++}\leq -b_i}z_j\right] = \E\left[b_i\1_{\vbrack{w_i,x_p^+}\leq -b_i}\1_{\vbrack{w_i,x_p^{++}}\geq b_i}z_j\right]
    \end{align*}
    which allows us to drop the \(b_i\) terms in \(J_3\). The analysis of the rest of \(J_3\) is somewhat similar. First we observe that whenever \(\1_{\vbrack{w_i,x_{p}^{++}}\geq b_i}\1_{\vbrack{w_i,x_{p}^{+}}\leq -b_i} \neq 0\), we have 
    \begin{align*}
        \vbrack{w_i,x_{p}^{++}}\geq b_i \text{ and }
        \vbrack{w_i,x_{p}^{+}}\leq -b_i \implies \bar{S}^{\setminus j} + \bar{\alpha}_j z_j \geq b_i + |S^{\setminus j} + \alpha_j z_j|
    \end{align*}
    When this inequality holds, we always have \(\bar{S}^{\setminus j} \geq b_i - \bar{\alpha}_j z_j\). Together with all the above observations, we proceed to compute as:
    \begin{align*}
        & \quad\, \left|\E\left[ (\alpha_j z_j + S^{\setminus j})\1_{\vbrack{w_i,x_{p}^{++}}\geq b_i}\1_{\vbrack{w_i,x_{p}^{+}}\leq -b_i}z_j\right]\right|\\
        & =\left|\E\left[ (\alpha_j z_j + S^{\setminus j})\1_{\bar{\alpha}_j z_j + \bar{S}^{\setminus j} \geq b_i + |\alpha_jz_j + S^{\setminus j}|}z_j\right]\right| \\
        & \stackrel{\text{\ding{172}}}{\leq} \E\left[ |\bar{S}^{\setminus j} + \bar{\alpha}_j z_j||z_j| \left|\1_{ \bar{S}^{\setminus j} \geq b_i -\bar{\alpha}_j z_j + |\alpha_jz_j + S^{\setminus j}|} - \1_{ \bar{S}^{\setminus j} \geq b_i + \bar{\alpha}_j z_j  + |-\alpha_jz_j + S^{\setminus j}|}\right|\right]\\
        & \leq \E\left[ |\bar{S}^{\setminus j} + \bar{\alpha}_jz_j||z_j|\1_{\bar{S}^{\setminus j} \geq b_i - \bar{\alpha}_j z_j}\cdot \1_{\bar{S}^{\setminus j} \in [ b_i - \alpha_j|z_j| + |S^{\setminus j} + \bar{\alpha}_j |z_j| |, b_i +\alpha_j|z_j| + |S^{\setminus j} -\bar{\alpha}_j |z_j||]} \right]\\
        & \leq \sqrt{\E\left[|\bar{S}^{\setminus j}|^2 |z_j|^2\1_{\bar{S}^{\setminus j} \geq b_i - \bar{\alpha}_j z_j}\right]}\sqrt{\E\left[\1_{\bar{S}^{\setminus j} \in [ b_i - \alpha_j |z_j| + |S^{\setminus j} + \bar{\alpha}_j |z_j| |, b_i +\alpha_j |z_j| + |S^{\setminus j} -\bar{\alpha}_j |z_j||]}\right]}\\
        &\quad + \E\left[\bar{\alpha}_j|z_j^2|\1_{\bar{S}^{\setminus j} \geq b_i - \bar{\alpha}_j z_j}\right] \\
        &\leq \sqrt{\E\left[|\bar{S}^{\setminus j}|^2 |z_j|^2\1_{\bar{S}^{\setminus j} \geq b_i - \bar{\alpha}_j z_j}\right]}\sqrt{\E\left[\1_{\vbrack{w_i,(2\DD-\Id)\xi_p}\in [ - |\alpha_j z_j + |\bar{\alpha}_j|z_j |, |\alpha_j z_j + |\bar{\alpha}_j|z_j |})\right]}+ \E\left[\bar{\alpha}_j|z_j^2|\1_{\bar{S}^{\setminus j} \geq b_i - \bar{\alpha}_j z_j}\right] \\
        &\leq O(\alpha_j + \E[|\bar{\alpha}_j|^2]^{1/2})\E[z_j^2]\cdot \sqrt{\frac{\E[|\bar{S}^{\setminus j}|^2\1_{\bar{S}^{\setminus j} \geq b_i - \bar{\alpha}_j}]}{\E[\vbrack{w_i,\xi_p}^2]}} + (\E|\bar{\alpha}_j|^2)^{1/2}\E[z_j^2]\Pr(\bar{S}^{\setminus j} \geq b_i - \bar{\alpha}_j)
    \end{align*}
    where in the last inequality, we have use the following reasoning: conditioned on fixed \(\DD\sim\D_{\DD}\), we know that \(\vbrack{w_i,(\Id-2\DD)\xi_p}\) has the same distribution with \(\vbrack{w_i,\xi_p}\). We use the randomness to obtain that 
    \begin{align*}
        \E_{\xi_p} \left[\1_{\vbrack{w_i,(2\DD-\Id)\xi_p}\in [ - |\alpha_j z_j + |\bar{\alpha}_j|z_j |, |\alpha_j z_j + |\bar{\alpha}_j|z_j |}\right] & \leq \frac{|\alpha_j + \bar{\alpha}_j|^2\cdot|z_j|^2}{\E_{\xi}[\vbrack{w_i,(\Id-2\DD)\xi_p}^2]} \\
        &= \frac{|\alpha_j + \bar{\alpha}_j|^2\cdot|z_j|^2}{\E_{\xi}[\vbrack{w_i,\xi_p}^2]}
    \end{align*}
    The second term of \(J_3\) can be similarly bounded by the same quantity. Now by combining all the results of \(J_1, J_2,J_3\) above, we have the desired result for (a).
\end{proof}

\begin{proof}[Proof of \myref{lem:positive-gd-sparse-2}{Lemma} (b)]
    This proof is extremely similar to the above proof of \myref{lem:positive-gd-sparse-2}{Lemma} (a), we describe the differences here and sketch the remaining. First we need to decompose the expectation as follows:
        \begin{align*}
            \E\left[h_i(x_p^{++})\1_{|\vbrack{w_i,x_p^+}|\geq b_i }  z_j\right] & =  \E\left[ (\alpha_j - b_i ) z_j^2\1_{|\vbrack{w_i,x_{p}}|\geq b_i + |\vbrack{w_i,x_p^+ - x_p}|}\right] \\
            & + \E\left[ S^{\setminus j}\1_{|\vbrack{w_i,x_{p}}|\geq b_i + |\vbrack{w_i,x_p^+ - x_p}|}z_j\right]\\
            & + \E\left[ (\alpha_j z_j + S^{\setminus j})\1_{|\vbrack{w_i,x_{p}^{+} - x_p}|\geq b_i + |\vbrack{w_i,x_p}|}z_j\right]\\
            & = J_1 + J_2 + J_3
        \end{align*}
    where we have used the following facts:
    \begin{itemize}
        \item \(\E[b_i\1_{\vbrack{w_i,x_{p}^+}\geq b_i}\1_{\vbrack{w_i,x_{p}^{++}}\geq b_i}z_j] = -\E[b_i\1_{\vbrack{w_i,x_{p}^+}\leq -b_i}\1_{\vbrack{w_i,x_{p}^{++}}\leq -b_i}z_j]\);
        \item \(\E[b_i\1_{\vbrack{w_i,x_{p}^+}\geq b_i}\1_{\vbrack{w_i,x_{p}^{++}}\leq -b_i}z_j] = \E[b_i\1_{\vbrack{w_i,x_{p}^{++}}\geq b_i}\1_{\vbrack{w_i,x_{p}^{+}}\leq -b_i}z_j]\);
        \item \(\1_{\vbrack{w_i,x_{p}^+}\geq b_i}\1_{\vbrack{w_i,x_{p}^{++}}\geq b_i} = \1_{\vbrack{w_i,x_{p}}\geq b_i + |\vbrack{w_i,x_p^+ - x_p}|}\);
        \item \(\1_{\vbrack{w_i,x_{p}^+}\leq -b_i}\1_{\vbrack{w_i,x_{p}^{++}}\leq -b_i} = \1_{\vbrack{w_i,x_{p}}\leq -b_i - |\vbrack{w_i,x_p^+ - x_p}|}\);
        \item \(\1_{\vbrack{w_i,x_{p}^+}\geq b_i}\1_{\vbrack{w_i,x_{p}^{++}}\leq -b_i} = \1_{\vbrack{w_i,x_{p}^+}\leq -b_i}\1_{\vbrack{w_i,x_{p}^{++}}\geq b_i} = \1_{|\vbrack{w_i,x_{p}^{+} - x_p}|\geq b_i + |\vbrack{w_i,x_p}|}\).
    \end{itemize}
    Now observe that \(J_2\) can be deal with as follows: define events \(A_3 := \{|\bar{S}^{\setminus j} + \bar{\alpha}_j| \geq (\alpha_j - b_i)/2\}\) and \(A_4 := \{S^{\setminus j}\geq (\alpha_j - b_i)/2\}\) and notice that \(\1\leq \1_{A_3} + \1_{A_4}\), we can compute
    \begin{align*}
        J_2 &= \E\left[ S^{\setminus j}\1_{|\vbrack{w_i,x_{p}}|\geq b_i + |\vbrack{w_i,x_p^+ - x_p}|}z_j\right]\\
        & = \E\left[ S^{\setminus j}|z_j| (\1_{S^{\setminus j}  \in [b_i - \alpha |z_j| + |\bar{S}^{\setminus j} + \bar{\alpha}_j |z_j||, b_i + \alpha |z_j| + |\bar{S}^{\setminus j} - \bar{\alpha}_j |z_j||]})\right]\\
        &= \E\left[ S^{\setminus j}|z_j| (\1_{A_3} + \1_{A_4}\1_{A_3^c}) \1_{S^{\setminus j}  \in [b_i - \alpha |z_j| + |\bar{S}^{\setminus j} + \bar{\alpha}_j |z_j||, b_i + \alpha |z_j| + |\bar{S}^{\setminus j} - \bar{\alpha}_j |z_j||]}\right]\\
        & \stackrel{\text{\ding{172}}}{\leq} \E\left[\left( O(\alpha_j + |\bar{\alpha}_j|)|z_j|^2\1_{A_3}  + |\bar{S}^{\setminus j}||z_j| (\1_{A_3} + \1_{A_4}) \right)\1_{S^{\setminus j}  \in [ \alpha |z_j| - b_i + |\bar{S}^{\setminus j} + \bar{\alpha}_j |z_j||, b_i + \alpha |z_j| + |\bar{S}^{\setminus j} - \bar{\alpha}_j |z_j||]}\right] \\
        & \leq \E[z_j^2]\cdot O(\alpha_j + \E[|\bar{\alpha}_j|^2]^{1/2})\cdot\left(\Pr(A_3) + \sqrt{\E[|\bar{S}^{\setminus j}|^2(\1_{A_3} + \1_{A_4} )]/ \E[\vbrack{w_i,\xi_p}^2]}\right)
    \end{align*}
    where \ding{172} relies on the fact that whenever \(\1_{|\vbrack{w_i,x_{p}}|\geq b_i + |\vbrack{w_i,x_p^+ - x_p}|}\neq 0\), we have \(|S^{\setminus j}| \leq b_i + (\alpha_j+|\bar{\alpha}_j|)|z_j| + |\bar{S}^{\setminus j}|\), and also that we have assumed \(b_i<\alpha_j\). The term \(J_3\) can be bounded from similar analysis as in the proof of \myref{lem:positive-gd-sparse-2}{Lem.} (a), but changing the factor from \(O(L_1+L_2)\) to \(O(L_3+L_4)\). Combining these calculations give the desired results.
\end{proof}

\begin{lemma}[Gradient from dense signals]\label{lem:positive-gd-noise-2} 
    Let \(i \in [m]\) and \(j \in [d]\), suppose \myref{induct-2}{Induction Hypothesis} holds for the current iteration \(t\), we have
    \begin{align*}
        \left|\E\left[h_{i,t}(x_p^{++})\1_{|\vbrack{w_i^{(t)},x_p^+}|\geq b_i^{(t)} }  \vbrack{2\DD\xi_p,\M_j} \right]\right| \leq \widetilde{O}\left(\frac{\|w_i^{(t)} \|_2}{d^{2}}\right)\cdot \max\{\Pr(|\vbrack{w_i^{(t)},x_p^{++}}|\geq b_i^{(t)} ), \widetilde{O}(1/\sqrt{d})\}
    \end{align*}
    For dense features \(\Mperp_j\), \(j \in [d_1]\setminus [d]\), we have similar results:
    \begin{align*}
        \left|\E\left[h_{i,t}(x_p^{++})\1_{|\vbrack{w_i,x_p^+}|\geq b_i }  \vbrack{2\DD\xi_p,\Mperp_j} \right] \right| \leq \widetilde{O}\left(\frac{\|w_i^{(t)} \|_2}{d\sqrt{d_1}}\right)\cdot \max\{\Pr(|\vbrack{w_i^{(t)} ,x_p^{++}}|\geq b_i^{(t)} ), \widetilde{O}(1/\sqrt{d})\}
    \end{align*}
\end{lemma}

\begin{proof}
    Again in this proof we omit the time superscript \(^{(t)}\). First we deal with the case where the features under consideration is \(\M_j, j\in [d]\). Since after the \(\RandomMask\) augmentations, \(2\DD\xi_p\) and \(2(\Id-\DD)\xi_p\) are independent (conditioned on fixed \(\DD\)), we denote \(2\DD\xi_p = 2\DD\xi'\) and \(2(\Id-\DD)\xi_p = 2(\Id-\DD)\xi''\), where \(\xi'\) and \(\xi''\) are independent. Now we can write as follows:
    \begin{align*}
        &\quad \,\E\left[h_i(x_p^{++})\1_{|\vbrack{w_i,x_p^+}|\geq b_i }  \vbrack{2\DD\xi_p,\M_j} \right]\\
        & =\E\left[ \left( \vbrack{w_i,x_p^{++}} - b_i \right)\1_{\vbrack{w_i,x_p^{++}}\geq b_i}\1_{|\vbrack{w_i,x_{p}^{+}}|\geq b_i}\vbrack{2\DD\xi',\M_j}\right] \\
        &\quad +\E\left[\left( \vbrack{w_i,x_p^{++}} + b_i \right)\1_{\vbrack{w_i,x_p^{++}}\leq -b_i}\1_{|\vbrack{w_i,x_{p}^{+}}|\geq b_i}\vbrack{2\DD\xi',\M_j}\right]
    \end{align*}
    For the first term on the RHS, we have 
    \begin{align*}
        &\quad \, \E\left[ \left( \vbrack{w_i,x_p^{++}} - b_i \right)\1_{\vbrack{w_i,x_p^{++}}\geq b_i}\1_{|\vbrack{w_i,x_{p}^{+}}|\geq b_i}\vbrack{2\DD\xi',\M_j}\right]\\
        & =\E\left[ \left( \vbrack{w_i,x_p^{++}} - b_i \right)\1_{\vbrack{w_i,x_p^{++}}\geq b_i}\1_{|\vbrack{w_i,x_{p}^{+}}|\geq b_i}(\vbrack{\xi',\M_j}+\vbrack{(2\DD-\Id)\xi',\M_j})\right]\\
        & = \E\left[ \left( \vbrack{w_i,x_p^{++}} - b_i \right)\1_{\vbrack{w_i,x_p^{++}}\geq b_i}\1_{|\vbrack{w_i,x_{p}^{+}}|\geq b_i}\vbrack{\xi',\M_j}\right]\\
        & \quad + \E\left[ \left( \vbrack{w_i,x_p^{++}} - b_i \right)\1_{\vbrack{w_i,x_p^{++}}\geq b_i}\1_{|\vbrack{w_i,x_{p}^{+}}|\geq b_i}\sum_{j'\in [d_1]}\vbrack{(2\DD-\Id)\widehat{\M}_{j'},\M_j}\vbrack{\xi',\widehat{\M}_{j'}} \right]\\
        & = I_1 + I_2
    \end{align*}
    where \(\{\widehat{\M}_j\}_{j\in[d_1]}\) is a basis for \(\R^{d_1}\) satisfying \(\|\widehat{\M}_{j}\|_{\infty}\leq O(1/\sqrt{d_1})\). For \(I_1\), notice that we can use approach similar to the proof of \myref{lem:grad-positive-1}{Lemma} as (denoting \([x]_+:= x\1_{x\geq 0}\))
    \begin{align*}
        & \E\left[ \left( \vbrack{w_i,x_p^{++}} - b_i \right)\1_{\vbrack{w_i,x_p^{++}}\geq b_i}\1_{|\vbrack{w_i,x_{p}^{+}}|\geq b_i}\vbrack{\xi',\M_j}\right] \\
        \leq \ &\E\left[ \left[ \vbrack{w_i,x_p^{++}} - b_i \right]_+\cdot\1_{h_i(x_p^{++})\neq 0}\left|\1_{|\vbrack{w_i,x_{p}^{+}}+ 2(\alpha-\alpha')|\vbrack{\xi',\M_j}||\geq b_i} - \1_{|\vbrack{w_i,x_{p}^{+}}- 2(\alpha-\alpha')|\vbrack{\xi',\M_j}| |\geq b_i}\right||\vbrack{\xi',\M_j}|\right]\\
        \leq \ &\E_{|\Id - \M_j\M_j^{\top})\xi'} \Big[ [\vbrack{w_i,x_p^{++}} - b_i]_+ |\vbrack{\xi',\M_j}|\1_{\vbrack{w_i,x_p^{++}}\geq b_i}\times \\
        &\qquad \times\E_{\Id - \M_j\M_j^{\top})\xi'}\left[\left|\1_{|\vbrack{w_i,x_{p}^{+}}+ 2(\alpha-\alpha')|\vbrack{\xi',\M_j}||\geq b_i} - \1_{|\vbrack{w_i,x_{p}^{+}} - 2(\alpha-\alpha')|\vbrack{\xi',\M_j}| |\geq b_i}\right|\right]\Big]\\
        \leq\ & O(1)\cdot\E\left[(\alpha_j + |\bar{\alpha}_j|)^2\vbrack{\xi',\M_j}^3\cdot \frac{[\vbrack{w_i,x_{p}^{++}}-b_i]_+ }{\E[\vbrack{w_i,\xi'}^2]}\1_{\vbrack{w_i,x_p^{++}}\geq b_i}\right]
    \end{align*}
    where in the last inequality we have used the randomness of \( (\Id - \M_j\M_j^{\top})\xi'\), which allow us to obtain the denominator \(\Omega(1)\E|\vbrack{w_i,\xi}^2|\). For \(I_2\), notice that w.h.p., we have 
    \begin{align*}
        |\vbrack{(2\DD-\Id)\widehat{\M}_{j'},\M_j}|\leq \widetilde{O}\left(\frac{1}{\sqrt{d_1}}\right)
    \end{align*}
    Denote \(\alpha'_{j'} = \vbrack{w_i,\widehat{\M}_{j'}}\) \(j'\in[d_1]\) and \(\{\bar{\alpha}'_{j'} = \vbrack{w_i,(\Id-2\DD)\widehat{\M}_{j'}}\) \(j'\in[d_1]\). Noticing that \(\sum_{j'\in[d_1]}(|\alpha'_{j'}|^2 +|\bar{\alpha}'_j|^2) = O\left(\|w_i\|_2^2\right)\) coupled with Cauchy-Schwarz inequality, we can similarly obtain:
    \begin{align*}
        I_2 &\leq \widetilde{O}\left(\frac{1}{\sqrt{d_1}}\right)\E\left[\sum_{j'\in[d_1]}(\alpha'_{j'} + |\bar{\alpha}'_{j'}|)^2|\vbrack{\xi',\M_j}|\vbrack{\xi',\widehat{\M}_{j'}}^2\cdot \frac{[\vbrack{w_i,x_{p}^{++}}-b_i]_+ }{\E[\vbrack{w_i,\xi'}^2]}\1_{\vbrack{w_i,x_p^{++}}\geq b_i}\right] \\
        &\leq \widetilde{O}\left(\frac{1}{\sqrt{d_1}}\right) \E\left[O(\|w_i\|_2^2)\cdot |\vbrack{\xi',\M_j}|\vbrack{\xi',\widehat{\M}_{j'}}^2\right]\cdot\frac{1}{\|w_i\|_2^2/d}\times \\
        &\qquad\times \max\left\{ \widetilde{O}\left(\frac{1}{\sqrt{d}}\Pr(|\vbrack{w_i,x_p^{++}}|\geq b_i)\right),\frac{1}{d} \right\} \\
        &\leq \widetilde{O}\left(\frac{\|w_i\|_2}{d\sqrt{d_1}}\right)\cdot \max\left\{\Pr(|\vbrack{w_i,x_p^{++}}|\geq b_i),\frac{1}{\sqrt{d}} \right\}
    \end{align*}
    where in the second inequality we have used the following arguments: first we can compute 
    \begin{align*}
        &\quad\ [\vbrack{w_i,x_{p}^{++}}-b_i]_+ \\
        &\leq \sum_{j\in\N_i}|\vbrack{w_i,\M_j}z_{p,j}| + \left|\sum_{j\notin\N_i}\vbrack{w_i,\M_j}z_{p,j}\right| + |\vbrack{w_i,(2\DD-\Id)\M z_p}| + |\vbrack{w_i,2(\DD-\Id)\xi_p}| \\
        & \leq \clubsuit + \spadesuit + \heartsuit + \diamondsuit
    \end{align*}
    And from \myref{induct-2}{Induction Hypothesis}, \myref{lem:activation-size}{Lemma} and \myref{lem:activation-size-2}{Lemma}, we have 
    \begin{align*}
        |\clubsuit| &\geq \Omega(\|w_i^{(t)}\|_2) \quad \text{with prob \(\leq \widetilde{O}(\frac{1}{d})\)}\qquad & |\spadesuit|, |\heartsuit|, |\diamondsuit| &\leq \widetilde{O}(\frac{\|w_i\|_2}{\sqrt{d}}) \quad \text{w.h.p.}
    \end{align*}
    Summing up over \(I_1,I_2\), we have the desired bound. For the dense feature \(\Mperp_j\), the analysis is similar and we omit for brevity.
\end{proof}

\subsection{The Learning Process at the Second Stage}

The second stage is defined as the iterations \( t \geq T_1\) but \(t\leq T_2\), where \(T_2 = \Theta\left(\frac{d\log d}{\eta\log\log d}\right)\) is defined as the iteration when one of the neuron \(i\in[m]\) satisfies \(\|w_i^{(T_2)}\|_2^2 \geq d\|w_i^{(T_1)}\|_2^2\). Our theorem for the training process in this stage is presented below:

\begin{theorem}[Emergence of singletons]\label{thm:2nd-stage}
    For each neuron \(i \in [m]\), not only \myref{induct-2}{Induction Hypothesis} but also the following conditions holds at iteration \(t = T_2\):
    \begin{enumerate}
        \item[(a)] For each \(j \in [d]\), if \(i \in \Mcal^{\star}_j\), then \(|\vbrack{w_i^{(T_2)} ,\M_j}| \geq \Omega(1)\|w_i^{(T_2)}\|_2\);
        \item[(b)] \(b_i^{(T_2)} \geq \frac{\polylog(d)}{\sqrt{d}}\|w_i^{(T_2)}\|_2\);
        \item[(c)] Let \(\alpha_{j}^* = \max_{i\in\Mcal^{\star}_j}|\vbrack{w_i^{(T_2)},\M_j}|\), then there is a constant \(C_j = \Theta(1)\) such that \(|\vbrack{w_i^{(t)},\M_j}| \leq C_j\alpha^*_j\) for all \(i \in \Mcal_j\).
    \end{enumerate}
\end{theorem}

Before proving this theorem, we prove \myref{induct-2}{Induction Hypothesis} as a preliminary step.

\begin{proof}[Proof of \myref{induct-2}{Induction Hypothesis}]
    At iteration \(t = T_1\), we have verified all the above properties in \myref{thm:initial-stage}{Theorem}. Now suppose all the properties hold for \(t < T_2\), we will verify that it still hold for \(t+1\). In order to calculate the gradient \(\nabla_{w_i}\Obj\) along each feature \(\M_j\) or \(\Mperp_j\), we have to apply \myref{lem:positive-gd-sparse-2}{Lemma}, \myref{lem:positive-gd-noise-2}{Lemma} and \myref{lem:grad-positive-1}{Lemma}. First we calculate parameters in \myref{lem:positive-gd-sparse-2}{Lemma} (a) and (b). In order to using \myref{lem:activation-size}{Lemma}, we have the followings
    \begin{itemize}
        \item \(|\bar{S}_{i,t}^{\setminus j}|^2 = | \vbrack{w_i^{(t)},(\Id-2\DD)x^{\setminus j}}|^2 \leq \widetilde{O}(\frac{\|w_i^{(t)}\|_2^2}{d})\);
        \item \(\Pr(A_1),\Pr(A_2) \leq  e^{-\Omega(\log^{1/4}d)}\) \hfill when \(|\alpha_{i,j}^{(t)}|\leq (1 - c_0/2)b_i^{(t)}\);
        \item \(\Pr(A_3),\Pr(A_4) \leq e^{-\Omega(\log^{1/4}d)}\) \hfill when \(|\alpha_{i,j}^{(t)}|\geq (1 + c_0/2)b_i^{(t)}\).
    \end{itemize}
    Which further implies that
    \begin{align*}
        \sqrt{\E[|\bar{S}^{\setminus j}|^2(\1_{A_1} + \1_{A_2})]} &\leq  \sqrt{\widetilde{O}\bigg(\frac{\|w_i^{(t)}\|_2^2}{d}\bigg)(\Pr(A_1) + \Pr(A_2))}  \leq \frac{\|w_i^{(t)}\|_2}{\sqrt{d}\polylog(d)}\\
        \implies \ & L_1, L_2,  \leq \frac{1}{\polylog(d)} \tag{ when \(|\alpha_{i,j}^{(t)}|\leq (1 - c_0/2)b_i^{(t)}\)}
    \end{align*}
    And similarly, we also have \(L_3, L_4,  \leq \frac{1}{\polylog(d)}\) when \(|\alpha_{i,j}^{(t)}|\geq (1 + c_0/2)b_i^{(t)}\). Now we separately discuss three cases:
    \begin{itemize}
        \item[(a)] When \(i \in \Mcal_j^{\star}\), if \(z_j\neq 0\), say \(z_j = 1\), we simply have
        \begin{displaymath}
            \Pr\left(|\vbrack{w_i^{(t)},x_p}|\geq b_i^{(t)} +|\vbrack{w_i^{(t)},x_p^+-x_p}| \right) \geq 1 - \Pr\left(|\vbrack{w_i^{(t)},x_p^{\setminus j}}|\geq b_i^{(t)} - \alpha_{i,j}^{(t)} + |\vbrack{w_i^{(t)},x_p^+-x_p}|\right)
        \end{displaymath}
        from the observations that: (1) \(|\vbrack{w_i^{(t)},x_p^{\setminus j}}|\leq \frac{c_0}{2}b_i^{(t)} \) with probability \(\geq 1 - e^{-\Omega(\log^{1/4}d)}\); (2) \(|\vbrack{w_i^{(t)},x_p^+-x_p}|\leq O(\|w_i^{(t)}\|_2\sigma_{\xi})\) with prob \(\geq 1 - e^{-\Omega(\log^{1/2}d)}\). So it can be easily verified that
        \begin{align*}
            \E\left[z_j^2\1_{|\vbrack{w_i^{(t)},x_p}|\geq b_i^{(t)} +|\vbrack{w_i^{(t)},x_p^+-x_p}|}\right] = \frac{C_z\log\log d}{d}\left(1 - \frac{1}{\polylog(d)}\right)
        \end{align*}
        Now we can compute as follows: for \(\M_j\) such that \(i\in\Mcal_j^{\star}\), at iteration \(t+1\):
        \begin{align*}
            &\quad\ \vbrack{w_i^{(t+1)},\M_j} \\
            & = \vbrack{w_i^{(t)},\M_j} - \eta\vbrack{\nabla_{w_i}\Obj(f_t),\M_j} \pm \frac{\eta\|w_i^{(t)} \|_2}{\poly(d_1)}\\
            & = \vbrack{w_i^{(t)},\M_j}(1 - \eta\lambda) \pm \frac{\eta\|w_i^{(t)} \|_2}{\poly(d_1)}\\
            &\quad + \eta \E\left[ (1 - \ell'_{p,t})h_{i,t}(x_p^{++})\1_{|\vbrack{w_i^{(t)},x_p^+}|\geq b_i^{(t)}}\left(z_{p,j} + \vbrack{ (2\DD-\Id)\M z_p,\M_j} + \vbrack{2\DD\xi_p,\M_j} \right) \right] \\
            &\quad -  \eta \E\left[ \sum_{x_{n,s}\in\Nfr} \ell'_{s,t}h_{i,t}(x_{n,s})\1_{|\vbrack{w_i^{(t)},x_p^+}|\geq b_i^{(t)}}\left(z_{p,j} + \vbrack{ (2\DD-\Id)\M z_{p},\M_j} + \vbrack{2\DD\xi_p,\M_j} \right) \right] \\
            & \geq \left(\vbrack{w_i^{(t)},\M_j} - \sign(\vbrack{w_i^{(t)},\M_j})\cdot b_i^{(t)}\right) \left(1 -\eta\lambda + \frac{\eta C_z\log\log d}{d} \left(1 - \frac{1}{\polylog(d)}\right)\right) \tag{By \myref{lem:positive-gd-sparse-2}{Lemma}} \\
            &\quad - O\left(\frac{\eta |\vbrack{w_i^{(t)},\M_j}|}{d\polylog(d)}\right) \pm O\left(\frac{\eta\sum_{i'\in[m]}\|w_{i'}^{(t)}\|_2^2\|w_i^{(t)}\|_2}{d\tau}\right) \pm \widetilde{O}\left(\frac{\eta\|w_i^{(t)}\|_2}{d\sqrt{d_1}}\right)\\
            & \geq \left(\vbrack{w_i^{(t)},\M_j} - \sign(\vbrack{w_i^{(t)},\M_j})\cdot b_i^{(t)}\right)  \left(1 + \frac{\eta C_z\log\log d}{d}\left(1- \frac{\eta}{\polylog(d)}\right) \right)
        \end{align*}
        where in the last inequality we have taken into consideration \(\sum_{i\in[m]}\|w_i^{(t)}\|_2^2 \leq \frac{1}{\poly(d_1)}\),which follows from our definition of iteration \(T_2\) and the properties at iteration \(T_1\) in \myref{thm:initial-stage}{Theorem}, and also that \(\vbrack{w_i^{(t)},\M_j}/d\geq b_i^{(t)}/d\gg \frac{\|w_i^{(t)}\|_2}{\sqrt{d_1}}\). Next we compare this growth to the growth of bias \(b_i^{(t+1)}\). Since we raise our bias by \(b_i^{(t+1)} =  \max\{b_i^{(t)}(1 +\frac{\eta}{d}), b_i^{(t)}\frac{\|w_i^{(t+1)}\|_2}{\|w_i^{(t)}\|_2}\}\), as long as \(\frac{\|w_i^{(t+1)}\|_2}{\|w_i^{(t)}\|_2} \leq \frac{|\vbrack{w_i^{(t+1)},\M_j}|}{|\vbrack{w_i^{(t)},\M_j}|}\), we can obtain the desired result (\(\frac{\|w_i^{(t+1)}\|_2}{\|w_i^{(t)}\|_2} \leq \frac{|\vbrack{w_i^{(t+1)},\M_j}|}{|\vbrack{w_i^{(t)},\M_j}|}\) will be proved later when we prove (d)).
        \item[(b)] When \(i \notin \Mcal_j\), we can similarly obtain that
        \begin{align*}
            \E\left[z_j^2\1_{|\vbrack{w_i^{(t)},x_p^+}|\geq b_i^{(t)} + |\vbrack{w_i^{(t)},x_p^+ - x_p}|}\right] \leq O\left(\frac{1}{d\polylog(d)}\right)
        \end{align*}
        And similarly we can compute the gradient descent dynamics as follows: For \(j \in [d]\) such that \(|\vbrack{w_i^{(t)},\M_j}| \geq \frac{\|w_i^{(t)}\|_2d}{\sqrt{d_1}}\), we have (assume here \(\vbrack{w_i^{(t)},\M_j}>0\), the opposite is similar)
        \begin{align*}
            &\quad \ \vbrack{w_i^{(t+1)},\M_j} \\
            & = \vbrack{w_i^{(t)},\M_j} - \eta\vbrack{\nabla_{w_i}\Obj(f_t),\M_j} + \frac{\eta\|w_i^{(t)} \|_2}{\poly(d_1)}\\
            & \leq \vbrack{w_i^{(t)},\M_j} \left(1 -\eta\lambda + \frac{O(\eta)}{d\polylog(d)} \right) \pm O\left(\frac{\eta\sum_{i'\in[m]}\|w_{i'}^{(t)}\|_2^2\|w_i^{(t)}\|_2}{d\tau}\right) \pm \widetilde{O}\left(\eta\frac{\|w_i^{(t)}\|_2}{d^2}\right)\\
            & \leq \vbrack{w_i^{(t)},\M_j} \left(1 + \frac{O(\eta)}{d\polylog(d)} \right) + \widetilde{O}\left(\eta\frac{\|w_i^{(t)}\|_2}{d^2}\right)
        \end{align*}
        Since from our update rule \(b_i^{(t+1)} \geq b_i^{(t)}(1 +  \frac{\eta}{d})\), we know that \(\frac{|\vbrack{w_i^{(t+1)},\M_j}|}{|\vbrack{w_i^{(t)},\M_j}|} \leq \frac{b_i^{(t+1)}}{b_i^{(t)}}\). Thus, if \(|\vbrack{w_i^{(t)},\M_j}| \leq (1-c_0)b_i^{(t)}\) at iteration \(t\), we have 
        \begin{itemize}
            \item \(|\vbrack{w_i^{(t+1)},\M_j}| \leq (1-c_0)b_i^{(t+1)}\) if \(|\vbrack{w_i^{(t)},\M_j}|\geq \frac{\|w_i^{(t)}\|_2d}{\sqrt{d_1}}\) at iteration \(t\);
            \item \(|\vbrack{w_i^{(t+1)},\M_j}| \leq \frac{\|w_i^{(t+1)}\|_2}{\sqrt{d}} \leq (1-c_0)b_i^{(t+1)}\) if \(|\vbrack{w_i^{(t)},\M_j}|\leq \frac{\|w_i^{(t)}\|_2d}{\sqrt{d_1}}\) at iteration \(t\).
        \end{itemize}
        It is also worth noting that similar calculations also leads to a lower bound 
        \begin{align}\label{eqref:lb-sparse}
            |\vbrack{w_i^{(t+1)},\M_j}| \geq |\vbrack{w_i^{(t)},\M_j}|(1 - \eta\lambda) - \widetilde{O}\left(\eta\frac{\|w_i^{(t)}\|_2}{d^{2}}\right)
        \end{align}
        We leave the part of proving \(|\vbrack{w_i^{(t+1)},\M_j}| \leq \widetilde{O}(\frac{\|w_i^{(t)}\|_2}{\sqrt{d}})\) to later.
        \item[(c)] The result (c) that there exist at most \(O(2^{-\sqrt{\log d}}d)\) many \(j\in[d]\) such that  \(|\vbrack{w_i^{(t)},\M_j}|^2\geq \frac{(b_i^{(t)})^2}{\log^{1/2} d}\) can be similarly proved.
        \item[(d)] Next we consider the learning dynamics for the dense features. We can use \myref{lem:positive-gd-noise-2}{Lemma} to calculate its dynamics by
        \begin{align*}
            &\quad\ \vbrack{w_i^{(t+1)},\Mperp_j} \\
            & = \vbrack{w_i^{(t)},\Mperp_j}(1 - \eta\lambda) \pm \frac{\eta\|w_i^{(t)} \|_2}{\poly(d_1)}\\
            &\quad + \eta \E\left[ (1 - \ell'_{p,t})h_{i,t}(x_p^{++})\1_{|\vbrack{w_i^{(t)},x_p^+}|\geq b_i^{(t)}}\left(\vbrack{ (2\DD-\Id)\M z_p,\Mperp_j} + \vbrack{2\DD\xi_p,\Mperp_j} \right) \right] \\
            &\quad -  \eta \sum_{x_{n,s}\in\Nfr} \E\left[  \ell'_{s,t}h_{i,t}(x_{n,s})\1_{|\vbrack{w_i^{(t)},x_p^+}|\geq b_i^{(t)}}\left(\vbrack{ (2\DD-\Id)\M z_{p},\Mperp_j} + \vbrack{2\DD\xi_p,\Mperp_j} \right) \right] \\
            & = \vbrack{w_i^{(t)},\Mperp_j}(1 - \eta\lambda) + \widetilde{O}\left(\frac{\eta\|w_i^{(t)}\|_2}{d\sqrt{d_1}}\right)\cdot\Pr(h_{i,t}(x_p^{++})\neq 0) \\
            &\leq \vbrack{w_i^{(t)},\Mperp_j} + O\left(\frac{\eta\|w_i^{(t)}\|_2}{d\sqrt{d_1}}e^{-\Omega(\log^{1/4}d)}\right)
        \end{align*}
    \end{itemize}
    After establishing the bounds of growth speed for each features, we now calculate the propotions they contribute to each neuron weight \(i \in [m]\). Namely, we need to prove that when \myref{induct-2}{Induction Hypothesis} holds at iteration \(t \in [T_1,T_2]\), we have
    \begin{itemize}
        \item To prove \( \frac{|\vbrack{w_i^{(t+1)},\M_j}|}{|\vbrack{w_i^{(t)},\M_j}|}\geq \frac{\|w_i^{(t+1)}\|_2}{\|w_i^{(t)}\|_2}\) for \(i \in \Mcal_j^{\star}\), we argue as follows: from previous calculations we have 
        \begin{align*}
            &\sum_{j'\in [d], j'\neq j}\vbrack{w_i^{(t+1)},\M_{j'}}^2 + \sum_{j'\in[d_1]\setminus [d]}\vbrack{w_i^{(t+1)},\Mperp_{j'}}^2\\
            &\leq \sum_{j'\in [d], j'\neq j}\vbrack{w_i^{(t)},\M_{j'}}^2(1 + \frac{O(\eta)}{\polylog(d)}) + \sum_{j'\in[d_1]\setminus [d]}\vbrack{w_i^{(t)},\Mperp_{j'}}^2 + \widetilde{O}(\frac{\eta}{d})e^{-\Omega(\log^{1/4}d)}\|w_i^{(t)}\|_2^2 
        \end{align*}
        Therefore by adding \(\vbrack{w_i^{(t+1)},\M_j}^2\) to the LHS we have
        \begin{align*}
            \|w_i^{(t+1)}\|_2^2 \leq \|w_i^{(t)}\|_2^2(1 + \frac{O(\eta)}{d\polylog(d)})^2 + (\frac{|\vbrack{w_i^{(t+1)},\M_j}|}{|\vbrack{w_i^{(t)},\M_j}|} - \frac{O(\eta)}{d\polylog(d)})|\vbrack{w_i^{(t)},\M_j}|^2
        \end{align*}
        which implies \( \frac{|\vbrack{w_i^{(t+1)},\M_j}|}{|\vbrack{w_i^{(t)},\M_j}|}\geq \frac{\|w_i^{(t+1)}\|_2}{\|w_i^{(t)}\|_2}\) as desired.
        \item To prove \(|\vbrack{w_i^{(t+1)},\M_j}|\leq \widetilde{O}(\frac{\|w_i^{(t+1)}\|_2}{\sqrt{d}})\) if \(i \notin \Mcal_j\), we first use inequality \eqref{eqref:lb-sparse} to compute 
        \begin{align*}
            \|\M\M^{\top}w_i^{(t+1)}\|_2 & \geq \|\M\M^{\top}w_i^{(T_1)}\|_2(1 - \eta\lambda)^{t-T_1} - O(\frac{\eta(t-T_1+1) \max_{t'\in [T_1,t+1]}\|w_i^{(t')}\|_2}{\sqrt{d d_1}}) \\
            &\geq \|\M\M^{\top}w_i^{(T_1)}\|_2(1 - \eta\lambda)^{t-T_1+1} - O(\frac{\eta(t-T_1) \|w_i^{(T_1)}\|_2\sqrt{d}}{\sqrt{d_1}})\\
            &\geq  \|\M\M^{\top}w_i^{(T_1)}\|_2( 1 - o(1)) \tag*{for \( t \leq \Theta(\frac{d\log d}{\eta \log\log d})\).}
        \end{align*}
        Notice that \(|\vbrack{w_i^{(T_1)},\M_j}| \leq O(\sqrt{\frac{\log d}{d}})\|\M\M^{\top}w_i^{(T_1)}\|_2\) from \myref{thm:initial-stage}{Theorem}. Suppose it also holds for iteration \(t\), we have 
        \begin{align*}
            |\vbrack{w_i^{(t+1)},\M_j}| &\leq |\vbrack{w_i^{(t)},\M_j}|( 1 + \frac{O(\eta)}{d\polylog(d)}) + \widetilde{O}(\frac{\eta\|w_i^{(t)}\|_2}{d^2})\\
            & \leq  |\vbrack{w_i^{(T_1)},\M_j}|( 1 + \frac{O(\eta)}{d\polylog(d)})^{t-T_1} +\widetilde{O}(\frac{\eta(t-T_1) \|w_i^{(T_1)}\|_2}{d^{3/2}}) \tag{because \(\|w_i^{(t)}\|_2\leq \sqrt{d}\|w_i^{(T_1)}\|_2\) by the definition of \(T_2\)}\\
            & \leq |\vbrack{w_i^{(T_1)},\M_j}|(1 + o(1)) + \widetilde{O}(\frac{ \|w_i^{(T_1)}\|_2}{d}) \\
            &\leq O(\sqrt{\frac{\log d}{d}}) \|\M\M^{\top}w_i^{(T_1)}\|_2 \leq O(\sqrt{\frac{\log d}{d}}) \|\M\M^{\top}w_i^{(t+1)}\|_2 \\
            &\leq O(\sqrt{\frac{\log d}{d}}) \|w_i^{(t+1)}\|_2
        \end{align*}
        \item For the dense features, we can compute as follows:
        \begin{align*}
            &\quad\, |\vbrack{w_i^{(t + 1)},\Mperp_j}| \\
            &\leq  |\vbrack{w_i^{(t)},\Mperp_j}| + O\left(\frac{\eta\|w_i^{(t)}\|_2}{d\sqrt{d_1}}e^{-\Omega(\log^{1/4}d)}\right) \\
            &\leq |\vbrack{w_i^{(T_1)},\Mperp_j}| + \sum_{t'=T_1}^{t}O\left(\frac{\eta\|w_i^{(t')}\|_2}{d\sqrt{d_1}}e^{-\Omega(\log^{1/4}d)}\right)\\
            &\leq   O\left(\sqrt{\frac{\log d}{d_1}}\right)\|w_i^{(T_1+1)}\|_2 \left(1 + \frac{O(\eta)}{d\polylog(d)}\right)^2 + \sum_{t'=T_1+1}^{t}O\left(\frac{\eta\|w_i^{(t')}\|_2}{d\sqrt{d_1}}e^{-\Omega(\log^{1/4}d)}\right)\\
            &\leq O\left(\sqrt{\frac{\log d}{d_1}}\right)\|w_i^{(t+1)}\|_2\left(1 + \frac{O(\eta)}{d\polylog(d)}\right)^{2(t-T_1+1)}\\
            &\leq O\left(\sqrt{\frac{\log d}{d_1}}\right)\|w_i^{(t+1)}\|_2
        \end{align*}
        where we have used the assumption that \(\|w_i^{(t)}\|_2 \leq \|w_i^{(t+1)}\|_2(1 + \frac{O(\eta)}{d\polylog(d)})\) for all \(i \in [m]\) and all \( t \leq T_2 = \Theta(\frac{d\log d}{\eta \log\log d})\), which we prove here: First of all, from previous calculations we have
        \begin{align*}
            \|\M\M^{\top}w_i^{(t+1)}\|_2 & \geq \|\M\M^{\top}w_i^{(t)}\|_2(1 - \eta\lambda) - O\left(\frac{\eta\|w_i^{(t)}\|_2}{d^{3/2}}e^{-\Omega(\log^{1/4}d)}\right)
        \end{align*}
        also the trajectory of \(\|\Mperp(\Mperp)^{\top}w_i^{(t+1)}\|_2\) can be lower bounded as
        \begin{align*}
            \|\Mperp(\Mperp)^{\top}w_i^{(t+1)}\|_2 &\geq \|\Mperp(\Mperp)^{\top}w_i^{(t)}\|_2(1 - \eta\lambda) - O(\eta/d)e^{-\Omega(\log^{1/4}d)}\|w_i^{(t)}\|_2 
        \end{align*}
        thus by combining the change of \(w_i^{(t)}\) over two subspaces, we have 
        \begin{align*}
            \|w_i^{(t+1)}\|_2^2 &\geq \|w_i^{(t)}\|_2^2(1 - \eta \lambda)^2 - O(\eta/d)e^{-\Omega(\log^{1/4}d)}\|w_i^{(t)}\|_2^2 \geq \|w_i^{(t)}\|_2^2(1 - \frac{O(\eta)}{d\polylog(d)})
        \end{align*}
        which gives the desired bound.
    \end{itemize}
    In the proof above, we have depend on the crucial assumption that \(T_2:=\min\{t \in \mathbb{N}: \exists i\in[m] \text{ s.t. } \|w_i^{(t)}\|_2^2 \geq d\|w_i^{(T_1)}\|_2^2 \}\) is of order \(\Theta(\frac{d \log d}{\eta \log\log d})\). Now we verify it as follows. If \( i \in \Mcal_j^{\star}\) for some \(j \in [d]\) (which also means \(j' \notin \N_i\) for \(j'\neq j\)), we have 
    \begin{align*}
        |\vbrack{w_i^{(t)} ,\M_j}| \geq |\vbrack{w_i^{(T_1)},\M_j}|\left(1 + \Omega(\frac{\eta \log\log d}{d})\right)^{t-T_1}
    \end{align*}
    Thus for some \(t = O(\frac{d\log d}{\eta \log\log d})\), we have \(|\vbrack{w_i^{(t)} ,\M_j}|^2 \geq d\|w_i^{(T_1)}\|_2^2\), which proves that \(T_2 \leq O(\frac{d\log d}{\eta \log\log d})\). Conversely, we also have for all \(t\leq O(\frac{d\log d}{\eta \log\log d})\)
    \begin{align*}
        &\quad\, \sum_{j'\in[d]:j'\neq j}\vbrack{w_i^{(t)},\M_{j'}}^2 + \sum_{j'\in[d_1]\setminus [d]}\vbrack{w_i^{(t)},\Mperp_{j'}}^2 \\
        &\leq \|w_i^{(T_1)}\|_2^2(1 + \frac{O(\eta)}{d\polylog(d)})^{t-T_1} + \max_{t'\leq t}O(\eta(t - T_1)/d)e^{-\Omega(\log^{1/4}d)}\|w_i^{(t')}\|_2^2 \\
        &\leq o(d\|w_i^{(T_1)}\|_2^2)
    \end{align*}
    And also 
    \begin{align*}
        |\vbrack{w_i^{(t)} ,\M_j}| &\leq |\vbrack{w_i^{(T_1)},\M_j}|\left(1 + \frac{C_z\eta \log\log d}{d}(1 - \frac{1}{\polylog(d)})\right)^{t-T_1} \\
        &\leq O(\sqrt{\frac{\log d}{d}}\|w_i^{(T_1)})\|_2\left(1 + \frac{C_z\eta \log\log d}{d}(1 - \frac{1}{\polylog(d)})\right)^{t-T_1}
    \end{align*} 
    Therefore we at least need \(\frac{d\log (\Omega(\sqrt{d/\log d}))}{\eta C_z\log\log d}(1 - o(1))\) iteration to let any neuron \(i\in [m]\) reach \(\|w_i^{(t)}\|_2^2\geq d\|w_i^{(T_1)}\|_2\), which proves that \(T_2 = \Theta(\frac{d\log d}{\eta \log\log d})\).
\end{proof}

\begin{proof}[Proof of \myref{thm:2nd-stage}{Theorem}]
    We follow similar analysis as in the proof of \myref{induct-2}{Induction Hypothesis}. In order to prove (a) -- (d), we have to discuss the two substages of the learning process below.
    \begin{itemize}
        \item \textcolor{blue}{When all \(\|w_i^{(t)}\|_2\leq 2\|w_i^{(T_1)}\|_2\):} From similar analysis in the proof of \myref{induct-2}{Induction Hypothesis}, the iteration complexity for a neuron \(i \in [m]\) to reach \(\|w_i^{(t)}\|_2 \geq 2\|w_i^{(T_1)}\|_2\) is no smaller than \(T'_{i,1} :=\max\{\Omega(\frac{d\log d}{\eta \log\log d}), T_2\}\). At this substage, we have 
        \begin{itemize}
            \item the bias growth is large, i.e., 
            \begin{align*}
                b_i^{(T'_{i,1})} \geq b_i^{(T_1)}(1 + \eta/d)^{T'_{i,1}-T_1} \geq b_i^{(T_1)}\polylog(d) \geq \frac{\polylog(d)}{\sqrt{d}}\|w_i^{(T_1)}\|_2\geq \frac{\polylog(d)}{\sqrt{d}}\|w_i^{(T'_{i,1})}\|_2
            \end{align*}
            \item For \(j \notin \N_i\) we have 
            \begin{align*}
                \sum_{j\in[d],j\notin\N_i}\vbrack{w_i^{(T'_{i,1})},\M_j}^2 &\leq \sum_{j\in[d],j\notin\N_i}\vbrack{w_i^{(T_{1})},\M_j}^2\left(1 + \frac{O(\eta)}{d\polylog(d)}\right)^{T_2} + \widetilde{O}\left(\frac{\eta\|w_i^{(T_1)}\|_2^2}{d^{3/2}}\right)\\
                &\leq (1 + o(1))\|\M\M^{\top}w_i^{(T_1)}\|_2^2 \tag{since \(\|w_i^{(T_1)}\|_2 \lesssim \|\M\M^{\top}w_i^{(T_1)}\|_2\)}
            \end{align*}
            \item For \(j \in [d_1]\setminus[d]\) we have 
            \begin{align*}
                \sum_{j\in[d_1]\setminus [d]}\vbrack{w_i^{(T'_{i,1})},\Mperp_j}^2 &\leq \sum_{j\in[d_1]\setminus [d]}\vbrack{w_i^{(T_{1})},\Mperp_j}^2 + O(\eta(T'_{i,1}-T_1)/d)e^{-\Omega(\log^{1/4}d)}\max_{t'\in [T_1,T'_{i,1}]}\|w_i^{(t')}\|_2^2\\
                &\leq (1 + o(1))\|\Mperp(\Mperp)^{\top}w_i^{(T_1)}\|_2^2
            \end{align*}
            \item If \(i\in\Mcal_j^{\star}\), there exist \(t \leq T_2\) such that \(\|w_i^{(t)}\|_2 \geq 2 \|w_i^{(T_2)}\|_2\), as we have argued in the proof of \myref{induct-2}{Induction Hypothesis}. Thus we have 
            \begin{align*}
                |\vbrack{w_i^{(T'_{i,1})},\M_j}|^2 &\geq \|w_i^{(T'_{i,1})}\|_2^2 - \sum_{j\in[d],j\notin\N_i}\vbrack{w_i^{(T'_{i,1})},\M_j}^2 - \sum_{j\in[d_1]\setminus [d]}\vbrack{w_i^{(T'_{i,1})},\Mperp_j}^2 \\
                &\geq 2\|w_i^{(T_1)}\|_2^2 - (1 + o(1))\|w_i^{(T_1)}\|_2^2 \geq (1 - o(1))\|w_i^{(T_1)}\|_2^2
            \end{align*}
            which proves the claim.
        \end{itemize}
        \item \textcolor{blue}{When some \(\|w_i^{(t)}\|_2\geq 2\|w_i^{(T_1)}\|_2\):} At this substage, we have 
        \begin{itemize}
            \item The bias is large consistently, i.e.,
            \begin{align*}
                b_i^{(t+1)} \geq b_i^{(t)}\cdot\frac{\|w_i^{(t+1)}\|_2}{\|w_i^{(t)}\|_2} \geq  \frac{\polylog(d)}{\sqrt{d}}\|w_i^{(t+1)}\|_2 \geq \frac{1}{4}\|w_i^{(T'_{i,1})}\|_2
            \end{align*}
            \item If \(i \in \Mcal_j^{\star}\), then from similar calculations as above, we can prove by induction that starting from \(t = T'_{i,1}\), it holds:
            \begin{align*}
                |\vbrack{w_i^{(t+1)},\M_j}| &\geq |\vbrack{w_i^{(t)},\M_j}|\left(1 + \Omega(\frac{\eta\log\log d}{d})\right) \geq \|w_i^{(t)}\|_2\left(1 + \Omega(\frac{\eta\log\log d}{d})\right)\\
                \sum_{j'\in[d],j'\neq j}\vbrack{w_i^{(t+1)},\M_{j'}}^2 & \leq \sum_{j'\in[d],j'\neq j}\vbrack{w_i^{(t)},\M_{j'}}^2(1 + \frac{O(\eta)}{d\polylog(d)})^2\\
                \sum_{j\in[d_1]\setminus [d]}\vbrack{w_i^{(t+1)},\Mperp_j}^2 & \leq \sum_{j\in[d_1]\setminus [d]}\vbrack{w_i^{(t)},\Mperp_j}^2(1 + \frac{O(\eta)}{d\polylog(d)})^2
            \end{align*}
            which implies
            \begin{align*}
                |\vbrack{w_i^{(t+1)},\M_j}| \geq |\vbrack{w_i^{(t)},\M_j}|  \frac{\|w_i^{(t+1)}\|_2}{\|w_i^{(t)}\|_2} \geq (1 - o(1))\|w_i^{(t+1)}\|_2
            \end{align*}
        \end{itemize}
    \end{itemize}
    Now we only need to prove (c). Assuming \(\vbrack{w_i^{(t)},\M_j} > 0\) (the opposite case is similar), from \(t = T_1\), for \(i \in \Mcal_j^{\star}\), we have 
    \begin{align*}
        \vbrack{w_i^{(t+1)},\M_j} &= (\vbrack{w_i^{(t)},\M_j} - b_i^{(t)})\left(1 + \frac{\eta C_z\log\log d}{d}  \right) \pm O(\frac{\eta|\vbrack{w_i^{(t)},\M_j}|}{d\polylog(d)}) \\
        &\geq \Omega(1)\vbrack{w_i^{(t)},\M_j}\left(1 + \frac{\eta C_z\log\log d}{d} \left(1 - \frac{1}{\polylog(d)}\right) \right)\\
        &\geq \Omega(1)\vbrack{w_i^{(T_1)},\M_j}\left(1 + \frac{\eta C_z\log\log d}{d} \left(1 -\frac{1}{\polylog(d)}\right) \right)^{t-T_1}
    \end{align*}
    which implies that after certain iteration \(t = T_1 + T'\), where \(T' = \Theta( \frac{d}{\eta})\), we shall have
    \begin{align*}
        |\vbrack{w_i^{(T_1 + T')},\M_j}| \geq |\vbrack{w_i^{(T_1)},\M_j}| \geq \polylog(d)|\vbrack{w_i^{(T_1)},\M_j}| \geq b_i^{(T_1)}\polylog(d)
    \end{align*}
    However, at iteration \(t = T_1 + \Theta(\frac{d}{\eta})\), we can see from previous analysis that \(\|w_i^{(t)}\|_2 \leq (1 + o(1))\|w_i^{(T_1)}\|_2\), so the bias growth can be bounded as
    \begin{align*}
        b_i^{(t)} \leq b_i^{(T_1)}(1 + \frac{\eta}{d})^{\Theta(\frac{d}{\eta})}\cdot \max\left\{\frac{\|w_i^{(t)}\|_2}{\|w_i^{(T_1)}\|_2} , 1 \right\}\leq O(b_i^{(T_1)})
    \end{align*}
    Now from our initialzaition properties in \myref{lem:property-init}{Lemma}, we have that \(\vbrack{w_{i'}^{(0)},\M_j}^2 \leq O(\sigma_0^2\log d)\) for all \(i \in [m]\). Thus via similar arguments, we also have 
    \begin{align*}
        |\vbrack{w_{i'}^{(t)},\M_j}| \leq \vbrack{w_{i'}^{(0)},\M_j}\left(1 + \frac{\eta C_z\log\log d}{d} \left(1\pm \frac{1}{\polylog(d)}\right) \right)^t
    \end{align*}
    holds for all \(i'\in[m]\). Now it is easy to see that for \(t \leq T_2 = \Theta(\frac{d\log d}{\eta \log\log d})\), we have 
    \begin{align*}
        \frac{|\vbrack{w_i^{(t)},\M_j}|}{|\vbrack{w_{i'}^{(t)},\M_j}|} \geq \Omega(1) \frac{|\vbrack{w_{i}^{(0)},\M_j}|\cdot\left(1 + \frac{\eta C_z\log\log d}{d}\left(1- \frac{\eta}{\polylog(d)}\right) \right)^t}{|\vbrack{w_{i'}^{(0)},\M_j}|\left(1 + \frac{\eta C_z\log\log d}{d} + \frac{\eta}{\polylog(d)}\right)^t} \geq \left(1 - O(\frac{\eta \log\log d}{d\polylog(d)})\right)^t \geq \Omega(1)
    \end{align*}
    Thus the last claim is proved.
\end{proof}

\section{Stage III: Convergence to Sparse Features}

At the final stage, we are going to prove that as long as the neurons are sparsely activated, they will indeed converge to sparse solutions, which ensures sparse representations. We present the statement of our convergence theorem below.

\begin{theorem}[Convergence]\label{thm:convergence}
    At iteration for \(t \in [\Omega(\frac{d^{1.01}}{\eta}), O(\frac{d^{1.99}}{\eta})]\), we have the following results:
    \begin{enumerate}
        \item[(a)] If \(i \in \Mcal^{\star}_j\), then \(|\vbrack{w_i^{(t)},\M_j}| \in \left[\frac{\tau}{\Xi_2},O(1) \right]\);
        \item[(b)] If \(i \notin \Mcal_j\), then \(|\vbrack{w_i^{(t)},\M_j}| \leq O\left(\frac{1}{d^{2}\lambda}\right)\);
        \item[(c)] For all dense feature \(\Mperp_j,\, j\in [d_1]\setminus [d]\), we have \(|\vbrack{w_i^{(t)},\Mperp_j}|\leq O(\frac{1}{\sqrt{d_1}d^{1.5}\lambda})\).
        \item[(d)] We have the loss convergence guarantees: let \(T_3 = \Omega(\frac{d^{1.01}}{\eta})\), for any \(T \leq O(\frac{d^{1.99}}{\eta})\), we have 
        \begin{align*}
            \frac{1}{T}\sum_{t = T_3}^{T_3 + T - 1}\E[\mathcal{L}(f_t,x_p^+,x_p^{++},\Nfr)]  \leq O(\frac{1}{\log d})
        \end{align*} 
    \end{enumerate}
\end{theorem}

To prove this theorem, we need the following induction hypothesis, which we shall show to hold throughout the final stage.

\begin{induct}[Induction hypothesis at final stage]\label{induct-3}
    For all \(t \geq T_2\):
    \begin{enumerate}
        \item If \(i \in \Mcal_j^{\star}\), then \(|\vbrack{w_i^{(t)},\M_j}| \geq \Omega(1) \|w_i^{(t)}\|_2\);
        \item  For \(i \in [m]\), we have \(\|w_i^{(t)}\|_2 \leq O(1)\)
        \item For each \(j \in [d]\), \(\F_{j}^{(t)} := \sum_{i\in\Mcal_{j}}\vbrack{w_{i}^{(t)},\M_{j}}^2\leq O(\tau \log^3 d)\);
        \item Let \(j \in [d]\) and \(i \in \Mcal_{j}^{\star}\), there exist \(C = \Theta(1)\) such that \(|\vbrack{w_i^{(t)},\M_{j}}| \geq C\max_{i'\in\Mcal_j}|\vbrack{w_{i'}^{(t)},\M_{j}}|\);
        \item For \({i} \notin \Mcal_j\), it holds \(|\vbrack{w_{i}^{(t)},\M_j}|\leq O(\frac{1}{\sqrt{d}\Xi_2^5})\|w_i^{(t)}\|_2\);
        \item For any \(i \in [m]\) and any \(j\in[d_1]\setminus [d]\), it holds \(|\vbrack{w_{i}^{(t)},\Mperp_j}|\leq O(\frac{1}{\sqrt{d_1}\Xi_2^5})\|w_i^{(t)}\|_2\);
        \item The bias \(b_i^{(t)} \geq \frac{\polylog(d)}{\sqrt{d}}\|w_i^{(t)}\|_2\).
    \end{enumerate}
\end{induct}
When all the conditions in \myref{induct-3}{Induction Hypothesis} hold for some iteration \(t \geq T_2\), we have the following fact, which is a simple corollary of \myref{lem:activation-size-3}{Lemma}.

\begin{fact}\label{fact:activation-3}
    For any \(i\in [m]\), we denote \(\N_i = \{j\in[d]:i\in\Mcal_j\}\). Suppose \myref{induct-3}{Induction Hypothesis} hold at iteration \(t\geq T_2\), then with high probability over \(x \in \D_x\) and \(\DD\in\D_{\DD}\):
    \begin{align*}
        \max_{x \in \{x_p,x_p^+,x_p^{++}\}} \1_{h_{i,t}(x)\neq 0} \leq \sum_{j\in\N_i}\1_{z_{p,j}\neq 0} 
    \end{align*}
    which implies that \(\max_{x \in \{x_p,x_p^+,x_p^{++}\}\cup\Nfr}\Pr(h_{i,t}(x)\neq 0) \leq O(\frac{\log\log d}{d})\).
\end{fact}

Now for the simplicity of calculations, we define the following notations which are used throughout this section:

\begin{definition}[expansion of gradient]\label{def:expand-grad}
    For each \(i\in[m]\), \(j \in [d]\), we expand \(\vbrack{\nabla_{w_i}L(f_t),\M_j}\) as:
    \begin{align*}
        \vbrack{\nabla_{w_i}L(f_t),\M_j} & = \E\left[\Bigg((1 - \ell'_{p,t})h_{i,t}(x_p^{++}) + \sum_{x_{n,s} \in \Nfr}\ell'_{s,t} h_{i,t}(x_{n,s})\Bigg)\1_{|\vbrack{w_i^{(t)},x_{p}^+}|\geq b_i^{(t)}}\vbrack{x_p^+,\M_j}\right]  \\
        &= \Psi_{i,j}^{(t)} + \Phi_{i,j}^{(t)}+ \Ecal_{1,i,j}^{(t)} +\Ecal_{2,i,j}^{(t)}
    \end{align*}
    where the \(\Psi^{(t)}, \Phi^{(t)}, \Ecal_1^{(t)},\Ecal_2^{(t)} \) are defined as follows: for each \(x = \sum_j\M_j z_j + \xi \sim \D_x\) (or augmented instance \(x^+, x^{++}\)), we write
    \begin{displaymath}
        \psi_{i,j}^{(t)} (x) = \left(\vbrack{w_i^{(t)},\M_j}z_j - b_i^{(t)}\right)\1_{\vbrack{w_i^{(t)},x} > b_i^{(t)}} - \left(\vbrack{w_i^{(t)},\M_j}z_j + b_i^{(t)}\right)\1_{\vbrack{w_i^{(t)},x} < -b_i^{(t)}}
    \end{displaymath}
    and 
    \begin{displaymath}
        \phi_{i,j}^{(t)} (x) = \vbrack{w_i^{(t)},x^{\setminus j}} \1_{\vbrack{w_i^{(t)},x} > b_i^{(t)}} - \vbrack{w_i^{(t)},x^{\setminus j}}\1_{\vbrack{w_i^{(t)},x} < -b_i^{(t)}}
    \end{displaymath}
    Now we define 
    \begin{align}\label{eqdef:expand-grad-1}
        \begin{split}
            \Psi^{(t)}_{i,j} &:= \E\left[\Bigg(  (1 - \ell'_{p,t})\cdot\psi_{i,j}^{(t)}(x_{p}^{++}) + \sum_{x_{n,s} \in \Nfr}\ell'_{s,t}\cdot \psi_{i,j}^{(t)}(x_{n,s}) \Bigg)\1_{|\vbrack{w_i^{(t)},x_{p}^+}|\geq b_i^{(t)}}z_{p,j}\right]\\
            \Phi_{i,j}^{(t)} &:= \E\left[\Bigg(  (1 - \ell'_{p,t})\cdot\phi_{i,j}^{(t)}(x_{p}^{++}) + \sum_{x_{n,s} \in \Nfr}\ell'_{s,t}\cdot \phi_{i,j}^{(t)}(x_{n,s}) \Bigg)\1_{|\vbrack{w_i^{(t)},x_{p}^+}|\geq b_i^{(t)}}z_{p,j}\right]\\
            \Ecal_{1,i,j}^{(t)} &:= \E\left[\Bigg(  (1 - \ell'_{p,t})\cdot h_{i,t}(x_{p}^{++}) + \sum_{x_{n,s} \in \Nfr} \ell'_{s,t}\cdot h_{i,t}(x_{n,s}) \Bigg)\1_{|\vbrack{w_i^{(t)},x_{p}^+}|\geq b_i^{(t)}}\vbrack{\M_j,(2\DD-\Id)\M z_{p}}\right]\\
            \Ecal_{2,i,j}^{(t)} &:= \E\left[\Bigg(  (1 - \ell'_{p,t})\cdot h_{i,t}(x_{p}^{++}) + \sum_{x_{n,s} \in \Nfr}\ell'_{s,t}\cdot h_{i,t}(x_{n,s}) \Bigg)\1_{|\vbrack{w_i^{(t)},x_{p}^+}|\geq b_i^{(t)}}\vbrack{\M_j,2\DD\xi_p}\right]
        \end{split}
    \end{align}
    Moreover, for \( j\in [d_1]\setminus [d]\), we can similarly define the following notations:
    \begin{align}\label{eqdef:expand-grad-2}
        \begin{split}
            &\Psi_{i,j}^{(t)},\ \Phi_{i,j}^{(t)}  \equiv 0, \\
            &\Ecal_{1,i,j}^{(t)}:= \E\left[\Bigg(  (1 - \ell'_{p,t})\cdot h_{i,t}(x_{p}^{++}) + \sum_{x_{n,s} \in \Nfr}\ell'_{s,t}\cdot h_{i,t}(x_{n,s}) \Bigg)\1_{|\vbrack{w_i^{(t)},x_{p}^+}|\geq b_i^{(t)}}\vbrack{\M_j,(2\DD-\Id)\M z_{p}}\right]\\
            &\Ecal_{1,i,j}^{(t)} := \E\left[\Bigg(  (1 - \ell'_{p,t})\cdot h_{i,t}(x_{p}^{++}) + \sum_{x_{n,s} \in \Nfr}\ell'_{s,t}\cdot h_{i,t}(x_{n,s}) \Bigg)\1_{|\vbrack{w_i^{(t)},x_{p}^+}|\geq b_i^{(t)}}\vbrack{\Mperp_j,2\DD\xi_p}\right]
        \end{split}
    \end{align}
\end{definition}

Equipped with the above definition, we are ready to characterize the training process at the final stage.

\subsection{Gradient Upper and Lower Bounds for \(\Psi^{(t)}\)}

\begin{lemma}[lower bound for \(\Psi_1^{(t)}\)]\label{lem:signal-lowerbound}
    Suppose \myref{induct-3}{Induction Hypothesis} holds at iteration \(t\). For \(j \in [d]\) and \(i \in \Mcal_j^{\star}\), there exist \(G_1 = \Theta(1)\) such that if \(\F_j^{(t)}:=\sum_{i'\in\Mcal_j}\vbrack{w_{i'}^{(t)},\M_j}^2\leq G_1\tau \log d\), then we have 
    \begin{align*}
        \Psi_{i,j}^{(t)}\cdot\sign(\vbrack{w_i^{(t)},\M_j}) \geq \frac{\E[|z_j|]}{\polylog(d)}\left(1 - O(\frac{1}{\Xi_2^3})\right) |\vbrack{w_i^{(t)},\M_j}|
    \end{align*}
\end{lemma}

\begin{proof}
    We begin with the proof of (a).  We first decompose \(\Psi_{i,j}^{(t)} = \Psi_{i,j,1}^{(t)} + \Psi_{i,j,2}^{(t)}\), where 
    \begin{align*}
        \Psi_{i,j,1}^{(t)} & = \E\left[  (1 - \ell'_{p,t})\cdot\psi_{i,j}^{(t)}(x_{p}^{++}) \1_{|\vbrack{w_i^{(t)},x_{p}^+}|\geq b_i^{(t)}}z_{p,j}\right]\\
        \Psi_{i,j,2}^{(t)} & = \sum_{x_{n,s} \in \Nfr}\E\left[ \ell'_{s,t}  \cdot \psi_{i,j}^{(t)} (x_{n,s}) \1_{|\vbrack{w_i^{(t)},x_{p}^+}|\geq b_i^{(t)}}z_{p,j}\right]
    \end{align*}
    We first deal with \(\Psi_{i,j,2}^{(t)}\). Using the notation \(x_{p}^{\setminus j,+} := 2\DD(\sum_{j'\neq j}\M_{j}z_{p,j'} + \xi_p)\), we can rewrite as
    \begin{align}\label{eqdef:lem-signal-lb-1}
        \Psi_{i,j,2}^{(t)} & = \sum_{x_{n,s} \in \Nfr}\E\left[\ell'_{s,t}(x_p^+, \Bfr) \cdot \psi_{i,j}^{(t)} (x_{n,s})\1_{|\vbrack{w_i^{(t)},x_{p}^+ }|\geq b_i^{(t)}}z_{p,j} \right] \nonumber\\
        & = \sum_{x_{n,s} \in \Nfr}\E\left[\left(\ell'_{s,t}(x_p^+,\Bfr) - \ell'_{s,t}(x_p^{\setminus j,+}, \Bfr) \right)\cdot \psi_{i,j}^{(t)}(x_{n,s})\1_{|\vbrack{w_i^{(t)},x_{p}^+ }|\geq b_i^{(t)}}z_{p,j} \right] \nonumber\\
        & \qquad + \sum_{x_{n,s} \in \Nfr}\E\left[\ell'_{s,t}(x_p^{\setminus j,+}, \Bfr)\cdot \psi_{i,j}^{(t)}(x_{n,s})\1_{|\vbrack{w_i^{(t)},x_{p}^+ }|\geq b_i^{(t)}}z_{p,j} \right]\nonumber\\
        & = R_1+ R_2
    \end{align}
    We now deal with the term \(R_1\). Denoting \(\widehat{x}_{p}^+(v) = \widetilde{x}_{p}^+ + 2v\DD^{(s)}\M_jz_{p,j}\) for \(v \in [0,1]\), by Newton-Leibniz formula and the basic fact that \(\frac{\mathrm{d}}{\mathrm{d}r} \frac{e^r}{e^r+\sum_{s\neq r}e^s} =  \frac{e^r}{e^r+\sum_{s\neq r}e^s}(1 - \frac{e^r}{e^r+\sum_{s\neq r}e^s}) \), we can rewrite \(\ell'_{s,t}(\widehat{x}_{p}^+(v), \Bfr)\) and \(\ell'_{p,t}(\widehat{x}_{p}^+(v), \Bfr)\) as 
    \begin{align*}
        \widehat{\ell'_{s,t}}(\nu) &:= \frac{e^{\vbrack{f_t(x_{p}^{\setminus j, +}) + \nu (f_t(x_{p}^{+}) - f_t(x_{p}^{\setminus j, +})),f_t(x_{n,s})}/\tau}}{\sum_{x \in \Bfr}e^{\vbrack{f_t(x_{p}^{\setminus j, +}) + \nu (f_t(x_p^{+}) - f_t(x_{p}^{\setminus j, +})),f_t(x)}/\tau}}  \equiv  \ell'_{s,t}(\widehat{x}_{p}^+(v), \Bfr)\\
        \widehat{\ell'_{p,t}}(\nu) &:= \frac{e^{\vbrack{f_t(x_{p}^{\setminus j, +}) + \nu (f_t(x_{p}^{+}) - f_t(x_{p}^{\setminus j, +})),f_t(x_{p}^+)}/\tau}}{\sum_{x \in \Bfr}e^{\vbrack{f_t(x_{p}^{\setminus j, +}) + \nu (f_t(x_p^{+}) - f_t(x_{p}^{\setminus j, +})),f_t(x)}/\tau}}  \equiv  \ell'_{p,t}(\widehat{x}_{p}^+(v), \Bfr)
    \end{align*}
    and we can then proceed to calculate as follows:
    \begin{align}\label{eqdef:lem-signal-lb-2}
        R_1 = &\sum_{x_{n,s} \in \Nfr}\E\left[\left(\ell'_{s,t}( x_p^+,\Bfr) - \ell'_{s,t}(x_p^{\setminus j,+}, \Bfr) \right)\cdot \psi_{i,j}^{(t)}(x_{n,s})\1_{|\vbrack{w_i^{(t)},x_{p}^+ }|\geq b_i^{(t)}}z_{p,j} \right] \nonumber\\
        = \ & \sum_{x_{n,s} \in \Nfr}\E\Bigg[ \frac{1}{\tau}\Bigg(\int_{0}^{1}\widehat{\ell'_{s,t}}(\nu)(1 - \widehat{\ell'_{s,t}}(\nu))\vbrack{f_t(x_p^{+}) - f_t(x_p^{\setminus j,+}),f_t(x_{n,s})}\mathrm{d}\nu \tag{By Newton-Lebniz} \nonumber\\
        & \qquad- \sum_{x\in \Nfr\setminus \{x_{n,s}\} }\int_{0}^{1} \widehat{\ell'_{s,t}}(\nu ) \widehat{\ell'_{u,t}}(\nu) \vbrack{f_t(x_p^{+}) - f_t(x_p^{\setminus j,+}),f_t(x_{n,u})}\mathrm{d}\nu \nonumber\\
        & \qquad - \int_{0}^{1} \widehat{\ell'_{s,t}}(\nu ) \widehat{\ell'_{p,t}}(\nu) \vbrack{f_t(x_p^{+}) - f_t(x_p^{\setminus j,+}),f_t(x_{p}^{++})}\mathrm{d}\nu\Bigg)  \psi_{i,j}^{(t)} (x_{n,s})\1_{|\vbrack{w_i^{(t)},x_{p}^+ }|\geq b_i^{(t)}}z_{p,j}  \Bigg] \nonumber\\
        \stackrel{\text{\ding{172}}}\leq \ & \E\Bigg[ \frac{|\Nfr|}{\tau}\Bigg(\int_{0}^{1}\widehat{\ell'_{s,t}}(\nu)(1 - \widehat{\ell'_{s,t}}(\nu))\mathrm{d}\nu + \sum_{x_{n,u}\in \Nfr\setminus \{x_{n,s}\} }\int_{0}^{1} \widehat{\ell'_{s,t}}(\nu ) \widehat{\ell'_{u,t}}(\nu)\mathrm{d}\nu\Bigg)\times \nonumber\\
        &\qquad\times \max_{x_{n,u}\in \Nfr\setminus \{x_{n,s}\} } |\vbrack{f_t(x_p^{+}) - f_t(x_p^{\setminus j,+}),f_t(x_{n,u})}||\psi_{i,j}^{(t)} (x_{n,s})|\1_{|\vbrack{w_i^{(t)},x_{p}^+ }|\geq b_i^{(t)}}|z_{p,j}|  \Bigg] \nonumber\\
        & + \E\Bigg[ \frac{|\Nfr|}{\tau}\int_{0}^{1} \widehat{\ell'_{s,t}}(\nu ) \widehat{\ell'_{p,t}}(\nu)\mathrm{d}\nu  |\vbrack{f_t(x_p^{+}) - f_t(x_p^{\setminus j,+}),f_t(x_{p}^{++})}||\psi_{i,j}^{(t)} (x_{n,s})|\1_{|\vbrack{w_i^{(t)},x_{p}^+ }|\geq b_i^{(t)}}|z_{p,j}|  \Bigg]\nonumber \\
        \stackrel{\text{\ding{173}}}\leq \ & \frac{|\Nfr|}{\tau}\E\Bigg[ \int_{0}^{1}\widehat{\ell'_{s,t}}(\nu)\mathrm{d}\nu \max_{x \in\Bfr}\left(\sum_{i\in\Mcal_j}\vbrack{w_{i'}^{(t)},\M_j} |h_{i,t} (x)|\right) |\psi_{i,j}^{(t)} (x_{n,s})|z_{p,j}^2\Bigg]\nonumber\\
        & + \widetilde{O}\left(\Xi_2\right)\max_{i'\notin\Mcal_j}|\vbrack{w_{i'}^{(t)},\M_j}|\E\Bigg[ \sum_{x_{n,s} \in \Nfr}\frac{1}{\tau}\int_{0}^{1}\widehat{\ell'_{s,t}}(\nu)\mathrm{d}\nu  |\psi_{i,j}^{(t)}(x_{n,s})|z_{p,j}^2\Bigg]\nonumber\\
        & + \widetilde{O}\left(\frac{\Xi_2}{\sqrt{d_1}\tau}\right)\E\Bigg[ \sum_{x_{n,s} \in \Nfr}\frac{1}{\tau}\int_{0}^{1}\widehat{\ell'_{s,t}}(\nu)\mathrm{d}\nu |\psi_{i,j}^{(t)}(x_{n,s})|z_{p,j}^2\Bigg] +\frac{1}{d^{\Omega(\log d)}} \nonumber\\
        = \ & R_{1,1} + R_{1,2} + R_{1,3} + \frac{1}{d^{\Omega(\log d)}} 
    \end{align}
    where for \ding{172} and \ding{173}, we argue as follows:
    \begin{itemize}
        \item for \ding{172}, we used the fact that the expectations over \(s \in \Nfr\) in the summation can be view as independently and uniformly selecting from \(s \in \Nfr\). which allow us to equate \(\sum_{x_{n,s}\in\Nfr} = |\Nfr|\).
        \item for \ding{173}, we use \myref{fact:activation-3}{Fact} to ensure that \(\sum_{i\in[m]}\1_{h_{i,t}(x_p^+)\neq 0} \leq \widetilde{O}\left(\Xi_2\right)\) with high prob. Further noticing that \(\max_{i}\|w_i^{(t)}\|_2\leq O(1)\) and \(|\vbrack{w_{i'}^{(t)},(\Id-2\DD)\M_j}|\leq \widetilde{O}(\frac{1}{\sqrt{d_1}}\|w_{i'}^{(t)}\|_2)\) w.h.p, we have for any \(x \in \Bfr\):
        \begin{displaymath}
            |\vbrack{f_t(x_p^+) - f_t(x_p^{\setminus j, +}), f_t(x)}| \leq \sum_{i'\in\Mcal_j}|\vbrack{w_{i'}^{(t)},\M_j}|\cdot |h_{i',t}(x)| + \widetilde{O}(\Xi_2)\max_{i'\in\Mcal_j}|\vbrack{w_{i'}^{(t)},\M_j}| + \widetilde{O}\left(\frac{1}{\sqrt{d_1}}\right)
        \end{displaymath}
        which gives the desired inequality.
    \end{itemize}
     Now we proceed to deal with \(R_{1,1}\), since \(i \in \Mcal_j^{\star}\), we have automatically \(\1_{h_{i,t}(x_{n,s})\neq 0} = \1_{z_{n,s,j}\neq 0}\) w.h.p., so we can transform \(R_{1,1}\) as 
    \begin{align*}
        R_{1,1} & =\E\Bigg[   \frac{|\Nfr|}{\tau}\int_{0}^{1}\widehat{\ell'_{s,t}}(\nu)\mathrm{d}\nu\max_{x\in\Bfr}\left(\sum_{i'\in\Mcal_j}\vbrack{w_{i'}^{(t)},\M_j} |h_{i',t}(x)|\right) |\psi_{i,j}^{(t)}(x_{n,s})|z_{p,j}^2\Bigg] + \frac{1}{d^{\Omega(\log d)}} \nonumber \\
        & = \E\Bigg[\frac{|\Nfr|}{\tau}  \int_{0}^{1}\widehat{\ell'_{s,t}}(\nu)\mathrm{d}\nu \left(\sum_{i'\in\Mcal_j}\vbrack{w_{i'}^{(t)},\M_j}^2 + \Upsilon_j^{(t)} \right)  |\psi_{i,j}^{(t)}(x_{n,s})|z_{p,j}^2\Bigg] + \frac{\Xi_2}{d^{\Omega(\log d)}} 
    \end{align*}
    where \(\Upsilon_j^{(t)} \) is defined as the 
    \begin{align*}
        \Upsilon_j^{(t)} = \max_{x \in \Bfr}\sum_{i' \in \Mcal_j}|h_{i,t}(x)||\vbrack{w_{i'}^{(t)},\M_j}| - \sum_{i' \in \Mcal_j}|\vbrack{w_{i'}^{(t)},\M_j}|^2
    \end{align*}
    We proceed to give a high probability bound for \(\sum_{i \in \Mcal_j}|h_{i,t}(x_{n,s})||\vbrack{w_i^{(t)},\M_j}|\), which lies in the core of our proof. In order to apply \hyperref[lem:activation-size-2]{Lemma~\ref*{lem:activation-size}} to the pre-activation in \(h_{i,t}(x_{n,s})\), one can first expand as
    \begin{align*}
        & \quad \, \sum_{i \in [m]}|h_{i,t}(x_{n,s})| |\vbrack{w_i^{(t)},\M_j}| \\
        &\lesssim \sum_{j' \in [d]}\sum_{i \in \Mcal_{j'}}|\vbrack{w_i^{(t)},\M_{j'}}||z_{n,s,j'}|\cdot|\vbrack{w_i^{(t)},\M_j}| + \sum_{i \in [m]}\widetilde{O}(\frac{\|w_i^{(t)}\|_2}{\sqrt{d}})|\vbrack{w_i^{(t)},\M_j}| \\
        &= \sum_{j' \in [d]}\sum_{i \in \Mcal_j\cap \Mcal_{j'}}\vbrack{w_i^{(t)},\M_{j'}}|z_{p,j'}|\cdot|\vbrack{w_i^{(t)},\M_j}|   + \sum_{j' \in [d]}\sum_{i \in \Mcal_{j'}\setminus \Mcal_j}|\vbrack{w_i^{(t)},\M_{j'}}| |z_{n,s,j'}|\cdot|\vbrack{w_i^{(t)},\M_j}|\\
        &\quad + \sum_{i \in [m]}\widetilde{O}(\frac{\|w_i^{(t)}\|_2}{\sqrt{d}})|\vbrack{w_i^{(t)},\M_j}|\\
        &\stackrel{\text{\ding{172}}}{\leq}\sum_{i \in \Mcal_j}\vbrack{w_i^{(t)},\M_j}^2 |z_{n,s,j}| + \sum_{j' \neq j}\sum_{i \in \Mcal_j\cap \Mcal_{j'}}|\vbrack{w_i^{(t)},\M_{j'}}| |z_{n,s,j'}|\cdot|\vbrack{w_i^{(t)},\M_j}| \\
        &\quad + \sum_{j' \in [d]}\sum_{i \in \Mcal_{j'}\setminus \Mcal_j}|\vbrack{w_i^{(t)},\M_{j'}}|\cdot |z_{n,s,j'}|\cdot|\vbrack{w_i^{(t)},\M_j}|+ \sum_{i \in [m]}\widetilde{O}(\frac{\|w_i^{(t)}\|_2}{\sqrt{d}})|\vbrack{w_i^{(t)},\M_j}|
    \end{align*}
    And we proceed to calculate the last two terms on the RHS as follows: firstly, from \myref{lem:property-init}{Lemma} we know for the set of neurons \( \Gamma_j := \{j'\neq j, j'\in[d]: \Mcal_j\cap \Mcal_{j'} \neq \varnothing\} \), we have \(|\Gamma_j| \leq O(\log d)\), and 
    \begin{align*}
        \left|\sum_{j' \neq j}\sum_{i' \in \Mcal_j\cap \Mcal_{j'}}|\vbrack{w_{i'}^{(t)},\M_{j'}}||z_{n,s,j'}|\cdot|\vbrack{w_i^{(t)},\M_j}|\right|  & \leq O\left(\frac{\log d}{\sqrt{\Xi_2}}\right)\cdot \sum_{j' \in \Gamma_j}O(\tau\log^2 d) |z_{n,s,j'}|\\
        & \leq O\left(\frac{\tau}{\log d}\right) \tag*{w.h.p.}
    \end{align*}
    where in the last inequality we have taken into account the fact that \(\E|z_j| = \widetilde{O}(\frac{1}{d})\) and have used \myref{lem:activation-size-2}{Lemma}. The same techniques also provide the following bound:
    \begin{align*}
        \left|\sum_{j' \in [d]}\sum_{i' \in \Mcal_{j'}\setminus \Mcal_j}|\vbrack{w_{i'}^{(t)},\M_{j'}}||z_{n,s,j'}|\cdot|\vbrack{w_{i'}^{(t)},\M_j}| \right| & \leq O\left(\frac{1}{\sqrt{d}}\right)\sum_{j'\neq j}O(\Xi_2) |z_{n,s,j'}| \leq O\left(\frac{\Xi_2^2}{\sqrt{d}}\right)\tag*{w.h.p.}
    \end{align*}
    Therefore via a union bound, we have
    \begin{align*}
        \sum_{i' \in [m]}|h_{i',t}(x_{n,s})||\vbrack{w_{i'}^{(t)},\M_j}| \leq \sum_{i'\in \Mcal_j}\vbrack{w_{i'}^{(t)},\M_j}^2|z_{n,s,j}| + O\left(\frac{\tau}{\log d}\right) \tag*{w.h.p.}
    \end{align*}
    The same arguments also gives (+ further applying \myref{lem:activation-size}{Lemma})
    \begin{align*}
        \max_{x \in \{x_p^+,x_p^{++}\}} \left\{\sum_{i' \in [m]}|h_{i',t}(x)||\vbrack{w_{i'}^{(t)},\M_j}|\right\} \leq \sum_{ i'\in \Mcal_j}\vbrack{w_{i'}^{(t)},\M_j}^2|z_{p,j}| + O\left(\frac{\tau}{\log d}\right) \tag*{w.h.p.}
    \end{align*}
    which also implies that
    \begin{align*}
        \Upsilon_j^{(t)} \leq \max_{x \in \Bfr}\sum_{i' \in [m]}|h_{i,t}(x)||\vbrack{w_{i'}^{(t)},\M_j}| - \sum_{i' \in \Mcal_j}|\vbrack{w_{i'}^{(t)},\M_j}|^2 \leq  O(\frac{\tau}{\log d})
    \end{align*}
    Now we are ready to control the quantity \(R_{1,1}\). the idea here is to ``decorrelate'' the factor \(\widehat{\ell'_{s,t}} (\nu)\) from the others. Definining \(\Bfr^{\setminus s} := \{x_p^{++}\}\cup\Nfr\setminus \{x_{n,s}\}\) and \(\Bfr'_{s} := \Bfr^{\setminus s} \cup\{x_{n,s}^{\setminus j}\}\), there exist a constant \(G'_1 > 0\) such that, if \(\sum_{ i\in \Mcal_j}\vbrack{w_i^{(t)},\M_j}^2  \leq \tau G'_1 \log d \), we have w.h.p.
    \begin{align}\label{eqdef:lem-signal-lb-3}
        \frac{\ell'_{s,t}(x_p^+,\Bfr)}{\ell'_{s,t}(x_p^+,\Bfr'_{s})} & = \frac{e^{\vbrack{f_t(x_p^+),f_t(x_{n,s})}/\tau}}{e^{\vbrack{f_t(x_p^+),f_t(x_{n,s}^{\setminus j})}/\tau}}\cdot \frac{\sum_{x\in\Bfr'_s} e^{\vbrack{f_t(x_p^+),f_t(x)}/\tau}}{ \sum_{x\in\Bfr} e^{\vbrack{f_t(x_p^+),f_t(x)}/\tau}} \leq \frac{e^{2\vbrack{f_t(x_p^+),f_t(x_{n,s})}/\tau}}{e^{2\vbrack{f_t(x_p^+),f_t(x_{n,s}^{\setminus j})}/\tau}} \nonumber\\
        &\leq e^{2 G'_1 \log d + O(1/\log d)} \leq o\left(d/\Xi_2^5\right) 
    \end{align}
    Now we define \(\Nfr'_s = \{x_{n,s'} \in \Nfr\setminus\{x_{n,s}\} : z_{n,s',j} = 0\}\cup\{x_{n,s}^{\setminus j}\}\). Note that from concentration inequality of Bernoulli variables we know \(|\Nfr'_s| = \Omega(|\Nfr|)\) w.h.p. Thus we have (notice that the outer factor \(|\Nfr|\) can be insert into the expectation by sacrificing some constant factors):
    \begin{align*}
        R_{1,1} & \leq O\left(\frac{|\Nfr|}{\tau}\right)\cdot\E\left[\widehat{\ell'_{s,t}}(1)\left(\sum_{i\in\Mcal_j}\vbrack{w_{i'}^{(t)},\M_j}^2 + \Upsilon_j^{(t)}\right)  |\psi_{i,j}^{(t)}(x_{n,s}) | \cdot z_{p,j}^2 \right] \\
        & = O\left(\frac{|\Nfr|\log\log d}{d\tau}\right)\E\left[\ell'_{s,t}(x_p^+, \Bfr)\left(\sum_{i\in\Mcal_j}\vbrack{w_{i'}^{(t)},\M_j}^2 + \Upsilon_j^{(t)}\right) \cdot   |\psi_{i,j}^{(t)}(x_{n,s}) |  \,  \Bigg|\, |z_{p,j}|\neq 0 \right] \\
        & = \widetilde{O}\left(\frac{|\Nfr|}{d\tau}\right)\E\left[ \frac{\ell'_{s,t}(x_p^+, \Bfr)}{\ell'_{s,t}(x_p^+, \Bfr'_{s})}\times \ell'_{s,t}(x_p^+, \Bfr'_s) \left(\sum_{i\in\Mcal_j}\vbrack{w_{i'}^{(t)},\M_j}^2 + \Upsilon_j^{(t)}\right)  |\psi_{i,j}^{(t)}(x_{n,s})|  \,  \Bigg|\, |z_{p,j}|\neq 0  \right]\\
        & \leq O\left(\frac{1}{\tau \Xi_2^4}\right)\E\left[\sum_{x_{n,s} \in \Nfr'}\ell'_{s,t}(x_p^+, \Bfr'_s)\left(\sum_{i\in\Mcal_j}\vbrack{w_{i'}^{(t)},\M_j}^2 + \Upsilon_j^{(t)}\right) |\psi_{i,j}^{(t)}(x_{n,s})| \  \Bigg|\, |z_{p,j}|\neq 0  \right]\\
        & \stackrel{\text{\ding{172}}}{\leq} |\vbrack{w_i^{(t)},\M_j} - b_i^{(t)}|\cdot O\left(\frac{1}{\Xi_2^3}\right)\cdot\Pr(z_{n,s,j}\neq 0) \\
        & \leq  |\vbrack{w_i^{(t)},\M_j}|\cdot O(\frac{1}{\Xi_2^3}) \Pr(z_j\neq 0)
    \end{align*}
    where in inequality \ding{172} we have used the independence of \(z_{n,s,j}\) with respect to \(\ell'_{s,t}(x_p^+, \Bfr'_s)\), and the fact that \(\sum_{x_{n,s}\in\Nfr}\ell'_{s,t}(x_p^+,\Bfr'_s)\leq 1\). Now turn back to deal with \(R_{1,2}\) and \(R_{1,3}\) in \eqref{eqdef:lem-signal-lb-2}. Indeed ,noticing that \(\max_{i'\notin\Mcal_j}|\vbrack{w_{i'}^{(t)},\M_j}|\leq O(\frac{1}{\sqrt{d}})\) from \myref{induct-3}{Induction Hypothesis}, and that 
    \begin{displaymath}
        \sum_{x\in \Nfr}\E\left[\frac{1}{\tau}\int_{0,1}\widehat{\ell'_{s,t}}(\nu)\mathrm{d} \nu |\psi_{i,j}^{(t)}(x_{n,s})z_{p,j}^2 |\right] \leq \widetilde{O}(\frac{1}{d}) |\vbrack{w_i^{(t)},\M_j} - b_i^{(t)}|
    \end{displaymath}
    we have \(R_{1,2}, R_{1,3} \leq o(R_{1,1}) \). For \(R_2\) in \eqref{eqdef:lem-signal-lb-2}, we can see from the definition of \(\ell'_{s,t}(x_p^{\setminus j,+}, \Bfr)\) that it is independent to \(z_{p,j}\). Notice further that \(\1_{|\vbrack{w_i^{(t)},x_p^+}|\geq b_i^{(t)}} = \1_{z_{p,j}\neq 0}\) with high probability due to our assumption, and also the fact that \(\1_{z_{p,j}\neq 0}z_{p,j}\) has mean zero and is independent to \(\ell'_{s,t}(x_p^{\setminus j, +}, x_{n,s})\) we have 
    \begin{align*}
        R_2 & \leq \poly(d)\cdot e^{-\Omega(\log^2 d)} \lesssim \frac{1}{\poly(d)^{\Omega(\log d)}}
    \end{align*}
    Combining the pieces above together, we can have
    \begin{displaymath}
        \Psi_{i,j,2}^{(t)} \leq O(\frac{1}{\Xi_2^3}\E[z_j^2])|\vbrack{w_i^{(t)},\M_j} - b_i^{(t)}|
    \end{displaymath} 
    
    Now we turn to \(\Phi_{1,1}^{(t)}(j)\), whose calculation is similar. Defining \(\B_j := \{z_{p,j'} = 0, \forall j' \neq j\}\), we separately discuss the cases when events \(\B_j\) or \(\B_j^c\) holds:
    \begin{itemize}
        \item When \(\B_j^c\) happens, \(\sign(\psi_{i,j}^{(t)}(x_p^{++}) ) = \sign(\vbrack{w_i^{(t)},\M_j}z_j)\) with high prob by \myref{fact:activation-3}{Fact} since we assumed \(i \in \Mcal_j^{\star}\). Thus, if \(\vbrack{w_i^{(t)},\M_j} > 0\), we have 
        \begin{align*}
            \E\left[(1 - \ell'_{p,t})\psi_{i,j}^{(t)} (x_p^{++})\1_{|\vbrack{w_i^{(t)},x_p^+}|\geq b_i^{(t)}}z_{p,j} \Big| \B_j^c \right] \geq 0
        \end{align*}
        if \(\vbrack{w_i^{(t)},\M_j} < 0\), the opposite inequality holds as well.
        \item When \(\B_j\) happens, it is easy to derive that 
        \begin{align*}
            |\vbrack{f_t(x_p^+),f_t(x_p^{++})}| \leq \sum_{i'\in\Mcal_j}\vbrack{w_{i'}^{(t)},\M_j}^2z_{p,j}^2 + O(\Xi_2)\cdot O(\frac{\|w_{i'}^{(t)}\|_2}{\sqrt{d}}) \leq \vbrack{w_i^{(t)},\M_j}^2 + O(\frac{\Xi_2}{\sqrt{d}})
        \end{align*}
        and therefore
        \begin{align*}
            |\vbrack{f_t(x_p^+),f_t(x_p^{++})} - \vbrack{f_t(x_p^{\setminus j,+}),f_t(x_{p}^{\setminus j,++})}|  \leq \sum_{i'\in\Mcal_j}\vbrack{w_{i'}^{(t)},\M_j}^2z_{p,j}^2 + O(\frac{\Xi_2}{\sqrt{d}})
        \end{align*}
        from previous analysis, we also have 
        \begin{align*}
            |\vbrack{f_t(x_p^+),f_t(x_{n,s})} - \vbrack{f_t(x_p^{\setminus j,+}),f_t(x_{n,s})}|\leq \sum_{i'\in\Mcal_j}\vbrack{w_{i'}^{(t)},\M_j}^2z_{p,j}^2 + O(\frac{\tau}{\log d}) + O(\frac{\Xi_2}{\sqrt{d}})
        \end{align*}
        These inequalities allow us to apply the same techniques in bounding \(\Phi_{1,2}^{(t)}\) as follows. We define \(\Bfr'_{p} := \Nfr \cup\{x_{p}^{\setminus j,++}\}\). Then, similar to \eqref{eqdef:lem-signal-lb-3}, for some \(G'_2 = \Theta(1)\), we can have 
        \begin{align*}
            \frac{\ell'_{p,t}(x_p^{\setminus j,+},\Bfr)}{\ell'_{p,t}(x_p^{\setminus j,+},\Bfr'_{p})} & = \frac{e^{\vbrack{f_t(x_p^+),f_t(x_{p}^{++})}/\tau}}{e^{\vbrack{f_t(x_p^{\setminus j,+}),f_t(x_{p}^{\setminus j,++})}/\tau}}\cdot \frac{\sum_{x\in\Bfr'_p} e^{\vbrack{f_t(x_p^{\setminus j,+}),f_t(x)}/\tau}}{ \sum_{x\in\Bfr} e^{\vbrack{f_t(x_p^+),f_t(x)}/\tau}} \leq e^{2G'_2\log d + O(\frac{1}{\log d})} \leq O(\frac{d}{\Xi_2^5})
        \end{align*}
        Now we can proceed to compute as follows:
        \begin{align*}
            &\quad\, \E\left[\ell'_{p,t}(x_p^+,\Bfr)\cdot\psi_{i,j}^{(t)} (x_p^{++})\1_{|\vbrack{w_i^{(t)},x_p^+}|\geq b_i^{(t)}}z_{p,j}\Big|\, \B_j\right] \\
            & = \E\left[\frac{\ell'_{p,t}(x_p^{\setminus j,+},\Bfr)}{\ell'_{p,t}(x_p^{\setminus j,+},\Bfr'_{p})}\times \ell'_{p,t}(x_p^{\setminus j,+},\Bfr'_{p})\cdot \psi_{i,j}^{(t)} (x_p^{++})\1_{|\vbrack{w_i^{(t)},x_p^+}|\geq b_i^{(t)}}z_{p,j}\Big|\, \B_j\right] \\
            &\leq |\vbrack{w_i^{(t)},\M_j} - b_i^{(t)}| O\left(\frac{1}{\Xi_2^5 }\right)\Pr(z_j\neq 0)
        \end{align*}
        But from \myref{lem:positive-gd-sparse-2}{Lemma}, and that \myref{induct-2}{Induction Hypothesis} still holds for Stage III, we have 
        \begin{align*}
            \E\left[\psi_{i,j}^{(t)} (x_p^{++})\1_{|\vbrack{w_i^{(t)},x_p^+}|\geq b_i^{(t)}}z_{p,j}\1_{\B_j}\right] = \sign(\vbrack{w_i^{(t)},\M_j}) (|\vbrack{w_i^{(t)},\M_j}| - b_i^{(t)}) \frac{1}{\polylog(d)}\E[|z_j|]
        \end{align*}
    \end{itemize}
    Combining both cases above gives the bound of \(\Psi_{i,j,2}^{(t)}\). Combining results for \(\Psi_{i,j,1}^{(t)}\) and \(\Psi_{i,j,2}^{(t)}\) concludes the proof. The constant \(G_1\) in the statement can be defined as \(G_1 := \min\{G'_1,G'_2\}\).
\end{proof}

\begin{lemma}[upper bound of \(\Psi_{i,j}^{(t)}\)]\label{lem:signal-upper}
    Let \(j \in [d]\) and \(i \in \Mcal_{j}^{\star}\). Suppose \myref{induct-3}{Induction Hypothesis} hold at iteration \(t\), then there exist a constant \(G_2 = \Theta(1)\), if \(\F^{(t)}_j = \sum_{j: i\in \Mcal_j}\vbrack{w_i^{(t)},\M_j}^2 \geq G_2\tau\log d\), we have 
    \begin{align*}
        \Psi_{i,j}^{(t)} \leq \frac{1}{\poly(d)}|\vbrack{w_i^{(t)},\M_j}| 
    \end{align*}
    Similarly, for \(i \in \Mcal_j\), we have
    \begin{align*}
        \Psi_{i,j}^{(t)} \leq \frac{1}{\poly(d)}|\vbrack{w_i^{(t)},\M_j}|  +  O(\frac{1}{d^2})(b_i^{(t)})
    \end{align*}
\end{lemma}

\begin{proof}
    First we deal with the case of \(i \in \Mcal_{j}^{\star}\), we have \(\1_{|\vbrack{w_i^{(t)},x_{p}}|\geq b_i^{(t)}}z_{p,j} = \1_{z_{p,j}\neq 0}z_{p,j}\) w.h.p. when conditions in \myref{induct-3}{Induction Hypothesis} hold. Now by denoting
    \begin{align*}
        \widetilde{\psi}_{i,j}(z_j) := (\vbrack{w_i^{(t)},\M_j}z_j - b_i^{(t)})\1_{\vbrack{w_i^{(t)},\M_j}z_j > 0} - (\vbrack{w_i^{(t)},\M_j}z_j + b_i^{(t)})\1_{\vbrack{w_i^{(t)},\M_j}z_j < 0}
    \end{align*}
    we can then easily rewrite \(\Psi_{i,j}^{(t)}\) as (by using \myref{fact:activation-3}{Fact})
    \begin{align*}
        \Psi_{i,j}^{(t)} &=\E\left[ \Bigg( (1 - \ell'_{p,t}) \widetilde{\psi}_{i,j}^{(t)}(z_{p,j}) - \sum_{x_{n,s} \in \Nfr}\ell'_{s,t}\cdot \widetilde{\psi}^{(t)}_{i,j} (z_{n,s,j})\Bigg)\1_{z_{p,j}\neq 0} z_{p,j}\right]  + \frac{1}{\poly(d)^{\Omega(\log d)}} \\
        & = \E_{x_p^+,x_p^{++}}[( \widetilde{\psi}_{i,j}^{(t)}(z_{p,j}) - I_1 - I_2)\1_{z_{p,j}\neq 0}z_{p,j}] + \frac{1}{\poly(d)^{\Omega(\log d)}}
    \end{align*}
    where \(I_1\) and \(I_2\) are defined as follows:
    \begin{align*}
        I_1 &:= \E_{\Nfr}\left[ (1 - \ell'_{p,t}) \widetilde{\psi}_{i,j}^{(t)}(z_{p,j}) - \sum_{x_{n,s} \in \Nfr}\ell'_{s,t}\cdot \widetilde{\psi}^{(t)}_{i,j} (z_{n,s,j})\1_{z_{n,s,j} = z_{p,j}} \right] \\
        &\stackrel{\text{\ding{172}}}{=}  \widetilde{\psi}_{i,j}^{(t)}(z_{p,j}) \E_{\Nfr}\left[\frac{|\Nfr|e^{\vbrack{f_t(x_p^+),f_t({x_{n,s}})}/\tau}\1_{z_{n,s,j}=z_{p,j}} + e^{\vbrack{f_t(x_p^+),f_t({x_{p}^{++}})}/\tau}}{e^{\vbrack{f_t(x_p^+),f_t({x_{n,s}})}/\tau} +  \sum_{x \in \Bfr\setminus\{x_{n,s}\}}e^{\vbrack{f_t(x_p^+),f_t(x)}/\tau}}\right]\\
        I_2 &:= \E_{\Nfr}\left[ \sum_{x_{n,s} \in \Nfr}\ell'_{s,t}\cdot \widetilde{\psi}^{(t)}_{i,j} (z_{n,s,j})\1_{z_{n,s,j} \neq z_{p,j}} \right]
    \end{align*}
    where in \ding{172} we used the identification \(z_{p,j}= z_{n,s,j}\). The tricky part here is since all the variables inside the expectation is non-negative we can use Jensen's inequality to move the expectation of \(e^{\vbrack{f_t(x_p^+),f_t({x_{n,u}})}/\tau}\) to the denominator. We let \(V := e^{\vbrack{f_t(x_p^+),f_t({x_{p}^{++}})}/\tau}\) and consider it fix when computing \(I_{1}\) as follows: conditioned on \(z_{p,j}\neq 0\), we have 
    \begin{align*}
        \frac{I_{1}}{\widetilde{\psi}_{i,j}^{(t)}(z_{p,j})}  &= \E_{\Nfr}\left[\frac{|\Nfr|e^{\vbrack{f_t(x_p^+),f_t({x_{n,s}})}/\tau} \1_{z_{n,s,j}=z_{p,j}} + V}{e^{\vbrack{f_t(x_p^+),f_t({x_{n,s}})}/\tau} + V +  \sum_{x\in\Nfr\setminus \{x_{n,s}\}}e^{\vbrack{f_t(x_p^+),f_t({x_{n,u}})}/\tau}}\right]\\
        &\geq\E_{x_{n,s}}\left[\frac{e^{\vbrack{f_t(x_p^+),f_t({x_{n,s}})}/\tau} \1_{z_{n,s,j}=1} + \frac{1}{|\Nfr|} V}{\frac{1}{|\Nfr|}e^{\vbrack{f_t(x_p^+),f_t({x_{n,s}})}/\tau} + \frac{1}{|\Nfr|}V + \frac{|\Nfr|-1}{|\Nfr|}\E_{x_{n}}[e^{\vbrack{f_t(x_p^+),f_t({x_{n}})}/\tau}]}\right] \tag{by Jensen inequality}\\
        & = \E_{x_{n,s}}\left[\frac{e^{\vbrack{f_t(x_p^+),f_t({x_{n,s}})}/\tau}\1_{z_{n,s,j}=z_{p,j}}+ \frac{1}{|\Nfr|} V }{\frac{1}{|\Nfr|}(e^{\vbrack{f_t(x_p^+),f_t({x_{n,s}})}/\tau} +V) + \frac{|\Nfr|-1}{|\Nfr|}\E_{x_{n}}[e^{\vbrack{f_t(x_p^+),f_t({x_{n}})}/\tau}(\1_{z_{n,j}=z_{p,j}}+\1_{z_{n,j}\neq z_{p,j}})]} \right]\\
        & \stackrel{\text{\ding{172}}}{\geq} \E_{x_{n,s}}\left[\frac{X+ \frac{1}{|\Nfr|} V}{\frac{1}{|\Nfr|}(X +V) + \frac{|\Nfr|-1}{|\Nfr|}(1 + \frac{1}{\poly(d)}) \E_{x_n}[X]} \right] \tag{where \(X := e^{\vbrack{f_t(x_p^+),f_t({x_{n}})}/\tau}\1_{z_{n,s,j}=z_{p,j}} \geq 0\)}\\
         & \stackrel{\text{\ding{173}}}{\geq} 1 - O\left(\frac{1}{\poly(d)}\right)
    \end{align*}
    where for the above inequalities, we argue:
    \begin{itemize}
        \item in \ding{172}, we need to go through similar analysis as in the proof of \myref{lem:signal-lowerbound}{Lemma} to obtain that, with high probability over \(x_p^+ \) and \(x_n^{\setminus j}\):
        \begin{align*}
            \vbrack{f_t(x_p^+),f_t(x_{n})} - \vbrack{f_t(x_p^+),f_t(x_{n}^{\setminus j})}/\tau &\geq \frac{1}{\tau}\sum_{i\in\Mcal_j}\vbrack{w_i^{(t)},\M_j}^2 - O\left(\frac{1}{\log d}\right) \\
            &\geq G_2\log d - O\left(\frac{1}{\log d}\right)
        \end{align*}
        for some very large constant \(G_2 = \Theta(1)\), which gives (the \(\frac{1}{\poly(d)}\) here depends on how large \(G_2\) is)
        \begin{align*}
            \E_{x_{n}}[e^{\vbrack{f_t(x_p^+),f_t({x_{n}})}/\tau}\1_{z_{n,s,j}\neq z_{p,j}}] \leq \frac{1}{\poly(d)}\E_{x_{n}}[e^{\vbrack{f_t(x_p^+),f_t({x_{n}})}/\tau}\1_{z_{n,s,j}=z_{p,j}}].
        \end{align*} 
        \item in inequality \ding{173}, we need to argue as follows, where \(\widetilde{\E}[X]\stackrel{\text{abbr.}}{=} \E_{x_{n}}[X]\) is only integrated over the randomness of \(x_{n}\):
        \begin{align*}
            &\quad \, \E_{x_{n,s}}\left[\frac{(X+ \frac{1}{|\Nfr|} V)}{\frac{1}{|\Nfr|}(X +V) + \frac{|\Nfr|-1}{|\Nfr|}(1 + \frac{1}{\poly(d)}) \widetilde{\E}[X]} \right] \\
            & = 1 -\frac{1}{|\Nfr|} \E_{x_{n,s}}\left[\frac{\widetilde{\E}[X]}{\frac{1}{|\Nfr|}(X +V) + \frac{|\Nfr|-1}{|\Nfr|}(1 + \frac{1}{\poly(d)}) \widetilde{\E}[X]} \right]\\
            &\geq 1 - \frac{1}{|\Nfr|}\cdot\frac{\widetilde{\E}[X]}{\frac{|\Nfr|-1}{|\Nfr|}(1 + \frac{1}{\poly(d)}) \widetilde{\E}[X]} \tag{since \(X+V \geq 0\)}\\
            & \geq 1 - \frac{1}{\poly(d)}
        \end{align*}
    \end{itemize}
    The same analysis applies to \(I_2\), which we can bound as 
    \begin{align*}
        |\frac{I_{1}}{\widetilde{\psi}_{i,j}^{(t)}(- z_{p,j})}| \leq \frac{1}{\poly(d)}
    \end{align*}
    Combining both \(I_1\) and \(I_2\), we have 
    \begin{displaymath}
        \Psi_{i,j}^{(t)}\leq \frac{1}{\poly(d)}|\vbrack{w_i^{(t)},\M_j}|
    \end{displaymath}
    In the case of \(i \in \Mcal_j\), we have with prob \(\leq \widehat{O}(\frac{1}{d})\) that \(\1_{\vbrack{w_i^{(t)},x_p^+}\geq b_i^{(t)}}\neq \1_{\vbrack{w_i^{(t)},\M_j}z_{p,j}> 0} \) or \(\1_{\vbrack{w_i^{(t)},x_p^+}\geq b_i^{(t)}}\neq \1_{\vbrack{w_i^{(t)},\M_j}z_{p,j}> 0} \) . When such events happen, we can obtain a bound of \(O(\frac{1}{d})b_i^{(t)}\) over \(\Psi_{i,j}^{(t)}\), which times the prob \(\widetilde{O}(\frac{1}{d})\) leads to our bound. Combining the above observations and the analyses, we can complete the proof.
\end{proof}

\subsection{Gradient Computations II}

In this section, we give finer characterization of \(\Psi_3^{(t)}\) and \(\Psi_4^{(t)}\), which is the contributions of the dense features/noisy correlations to the gradient.

\begin{lemma}[bounds for \(\Ecal_1^{(t)}\)]\label{lem:Err_1}
    At iteration \(t \geq T_2\), let \(j \in [d]\) and \(i \in [m]\), \myref{induct-3}{Induction Hypothesis} holds at \(t\), for each \(j \in [d_1]\), we have 
    \begin{align*}
        |\Ecal_{1,i,j}^{(t)}| \leq O\left(\frac{\Xi_2^{2}\|w_{i}^{(t)}\|_2}{d^{3/2}\sqrt{d_1}}\right)
    \end{align*}
\end{lemma}

\begin{proof}
    Let \(j\in[d]\), since the case of \(j\in[d_1]\setminus [d]\) can be similarly dealt with. We first look at the following \( \Ecal_{1,1,i,j}^{(t)}\) term in \(\Ecal_{1,i,j}^{(t)}\):
    \begin{align*}
        \Ecal_{1,1,i,j}^{(t)} = \E\left[ h_{i,t}(x_p^{++})\1_{|\vbrack{w_i^{(t)},x_p^+}|\geq b_i^{(t)}}\vbrack{\M_j,(2\DD-\Id)\M z_{p}} \right]
    \end{align*}
    It is easy to observe that using the randomness and symmetry of \(2\DD-\Id\) w.r.t. zero, we have 
    \begin{align*}
        |\Ecal_{1,1,i,j}^{(t)}| \leq \E\left[|\vbrack{w_i^{(t)},(2\DD-\Id)x_p}| \1_{|\vbrack{w_i^{(t)},x_p}|\geq b_i^{(t)} + |\vbrack{w_i^{(t)},x_p- x_p^+}|}|\vbrack{\M_j,(2\DD-\Id)\M z_p}| \right]
    \end{align*}
    When \(\{|\vbrack{w_i^{(t)},x_p}|\geq b_i^{(t)} + |\vbrack{w_i^{(t)},x_p- x_p^+}|\}\) happens (which we know from \myref{fact:activation-3}{Fact} has prob \(\leq \widetilde{O}(\frac{1}{d})\)), using \myref{lem:activation-size}{Lemma}, we have 
    \begin{align*}
        |\Ecal_{1,1,i,j}^{(t)}| \leq \widetilde{O}(\frac{1}{\sqrt{d_1}d^{1.5}})\|w_i^{(t)}\|_2
    \end{align*}
    Now we similarly decompose the sum of expectations as follows: let \(\N_i := \{j\in [d]: i\in \Mcal_j\}\), which from \myref{lem:property-init}{Lemma} we know are of cardinality at most \(O(1)\), then
    \begin{align*}
        \Ecal_{1,2,i,j}^{(t)}& = \E\left[\sum_{x_{n,s} \in \Nfr}\ell'_{s,t} (x_p^+,\Bfr)\cdot h_{i,t}(x_{n,s})\1_{|\vbrack{w_i^{(t)},x_{p}^{+}}|\geq b_i^{(t)}}\sum_{j'\in [d]}\vbrack{\M_j,(2\DD - \Id)\M_{j'}}z_{p,j'}\right] \\
        & = \sum_{j'\in \N_i }\E\left[\sum_{x_{n,s} \in \Nfr}\ell'_{s,t} (x_p^+,\Bfr)\cdot h_{i,t}(x_{n,s})\1_{|\vbrack{w_i^{(t)},x_{p}^{+}}|\geq b_i^{(t)}} \vbrack{\M_j,(2\DD - \Id)\M_{j'}}z_{p,j'}\right] \\
        & \quad + \sum_{j'\notin \N_i } \E\left[\sum_{x_{n,s} \in \Nfr}\ell'_{s,t} (x_p^+,\Bfr)\cdot h_{i,t}(x_{n,s})\1_{|\vbrack{w_i^{(t)},x_{p}^{+}}|\geq b_i^{(t)}}\vbrack{\M_j,(2\DD - \Id)\M_{j'}}z_{p,j'}\right] 
    \end{align*}
    Notice that the major difference between the first and second terms are that the occurence of features \(j'\in\N_i\) has nontrivial probability \(\geq \widetilde{\Omega}(\frac{1}{d})\) to affect the indicator \(\1_{|\vbrack{w_i^{(t)},x_p^+}|\geq b_i^{(t)}}\). However, since \(|\vbrack{\M_j,(2\DD-\Id)\M_{j'}}|\) is w.h.p., small due to \myref{lem:corr-aug}{Lemma}, so for the first term, we can use the symmetry of \(2\DD-\Id\) and \(\Id-2\DD\) to compute as follows: denote \(\Bfr' = \{x_p^{++}\}\cup\{x_{n,s}\}_{\Nfr}\), we have
    \begin{align*}
        &\quad \, \sum_{j'\in \N_i } \E\left[\sum_{x_{n,s} \in \Nfr}\ell'_{s,t} (x_p^+,\Bfr)\cdot h_{i,t}(x_{n,s})\1_{|\vbrack{w_i^{(t)},x_{p}^{+}}|\geq b_i^{(t)}}\vbrack{\M_j,(2\DD - \Id)\M_{j'}}z_{p,j'}\right]\\
        & = \sum_{j'\in \N_i } \E\left[\sum_{x_{n,s} \in \Nfr}\left(\ell'_{s,t} (x_p^+,\Bfr)-\ell'_{s,t} (x_p^{++},\Bfr')\right)\cdot h_{i,t}(x_{n,s}) |\vbrack{\M_j,(2\DD - \Id)\M_{j'}}|z_{p,j'}\right] \\
        & \stackrel{\text{\ding{172}}}{=} \sum_{j'\in \N_i }\E\Bigg[ \sum_{x_{n,s} \in \Nfr}\frac{1}{\tau}\Bigg(\int_{0}^{1}\widehat{\ell'_{s,t}}(\nu)(1 - \widehat{\ell'_{s,t}}(\nu))\vbrack{f_t(x_p^{+}) - f_t(x_p^{++}),f_t(x_{n,s})}\mathrm{d}\nu \\
        & \qquad\qquad\qquad\qquad - \sum_{u \neq s, x_{n,u} \in \Nfr}\int_{0}^{1} \widehat{\ell'_{s,t}}(\nu) \widehat{\ell'_{u,t}}(\nu) \vbrack{f_t(x_p^{+}) - f_t(x_p^{++}),f_t(x_{n,u}^+)}\mathrm{d}\nu\Bigg)\times  \\
        &\qquad\qquad\qquad\qquad\qquad\qquad\qquad \times h_{i,t}(x_{n,s}) \1_{\vbrack{w_i^{(t)},x_p^{+}}\geq b_i^{(t)}}|\vbrack{\M_j,(2\DD-\Id)\M_{j'}}|z_{p,j'} \Bigg]  \\
        & \leq \sum_{j'\in \N_i }\sum_{x_{n,s} \in \Nfr}\E\Bigg[ \Bigg(\int_{0}^{1}\widehat{\ell'_{s,t}}(\nu)(1 - \widehat{\ell'_{s,t}}(\nu))\mathrm{d}\nu + \sum_{u \neq s, x_{n,u} \in \Nfr}\int_{0}^{1} \widehat{\ell'_{s,t}}(\nu ) \widehat{\ell'_{u,t}}(\nu)\mathrm{d}\nu\Bigg)\times \\
        &\qquad\qquad \times \max_{x_{n,u} \in \Nfr} |\vbrack{f_t(x_p^{+}) - f_t(x_p^{++}),f_t(x_{n,u}^+)}||h_{i,t}(x_{n,s})|\1_{|\vbrack{w_i^{(t)},x_{p}^+ }|\geq b_i^{(t)}}z_{p,j'}  \Bigg]\\
        &\stackrel{\text{\ding{173}}}{\leq} \widetilde{O}(\Xi_2) \sum_{j'\in \N_i }\E\left[\sum_{x_{n,s} \in \Nfr}\frac{1}{\tau}\int_{0}^{1} \widehat{\ell'_{s,t}}(\nu)\mathrm{d}\nu\cdot \max_{i'\in[m]}|\vbrack{w_{i'}^{(t)},(\Id-2\DD)x_p}||\vbrack{\M_j,(2\DD-\Id)\M_{j'}}| |h_i(x_{n,s})||z_{p,j'}| \right]\\
        &\stackrel{\text{\ding{174}}}{\leq} \widetilde{O}\left(\frac{\Xi_2 \|w_i^{(t)}\|_2\max_{i'\in[m]}\|w_{i'}^{(t)} \|_2}{d^{3/2}\sqrt{d_1}\tau}\right) + \poly(d)e^{-\Omega(\log^2 d)} \\
        & \leq O(\frac{\Xi_2\|w_i^{(t)}\|_2}{d^{1.5} \sqrt{d_1}\tau})
    \end{align*} 
    where in the above calculations:
    \begin{itemize}
        \item In \ding{172} we have defined \(\widehat{\ell}_{s,t}(\nu)\) as (where \(x_p^{\setminus j, +}: = 2\DD(\sum_{j'\neq j}\M_{j'}z_{p,j'}+\xi_p)\)):
        \begin{displaymath}
            \widehat{\ell'_{s,t}}(\nu) := \frac{e^{\vbrack{f_t(x_{p}^{\setminus j, +}) + \nu (f_t(x_{p}^{+}) - f_t(x_{p}^{\setminus j, +})),f_t(x_{n,s})}}}{\sum_{x_{n,u} \in \Nfr}e^{\vbrack{f_t(x_{p}^{\setminus j, +}) + \nu (f_t(x_p^{+}) - f_t(x_{p}^{\setminus j, +})),f_t(x_{n,u})}}};
        \end{displaymath}
        \item In \ding{173} we have used the fact that at \(t \geq T_2\), it holds \(\sum_{i\in[m]}\1_{h_{i,t}(x_{n,s})\neq 0}\leq \widetilde{O}(\Xi)\) with high probability over all negative samples \(\{x_{n,s}\}_{\Bfr}\), which is from applying \myref{lem:activation-size-3}{Lemma} using the conditions as we assumed in \myref{lem:property-init}{Lemma};
        \item In \ding{174} we have used mainly \myref{lem:activation-size-2}{Lemma} to obtain that \(|\vbrack{w_{i'}^{(t)},(\Id-2\DD)x_p}|\leq \widetilde{O}(\|w_{i'}^{(t)}\|_2/\sqrt{d})\) holds with high probability, combining with the fact that \(\Pr(z_{p,j'}\neq 0) = \widetilde{O}(1/d)\), and \(\sum_{x_{n,s} \in \Nfr}\widehat{\ell'_{s,t}}(\nu) \leq 1\) for all \(\nu \in [0,1]\).
    \end{itemize}
    Combining the results of \(\Ecal_{1,1,i,j}^{(t)} \) and \(\Ecal_{1,2,i,j}^{(t)}\), we can conclude the proof.
\end{proof}

We can also obtain the following lemmas bounding the gradient contributed by the spurious noise via the same approach as in the proof of \myref{lem:Err_1}{Lemma} below. We sketch the proof below.

\begin{lemma}[bounds for \(\Ecal_2^{(t)}\)]\label{lem:Err_2}
    Let \(j \in [d]\) and \(i \in [m]\), suppose \myref{induct-3}{Induction Hypothesis} holds at \(t\), for all \( j\in [d]\), we have
    \begin{align*}
        \Ecal_{2,i,j}^{(t)} \leq O\left(\frac{\|w_{i}^{(t)}\|_2\Xi_2^2}{d^2\tau}\right)\cdot\max_{i'\in [m]}\left(|\vbrack{w_{i'}^{(t)},\M_j}| + \frac{\|w_{i'}^{(t)}\|_2}{\sqrt{d_1}} \right)
    \end{align*}
    The same bound holds for \(j \in [d_1]\setminus [d]\), with \(|\vbrack{w_{i'}^{(t)},\M_j}|\) changing to \(|\vbrack{w_{i'}^{(t)},\Mperp_j}|\).
\end{lemma}

\begin{proof}
    The proof is extremely similar to those in \myref{lem:Err_1}{Lemma}, which we will omit here, the only differences are: (1) For the first quantity, we do not have a mask applied to \(\vbrack{v ,\xi_p}\); (2) the variable \(\vbrack{\M_j,\xi_p}\) cannot affect the firing probability (prob of being nonzero) of \(\1_{|\vbrack{w_i^{(t)},x_p^+}|\geq b_i^{(t)}}\) w.h.p due to \myref{induct-3}{Induction Hypothesis}; (3) one could use a different basis as in the proof of \myref{lem:positive-gd-noise-2}{Lemma} to obtain the desired \(\widetilde{O}(1/\sqrt{d_1})\) factor in the second bound.
\end{proof}

\subsection{Learning Process at the Final Stage}

Before proving \myref{thm:convergence}{Theorem}, we need prove \myref{induct-3}{Induction Hypothesis}, which characterized the trajectory of gradients at iterations \(t\geq T_2\). We first prove a lemma, which allow us to obtain the full characterization of \(\Phi^{(t)}\) term (defined in \myref{def:expand-grad}{Definition}) in gradient calculations.

\begin{lemma}[reduction of \(\Phi^{(t)}\) to the bounds of \(\Psi^{(t)}\)]\label{lem:reduction-Phi-to-Psi}
    Let \(j \in [d]\) and \(i \in \Mcal_{j}\). Suppose \myref{induct-3}{Induction Hypothesis} hold for all iteration before \(t \in [\frac{d^{1.01}}{\eta},\frac{d^{1.99}}{\eta}]\) and after \(T_2\), and also we suppose for all \(l \in [d]\), \(\F_{l}^{(t')} = \Omega(\tau \log d)\) at some \(t' = \Theta(T_2)\), then
    \begin{itemize}
        \item for iteration \(t \in [\frac{d^{1.01}}{\eta},\frac{d^{1.495}}{\eta}]\):
        \begin{align*}
            \Phi_{i,j}^{(t)} \leq \widetilde{O}(\frac{\Xi_2^2}{d^{3/2}})\|w_i^{(t)}\|_2
        \end{align*}
        \item for iteration \(t \in [\frac{d^{1.495}}{\eta},\frac{d^{1.99}}{\eta}]\):
        \begin{align*}
            \Phi_{i,j}^{(t)} \leq \widetilde{O}(\frac{1}{d^{1.98}})\|w_i^{(t)}\|_2
        \end{align*}
    \end{itemize}
\end{lemma}

\begin{proof}
    The proof essentially relies on the condition that \myref{induct-3}{Induction Hypothesis} holds for all \(t' \in [T_3, t]\). We first consider the case where \(i \in \Mcal_j^{\star}\). 
    Similar to how \(\psi_{i,j}^{(t)}(x)\) are defined for each \(x\) in \myref{def:expand-grad}{Definition}, for each \(j' \neq j\), we let
    \begin{align*}
        \rho_{i,j}^{(t)} (x) := (\vbrack{w_i^{(t)},x} - \vbrack{w_i^{(t)},\M z})\1_{z_j\neq 0}
    \end{align*}
    Now it is straightforward to decompose \(\Phi_{i,j}^{(t)}\) as follows:
    \begin{align*}
        \Phi_{i,j}^{(t)} &= \E\left[\Bigg(  (1 - \ell'_{p,t})\cdot\phi_{i,j}^{(t)}(x_{p}^{++}) + \sum_{x_{n,s} \in \Nfr}\ell'_{s,t}\cdot \phi_{i,j}^{(t)}(x_{n,s}) \Bigg)\1_{|\vbrack{w_i^{(t)},x_{p}^+}|\geq b_i^{(t)}}z_{p,j}\right] \\
        & = \sum_{j' \in [d], j' \neq j}\vbrack{w_i^{(t)},\M_{j'}}\E\left[\Bigg(  (1 - \ell'_{p,t})\cdot z_{p,j'} + \sum_{x_{n,s} \in \Nfr}\ell'_{s,t}\cdot z_{n,s,j'}\1_{z_{n,s,j}\neq 0} \Bigg)\1_{|\vbrack{w_i^{(t)},x_{p}^+}|\geq b_i^{(t)}}z_{p,j}\right] \tag{By \myref{fact:activation-3}{Fact}} \\
        & \quad + \E\left[\Bigg(  (1 - \ell'_{p,t})\cdot\rho_{i,j}^{(t)}(x_{p}^{++}) + \sum_{x_{n,s} \in \Nfr}\ell'_{s,t}\cdot \rho_{i,j'}^{(t)}(x_{n,s}) \Bigg)\1_{|\vbrack{w_i^{(t)},x_{p}^+}|\geq b_i^{(t)}}z_{p,j}\right] + \frac{1}{\poly(d)^{\Omega(\log d)}} \\
        & = H_1 + H_2 + \frac{1}{\poly(d)^{\Omega(\log d)}}
    \end{align*}
    Indeed, from similar arguments as in the proof of \myref{lem:Err_1}{Lemma} and \myref{lem:Err_2}{Lemma}, we can trivially obtain \(|H_2| \leq O(\frac{\Xi_2^2}{d^2})\|w_i^{(t)}\|_2\). Now we turn to \(H_1\). Since \(\max_{j'\neq j}|\vbrack{w_i^{(t)},\M_{j'}}| \leq O(\frac{\|w_i^{(t)}\|_2}{\sqrt{d}\Xi_2^5})\), we can simply get (Since w.h.p., \(|\{j'\in[d]:z_{p,j'}\neq 0\}| = \widetilde{O}(1)\), and if \(z_{p,j'}=0\), the negative terms are small from similar analysis in \myref{lem:signal-lowerbound}{Lemma})
    \begin{align*}
        |H_1| &\leq O(\frac{\|w_i^{(t)}\|_2}{\sqrt{d}\Xi_2^5})\sum_{j'\neq j, j'\in [d]}\E\left[\Bigg(  (1 - \ell'_{p,t})\cdot  z_{p,j'} + \sum_{x_{n,s} \in \Nfr}\ell'_{s,t}\cdot z_{n,s,j'}\1_{z_{n,s,j}\neq 0} \Bigg)\1_{|\vbrack{w_i^{(t)},x_{p}^+}|\geq b_i^{(t)}}z_{p,j}\right] \\
        & \leq O(\frac{\|w_i^{(t)}\|_2}{d^{3/2}})
    \end{align*}
    Then we can obtain a crude bound for all \(t \in [\frac{d^{1.01}}{\eta},\frac{d^{1.99}}{\eta}]\) by 
    \begin{align*}
        \Phi^{(t)}_{i,j} \leq (H_1 + H_2) + \frac{1}{\poly(d)^{\Omega(\log d)}} \leq \widetilde{O}(\frac{\|w_i^{(t)}\|_2}{d^{3/2}})
    \end{align*}
    The harder part is to deal with iterations \(t \in [ \frac{d^{1.495}}{\eta}, \frac{d^{1.498}}{\eta}]\). We first establish a connection between \(\Psi^{(t)}\) and \(\Phi^{(t)}\). We first assume that for all \(j'\neq j, j'\in [d]\), it holds that \( |\Psi_{i',j'}^{(t_1)}| / |\vbrack{w_{i'}^{(t_1)},\M_{j'}}| \leq \Omega(\frac{\Xi_2^2}{\sqrt{d} t\eta})\), which is true for all iteration \(t \leq \frac{d\polylog(d)}{\eta}\) from simple calculations. Now suppose at some \(t_1 \geq \frac{d\polylog(d)}{\eta}\), there exist some \(j'\neq j, j'\in [d]\) and \(i' \in \Mcal_j^{\star}\) such that
    \begin{align*}
        |\Psi_{i',j'}^{(t_1)}| / |\vbrack{w_{i'}^{(t_1)},\M_{j'}}| \geq \Omega(\frac{\Xi_2}{\sqrt{d} t\eta})
    \end{align*}
    which means we have the followings:
    \begin{align*}
        \E\left[ \Big((1 - \ell'_{p,t_1})z_{p,j'} + \sum_{x_{n,s} \in \Nfr}\ell'_{s,t_1}z_{n,s,j'} \Big) z_{p,j'}   \right] \geq \Omega(\frac{\Xi_2}{\sqrt{d} t\eta})
    \end{align*}
    Letting \(\Delta > 0\) be defined as the number such that if \(\F_j^{(t)} = \Delta\), we can have \(|\Psi_{i',j'}^{(t)}| / |\vbrack{w_{i'}^{(t)},\M_{j'}}| \leq O(\frac{\sqrt{\Xi_2}\tau\log d}{\sqrt{d} t\eta})\). Then from the calculations in the proof of \myref{lem:signal-lowerbound}{Lemma}, there must be a constant \(\delta > \Omega(1)\) such that \(\F_j^{(t_1)} \geq \Delta - \delta \tau\log d\). However, such growth cannot continue since for some \(t' = \Theta(t/\sqrt{\Xi_2})\), we have for each \(i' \in \Mcal_{j'}\):
    \begin{align*}
        |\vbrack{w_{i'}^{(t_1 + t')},\M_{j'}}| &\geq |\vbrack{w_{i'}^{(t_1 + t')},\M_{j'}}| (1 - \eta\lambda) + \Psi^{(t_1+ t'-1)}_{i',j'} + \Phi^{(t)}_{i',j'} + O(\frac{\Xi_2^2}{d^2}) \\
        &\geq |\vbrack{w_{i'}^{(t )},\M_{j'}}| (1 - \eta\lambda)^{t'} + \sum_{s = t}^{t+t'-1}\Psi^{(s)}_{i',j'} - O(\frac{t' \Xi_2^2}{d^{3/2}})
    \end{align*}
    where the bounds for \(\Phi^{(s)}_{i',j'}\) for each \(s \in [t_1,t_1+t']\) are obtained from induction over iterations \(s' \in [\frac{d^{1.01}}{\eta},s]\). Therefore there must exist \(t''' \in [t,t+t']\) such that \(|\Psi_{i',j'}^{(t)}| \leq O(\frac{\sqrt{\Xi_2}}{\tau\sqrt{d} t\eta})\) or otherwise \(\F_j^{(t+t')} \geq \F_j^{(t)} + t'\cdot O(\frac{\sqrt{\Xi_2}\tau\log d}{\sqrt{d} t\eta}) \geq \Delta + \delta \tau \log d\), which results in that \(|\Psi_{i',j'}^{(t)}| \leq O(\frac{\tau\log d}{\sqrt{d} t\eta})\|w_i^{(t)}\|_2 \), following the same reasoning in \myref{lem:signal-upper}{Lemma}. Above arguments actually proved that \(|\Psi_{i',j'}^{(t)}| \leq \Omega(\frac{\Xi_2}{\sqrt{d} t\eta})\|w_i^{(t)}\|_2 \) at all \(t \in [\frac{d\polylog(d)}{\eta}, \frac{d^{1.498}}{\eta}]\). Therefore we can use the results of all \(\Psi_{i',j'}^{(t)} \), where \(j'\neq j, j'\in[d]\) to get (combined with \myref{fact:activation-3}{Fact})
    \begin{align*}
        |H_1| \leq \widetilde{O}(\max_{j'\neq j, j'\in[d]}\Psi_{i,j}^{(t)} ) \leq \widetilde{O}(\frac{\Xi_2^2}{d^2})\|w_i^{(t)}\|_2)
    \end{align*}
    For iterations \(t \geq \frac{d^{1.498}}{\eta}\), the proof is essentially the same: we only need to notice that the difference \(\Psi_{i,j}^{(t)} - \lambda \vbrack{w_i^{(t)},\M_j}\) here will bounce around zero, while the compensation terms in \(H_1\) are bounded by \(\widetilde{O}(\frac{\|w_i^{(t)}\|_2}{d^{1.98}})\). These observations indeed prove the case \(i \in \Mcal_j^{\star}\). When \(i \in \Mcal_j \setminus \Mcal_j^{\star}\), notice that with prob \(\leq \widetilde{O}(\frac{1}{d})\) it holds \(\1_{|\vbrack{|w_i^{(t)},x}|\geq b_i^{(t)}} = \1_{z_j\neq 0}\) for any \(x \in \Bfr\). Now we expand 
    \begin{align*}
        H_1 &= \sum_{j' \in \N_i , j' \neq j}\vbrack{w_i^{(t)},\M_{j'}}\E\left[\Bigg(  (1 - \ell'_{p,t})\cdot z_{p,j'} + \sum_{x_{n,s} \in \Nfr}\ell'_{s,t}\cdot z_{n,s,j'}\1_{z_{n,s,j}\neq 0} \Bigg)\1_{|\vbrack{w_i^{(t)},x_{p}^+}|\geq b_i^{(t)}}z_{p,j}\right]  \\
        & \quad + \sum_{j' \notin \N_i , j' \neq j}\vbrack{w_i^{(t)},\M_{j'}}\E\left[\Bigg(  (1 - \ell'_{p,t})\cdot z_{p,j'} + \sum_{x_{n,s} \in \Nfr}\ell'_{s,t}\cdot z_{n,s,j'}\1_{z_{n,s,j}\neq 0} \Bigg)\1_{|\vbrack{w_i^{(t)},x_{p}^+}|\geq b_i^{(t)}}z_{p,j}\right] 
    \end{align*}
    Indeed, the event that there are some \(j' \in \N_i\) (which means \(i \in \Mcal_j\)) such that \(z_{p,j'}\neq 0\) has probability \(\leq \widetilde{O}(\frac{1}{d})\), Thus the first term on the RHS is trivially bounded by \(\widetilde{O}(\frac{1}{d^2})\|w_i^{(t)}\|_2\). For the second term of \(H_1\), we can again go through similar procedure as above to obtain that 
    \begin{align*}
        \E\left[\Bigg(  (1 - \ell'_{p,t}) z_{p,j'} + \sum_{x_{n,s} \in \Nfr}\ell'_{s,t} z_{n,s,j'}\1_{z_{n,s,j}\neq 0} \Bigg)\1_{|\vbrack{w_i^{(t)},x_{p}^+}|\geq b_i^{(t)}}z_{p,j}\right] \leq \max\{\widetilde{O}(\frac{\sqrt{\Xi_2}}{\sqrt{d}t\eta}),\frac{1}{d^{1.99}}\}\|w_i^{(t)}\|_2.
    \end{align*}
    Then again we have 
    \begin{align*}
        |H_1| \leq \widetilde{O}(\max_{j'\neq j, j'\in[d]}\Psi_{i,j}^{(t)} ) \leq \widetilde{O}(\max\{ \frac{\Xi_2^2}{d^2}), \frac{\Xi_2}{\sqrt{d}t\eta}\})\|w_i^{(t)}\|_2,
    \end{align*}
    which can be combine with the bound for \(H_2\) to conclude the proof.
\end{proof}

\begin{proof}[Proof of \myref{induct-3}{Induction Hypothesis}]
    First we need to prove all the induction hypothesis hold for \(t = T_2\). Indeed, (1), (4), (5), (6), (7) is valid at \(T_2\) from \myref{lem:activation-size-3}{Lemma} and \myref{thm:2nd-stage}{Theorem}; (2) and (3) holds at \(T_2\) obviously. Now suppose it hold for some \(t \geq T_2\), we will prove that it still hold for \(t+1\). We first deal with the case where \(j \in [d]\) and \(i \notin \Mcal_j\), where it holds that 
    \begin{align*}
        \vbrack{w_i^{(t+1)},\M_j} &= \vbrack{w_i^{(t)},\M_j}(1 - \eta\lambda) + \eta \E[ h_{i,t}(x_p^{++})\1_{\vbrack{w_i^{(t)},x_p^{+}}\geq b_i^{(t)}}\vbrack{x_p^{+},\M_j}] \\
        &\quad - \E\left[\sum_{x_{n,s} \in \Nfr}\ell'_{s,t}\cdot h_{i,t}(x_{n,s})\1_{\vbrack{w_i^{(t)},x_p^+}\geq b_i^{(t)} }\vbrack{x_p^{+},\M_j}\right] \pm \frac{\eta}{\poly(d_1)}
    \end{align*}
    In this case, to calculate the expectation, we need to use \myref{lem:positive-gd-sparse-2}{Lemma}, \myref{lem:Err_1}{Lemma} and \myref{lem:Err_2}{Lemma}. First we compute the probability of events \(A_1--A_4\) by using \myref{lem:activation-size}{Lemma}, \myref{lem:activation-size-2}{Lemma}, \myref{lem:activation-size-3}{Lemma} and our induction hypothesis to obtain
    \begin{displaymath}
        \Pr(A_1), \Pr(A_2)\leq \frac{1}{\poly(d)^{\Omega(\log d)}}
    \end{displaymath}
    which implies
    \begin{align*}
        L_1, L_2 \leq \frac{1}{\poly(d)^{\Omega(\log d)}}
    \end{align*}
    Furthermore, from \myref{fact:activation-3}{Fact}, we also have
    \begin{align*}
        \E[z_j^2\1_{|\vbrack{w_i^{(t)},x_p^+}|\geq b_i^{(t)} + |\vbrack{w_i^{(t)},x_p^+ - x_p}|}] \leq \frac{1}{\poly(d)^{\Omega(\log d)}}
    \end{align*}
    Now we further take into considerations \myref{lem:positive-gd-noise-2}{Lemma}, \myref{lem:Err_1}{Lemma} and \myref{lem:Err_2}{Lemma}. We can obtain 
    \begin{align*}
        |\vbrack{w_i^{(t+1)},\M_j}| &\leq \vbrack{w_i^{(t)},\M_j}(1 - \eta\lambda) + \widetilde{O}(\frac{\Xi_2^2\|w_i^{(t)}\|_2}{d^2})  \pm \frac{\eta}{\poly(d_1)}
    \end{align*}
    Indeed, since we have chosen learning rate \(\eta = \frac{1}{\poly(d)}\) and \(\lambda \in [\frac{1}{d^{1.01}},\frac{1}{d^{1.49}}]\), it is easy to prove (5) as follows:
    \begin{itemize}
        \item For \(i \notin \Mcal_j\), \(|\vbrack{w_i^{(t)},\M_j}|\leq O(\frac{\|w_i^{(t)}\|_2}{\sqrt{d}\Xi_2^5})\): This is easy since by using \myref{lem:positive-gd-sparse-2}{Lemma}, \myref{lem:Err_1}{Lemma} and \myref{lem:Err_2}{Lemma}, we can prove the following inequality by contradiction\footnote{Indeed, if the \(|\vbrack{w_i^{(t)},\M_j}| \geq \Omega(\cdot\frac{\|w_i^{(t)}\|_2}{\sqrt{d}\Xi_2^5})\) at some iteration \(t\), then by our choice of \(\lambda\) and the calculation of \(\Psi_{i,j}^{(t)}\) using \myref{lem:positive-gd-sparse-2}{Lemma}, the gradient sign of \(\vbrack{w_i^{(t)},\M_j} \) will be opposite to itself.}
        \begin{align*}
            |\vbrack{w_i^{(t)},\M_j}| \leq |\vbrack{w_i^{(t-1)},\M_j}|(1 + \frac{\eta}{d^2} - \eta\lambda) +  \widetilde{O}(\frac{\eta \Xi_2^2}{d^2})\|w_i^{(t)}\|_2 \leq \cdots \leq O(\frac{\|w_i^{(t)}\|_2}{\sqrt{d}\Xi_2^5})
        \end{align*} 
    \end{itemize}
    Now we begin to prove (6). For all \(i \in [m]\), we have \(\max_{j\in[d_1]\setminus [d]}|\vbrack{w_i^{(t)},\Mperp_j}|\leq O(\frac{\|w_i^{(t)}\|_2}{\sqrt{d_1}\Xi_2^5})\) at iteration \(t=T_2\); Now, by expanding the gradient updates of \(\vbrack{w_i^{(t)},\Mperp_j}\), we can see that 
    \begin{align*}
        |\vbrack{w_i^{(t+1)},\Mperp_j}| &\leq |\vbrack{w_i^{(t)},\Mperp_j}|(1 - \eta \lambda) + |\Psi_{i,j}^{(t)}| + |\Phi_{i,j}^{(t)}| + |\Ecal_{1,i,j}^{(t)}| + |\Ecal_{2,i,j}^{(t)}|\\
        & \leq |\vbrack{w_i^{(t)},\Mperp_j}|(1 - \eta \lambda) + \widetilde{O}(\frac{\Xi_2^5}{\sqrt{d_1}d^{1.5}})\|w_i^{(t)}\|_2 + \widetilde{O}(\frac{\Xi_2^5}{\tau \sqrt{d_1}d^{2}})\|w_i^{(t)}\|_2
    \end{align*}
    where the last inequality are obtained as follows: first from \myref{lem:Err_1}{Lemma} we simply have \(|\Ecal_{1,i,j}^{(t)}| \leq \widetilde{O}(\frac{\Xi_2^5}{\sqrt{d_1}d^{1.5}})\|w_i^{(t)}\|_2 \), then from \myref{lem:Err_2}{Lemma} we have 
    \begin{align*}
        |\Ecal_{2,i,j}^{(t)}| &\leq O\left(\frac{\|w_{i}^{(t)}\|_2\Xi_2^2}{d^2\tau}\right)\cdot\max_{i'\in [m]}\left(|\vbrack{w_{i'}^{(t)},\Mperp_j}| + \frac{\|w_{i'}^{(t)}\|_2}{\sqrt{d_1}} \right) \\
        &\leq \widetilde{O}(\frac{\Xi_2^5}{\tau \sqrt{d_1}d^{2}})\|w_i^{(t)}\|_2 \tag{since \(\max_{i'\in [m]}|\vbrack{w_{i'}^{(t)},\Mperp_j}| \leq \O(\frac{1}{\sqrt{d_1}\Xi_2^5})\) from induction}
    \end{align*}
    After (5) and (6) are proven, it is easy to observe (1) is true at \(t\). Below we shall prove (2), (3) and (4), after which (7) can be also trivially proven. Indeed, (2) is a corollary of (3) and (4), since if \(\F_j^{(t)} \leq O(\tau\log d)\) and (4) holds, we simply have 
    \begin{align*}
        \|w_i^{(t)}\|_2^2 &= \sum_{j \in \N_i}\vbrack{w_i^{(t)},\M_j}^2 + \sum_{j \notin \N_i, j \in [d]}\vbrack{w_i^{(t)},\M_j}^2 + \sum_{j \in [d_1]\setminus [d]}\vbrack{w_i^{(t)},\Mperp_j}^2 \\
        & \leq \sum_{j \in \N_i}\vbrack{w_i^{(t)},\M_j}^2 + O(d)\cdot O(\frac{\|w_i^{(t)}\|_2^2}{d\Xi_2^{10}}) + O(d_1)\cdot O(\frac{\|w_i^{(t)}\|_2^2}{d_1\Xi_2^{10}})\\
        & \leq \sum_{j \in \N_i}\vbrack{w_i^{(t)},\M_j}^2 + o(\frac{1}{\Xi_2^{10}}\|w_i^{(t)}\|_2^2)
    \end{align*}
    which implies (2). Thus we only need to prove (3) and (4). Indeed, for (3), letting \(i \in \Mcal_j\), we proceed as follows: we first write the updates of \(\vbrack{w_i^{(t)},\M_j}\) as 
    \begin{align*}
        \vbrack{w_i^{(t+1)},\M_j} &= \vbrack{w_i^{(t)},\M_j}(1 - \eta \lambda) + \Psi_{i,j}^{(t)} + \Phi_{i,j}^{(t)} + \Ecal_{1,i,j}^{(t)} + \Ecal_{2,i,j}^{(t)}\\
        & =\vbrack{w_i^{(t)},\M_j}(1 - \eta \lambda) + \Psi_{i,j}^{(t)} + \widetilde{O}(\frac{\Xi_2^2}{d^2})\|w_i^{(t)}\|_2
    \end{align*}
    where the last inequality comes from again from \myref{lem:Err_1}{Lemma} and \myref{lem:Err_2}{Lemma}. Now suppose for some \(t\) we have \(\F_j^{(t)}\geq \Omega(\tau \log^2 d)\), by \myref{lem:signal-upper}{Lemma}, we have 
    \begin{align*}
        \vbrack{w_i^{(t+1)},\M_j} = \vbrack{w_i^{(t)},\M_j}(1 +  \frac{1}{\poly(d)} - \eta \lambda) + \widetilde{O}(\frac{\Xi_2^2}{d^2})\|w_i^{(t)}\|_2 \leq \vbrack{w_i^{(t)},\M_j}(1 +  \frac{1}{\poly(d)} - \eta \lambda/2)
    \end{align*}
    which means that \(\vbrack{w_i^{(t+1)},\M_j} \leq \vbrack{w_i^{(t)},\M_j}\). This in fact gives \(\F_j^{(t+1)}\leq \F_j^{(t)}\), so that (3) is proven. 
    
    Now for (4), we need to induct as follows: for \(t \leq T_{j}':= \frac{d\log d}{\eta\log\log d}\) which is the specific iteration when \(\F_j^{(t)} \geq G_1\tau\log d\), where \(G_1\) is defined in \myref{lem:signal-lowerbound}{Lemma}. The induction of (4) follows from similar proof in \myref{thm:2nd-stage}{Theorem}. After \(T_{j}'\), we discuss as follows
    \begin{itemize}
        \item When \(t \in [T_{j}', \frac{d^{1.49}}{\eta}]\), from above calculations, for each \(i' \in \Mcal_j\), we have 
        \begin{align*}
            \frac{|\vbrack{w_i^{(t+1)},\M_j}|}{|\vbrack{w_{i'}^{(t+1)},\M_j}|} = \frac{|\vbrack{w_i^{(t)},\M_j}|(1 - \eta \lambda) + \eta\Psi_{i,j}^{(t)} \pm O(\frac{\sqrt{\Xi_2}}{t\sqrt{d}})\|w_i^{(t)}\|_2 }{|\vbrack{w_{i'}^{(t)},\M_j}|(1 - \eta \lambda) + \eta \Psi_{i',j}^{(t)} \pm O(\frac{\sqrt{\Xi_2}}{t\sqrt{d}})\|w_{i'}^{(t)}\|_2}
        \end{align*}
        On one hand, for those \(i' \in \Mcal_j\) such that \(|\vbrack{w_{i'}^{(t)},\M_j}| \leq b_i^{(t)}\Xi_2^2\leq O(\frac{\Xi_2^2}{\sqrt{d}}\|w_i^{(t)}\|_2)\), we can safely get \(|\vbrack{w_i^{(t+1)},\M_j}| \gg |\vbrack{w_{i'}^{(t+1)},\M_j}|\). On the other hand, if \(|\vbrack{w_{i'}^{(t)},\M_j}| \geq b_i^{(t)}\Xi_2^2\), then we have 
        \begin{align*}
            \left|\frac{\Psi_{i,j}^{(t)}}{\vbrack{w_{i}^{(t)},\M_j}|} - \frac{\Psi_{i',j}^{(t)}}{\vbrack{w_{i'}^{(t)},\M_j}} \right| \leq O(\frac{\Xi_2}{t\sqrt{d}\eta}b_i^{(t)})
        \end{align*}
        Thus by letting \(\widetilde{\Psi}_j :=  \frac{\Psi_{i,j}^{(t)}}{\vbrack{w_{i}^{(t)},\M_j}}\), then 
        \begin{align*}
            \frac{|\vbrack{w_i^{(t+1)},\M_j}|}{|\vbrack{w_{i'}^{(t+1)},\M_j}|} = \frac{|\vbrack{w_i^{(t)},\M_j}|(1 + \eta \widetilde{\Psi}_j^{(t)} - \eta \lambda) \pm O(\frac{\sqrt{\Xi_2}}{t\sqrt{d}}) \|w_{i}^{(t)}\|_2}{|\vbrack{w_{i'}^{(t)},\M_j}|(1 + \eta \widetilde{\Psi}_j^{(t)} - \eta \lambda) \pm O(\frac{\sqrt{\Xi_2}}{t\sqrt{d}})\|w_{i'}^{(t)}\|_2}
        \end{align*}
        Since at iteration \(t \in [T_{j}', \frac{d^{1.49}}{\eta}]\), it is easy to obtain that \(|\widetilde{\Psi}_j^{(t)} - \lambda| \leq O(\frac{\Xi_2}{\eta t})\).\footnote{The techniques for proving this is extremely similar to the upper bound \(\widetilde{\Psi}_j^{(t)} \leq O(\frac{\Xi_2}{t\eta\sqrt{d}})\) in the proof of \myref{lem:reduction-Phi-to-Psi}{Lemma}. Indeed, one can assume at some iteration \(|\widetilde{\Psi}_j^{(t)} - \lambda| \geq \Omega(\frac{\Xi_2}{\eta t})\), and then proceed to find our that after some iterasions \(\tilde{t} = \Theta(\frac{t}{\sqrt{\Xi_2}})\), \(|\widetilde{\Psi}_j^{(t)} - \lambda|\) will decrease to \(|\widetilde{\Psi}_j^{(t)} -\lambda|\leq \frac{\sqrt{\Xi_2}}{\eta t}\), or otherwise the presumption collapse.} Thus we have 
        \begin{align*}
            \frac{|\vbrack{w_i^{(t+1)},\M_j}|}{|\vbrack{w_{i'}^{(t+1)},\M_j}|} &\geq \frac{|\vbrack{w_i^{(t)},\M_j}|(1 + \eta (\widetilde{\Psi}_j^{(t)} - \lambda)(1 - \frac{\Xi_2^2}{\sqrt{d}}) )}{|\vbrack{w_{i'}^{(t)},\M_j}|(1 + \eta( \widetilde{\Psi}_j^{(t)} - \lambda)(1 + \frac{\Xi_2}{\sqrt{d}}) )} \geq (1 - \frac{\Xi_2}{t\sqrt{d}}) \frac{|\vbrack{w_i^{(t)},\M_j}|}{|\vbrack{w_{i'}^{(t)},\M_j}|}\\
            &\geq \prod_{t'=T'_j}^{t-1} (1 - O(\frac{\Xi_2^2}{t'\sqrt{d}}))\frac{|\vbrack{w_i^{(T'_j)},\M_j}|}{|\vbrack{w_{i'}^{(T'_j)},\M_j}|} \geq \Omega(1)
        \end{align*}
        where in the last inequality we have used our induction hypotheis at \(T'_j\).
        \item The proof for iterations \(t \in [\frac{d^{1.49}}{\eta}, \frac{d^{1.99}}{\eta}]\) is largely similar to the above. The only difference here is that we relies on a slightly different comparison here: Indeed, we have 
        \begin{align*}
            \frac{|\vbrack{w_i^{(t+1)},\M_j}|}{|\vbrack{w_{i'}^{(t+1)},\M_j}|} = \frac{|\vbrack{w_i^{(t)},\M_j}|(1 +\eta\widetilde{\Psi}_{j}^{(t)} - \eta \lambda) \pm O(\frac{\Xi_2}{d^2})\|w_i^{(t)}\|_2 }{|\vbrack{w_{i'}^{(t)},\M_j}|(1 +\eta\widetilde{\Psi}_{j}^{(t)} - \eta \lambda) \pm O(\frac{\Xi_2}{d^2})\|w_{i'}^{(t)}\|_2}
        \end{align*}
        Here we can use similar techniques as above to require \(|\widetilde{\Psi}_{j}^{(t)} - \lambda| \leq \frac{\Xi_2}{t\eta}\). Now the we also have 
        \begin{align*}
            \frac{|\vbrack{w_i^{(t+1)},\M_j}|}{|\vbrack{w_{i'}^{(t+1)},\M_j}|} &\geq \frac{|\vbrack{w_i^{(t)},\M_j}|(1 + \eta (\widetilde{\Psi}_j^{(t)} - \lambda)(1 - \frac{\Xi_2^2}{\sqrt{d}}) )}{|\vbrack{w_{i'}^{(t)},\M_j}|(1 + \eta (\widetilde{\Psi}_j^{(t)} - \lambda)(1 + \frac{\Xi_2^2}{\sqrt{d}}) )} \geq (1 - \frac{\Xi_2^2}{t\sqrt{d}}) \frac{|\vbrack{w_i^{(t)},\M_j}|}{|\vbrack{w_{i'}^{(t)},\M_j}|}\\
            &\geq \prod_{t'= d^{1.49}/\eta}^{t-1} (1 - \frac{\Xi_2^2}{t'd^{0.01}})\frac{|\vbrack{w_i^{(d^{1.49}/\eta)},\M_j}|}{|\vbrack{w_{i'}^{(d^{1.49}/\eta)},\M_j}|} \geq \Omega(1)
        \end{align*}
    \end{itemize}
    Now (4) are proven. (7) is an immediate result of our update scheme.
\end{proof}

\begin{definition}[optimal learner]\label{def:optimal}
    We define a learner network that we deem as the ``optimal'' feature map for this task. Let \(\kappa > 0\), we define \(\theta^{\star} := \{\theta^{\star}_i\}_{i\in[m]}\) as follows: \( \theta_i^{\star} = \frac{\sqrt{\tau}\kappa}{|\Mcal_j^{\star}|}\cdot\M_j\cdot \sign(\vbrack{w_i^{(T_2)},\M_j})\) if \(i \in \Mcal_j^{\star}\), and \(\theta^{\star}_i = 0\) if \(i \notin \cup_{j\in[d]}\Mcal_j^{\star}\). Furthermore, we define the optimal feature map \(f_{t}^{\star}\) as follows: for \(i\in[m]\), the \(i\)-th neuron of \(f_{t,\theta}\) given weight \(\theta_i \in \R^{d_1}\) is
    \begin{align*}
        f_{t,\theta,i}(x) = \left( \vbrack{\theta_i,x} - b_i \right)\1_{\vbrack{w_i^{(t)},x}\geq b_i} - \left( -\vbrack{\theta_i,x} - b_i \right)\1_{-\vbrack{w_i^{(t)},x}\geq b_i}
    \end{align*}
    and \(f_{t,\theta}\) is just \(f_{t,\theta}(\cdot) = (f_{t,\theta_1}(\cdot),\dots,f_{t,\theta,m}(\cdot) )^{\top}\).
\end{definition}

Now in order to obtain the loss convergence result in \myref{thm:convergence}{Theorem}, we need the following lemma, which characterize the how well the optimal learner perform evaluated by a pseudo objective.

\begin{lemma}[optimality]\label{lem:optimal}
    Let \(\{\theta_i^{\star}\}_{i\in[m]}\) and \(f_{t,\theta}\) be defined as in \myref{def:optimal}{Definition}, when \myref{induct-3}{Induction Hypothesis} holds, defining a pseudo loss function by 
    \begin{align*}
        \widetilde{L}(f_{t,\theta^{\star}},f_t) := \E\left[ - \tau \log\left( \frac{e^{\vbrack{f_{t,\theta}(x_p^+),f_t(x_p^{++})}/\tau}}{\sum_{x \in \Bfr}e^{\vbrack{f_{t,\theta}(x_p^+),f_t(x)}/\tau}}\right) \right]
    \end{align*}
    then by choosing \(\kappa = \Theta(\Xi_2)\), and suppose \(\sum_{i\in\Mcal^{\star}_j}|\vbrack{w_i^{(t)},\M_j}|\geq \Omega(\frac{\sqrt{\tau}}{\Xi_2})\), we have loss guarantee:
    \begin{displaymath}
        \widetilde{L}(f_{t,\theta^{\star}},f_t)  \leq O\left(\frac{1}{\log d}\right)
    \end{displaymath} 
\end{lemma}

\begin{proof}
    The proof of (a) follows from the fact that \myref{induct-3}{Induction Hypothesis} holds at iteration \(t\geq T_3\), and also from \myref{lem:activation-size}{Lemma} and \myref{lem:activation-size-3}{Lemma}. (b) can be proven via the following calculations:
    \begin{align*}
        \widetilde{L}(f_{t,\theta^{\star}},f_t) = \E\left[- \tau \log\left( \frac{e^{\vbrack{f_{t,\theta^{\star}}(x_p^+),f_t(x_p^{++})}/\tau}}{\sum_{x \in \Bfr}e^{\vbrack{f_{t,\theta^{\star}}(x_p^+),f_t(x)}/\tau}}\right) \right]
    \end{align*}
    where the second inequality follows from the high probability bound (a). Now by using Bernoulli concentration, we know that whenever \(\sum_{j\in[d]}\1_{z_{p,j}\neq 0} = \Omega(\log\log d)\) (which happens with constant probability), we have 
    \begin{align*}
        \sum_{j\in[d]}\1_{z_{n,s,j}=z_{p,j}}\leq C\sum_{j\in[d]}\1_{z_{p,j}\neq 0} \tag{with prob \(\geq  1- \frac{1}{d^{\Omega(\log\log d)}}\) for all \(x_{n,s} \in \Nfr\)}
    \end{align*}
    And also from \myref{def:optimal}{Definition} we know that if for some \(j\in[d]\), \(z_{p,j} = z_{n,s,j}\), then
    \begin{align*}
        \sum_{i\in\Mcal_j^{\star}} \left(f_{t,\theta^{\star},i}(x_p^+)h_{i,t}(x_{n,s}) - f_{t,\theta^{\star},i}(x_p^+)h_{i,t}(x_p^{++}) \right) \geq \kappa \tau \sum_{j\in\Mcal_j^{\star}}|\vbrack{w_i^{(t)},\M_j}| + O(\frac{1}{\log d})
    \end{align*}
    which can be obtained by similar calculations in \myref{lem:signal-lowerbound}{Lemma}. Noticing that the event \(z_p \neq 0\) happens with prob \(\geq 1 - \frac{1}{\polylog(d)}\), we have 
    \begin{align*}
        &\quad \,\widetilde{L}(f_{t,\theta^{\star}},f_t) \\
        &  \leq (1 - \frac{1}{\polylog(d)})\E\left[\log\Bigg(\sum_{x \in \Bfr}e^{\vbrack{f_{t,\theta^{\star}}(x_p^+),f_{t}(x)}/\tau - \vbrack{f_{t,\theta^{\star}}(x_p^+),f_{t}(x_{p}^{++})}/\tau}\Bigg)\,\Big|\, z_p\neq 0 \right] \\
        &\quad + \Pr(z_p = 0)\cdot O(\log |\Bfr|) \\
        & \leq (1 - \frac{1}{\polylog(d)})\E\left[\log\Bigg(\sum_{x \in \Bfr}e^{\sum_{j\in[d]}\sum_{i\in\Mcal_j^{\star}}(f_{t,\theta^{\star},i}(x_p^+)h_{i,t}(x) - f_{t,\theta^{\star},i}(x_p^+)h_{i,t}(x_p^{++}))/\tau} \Bigg)\,\Big|\, z_p\neq 0\right]\\
        &\quad + \Pr(z_p = 0)\cdot O(\log |\Bfr|) \\
        & = (1 - \frac{1}{\polylog(d)})\E\left[\log\Bigg(1 + \sum_{x_{n,s} \in \Nfr}e^{- \sum_{j\in\Mcal_j^{\star}}|\vbrack{w_i^{(t)},\M_j}|(z_{p,j}^2 - z_{p,j}z_{n,s,j}) + O(\frac{1}{\log d}))}\Bigg)\,\Big|\, z_p\neq 0\right] + O(\frac{1}{\log d})\\
        &\leq O(\frac{1}{\log d}) 
    \end{align*}
    where the last inequality combines the Bernoilli concentration results of \(\sum_{j\in[d]}z_{p,j}z_{n,s,j}\) and a union bound for all \(s\in[\Nfr]\), and that \(|\vbrack{w_i^{(t)},\M_j}| \geq \Omega(\frac{\sqrt{\tau}}{\Xi_2})\).
\end{proof}

\begin{proof}[Proof of \myref{thm:convergence}{Theorem}] 
    Due to \myref{induct-3}{Induction Hypothesis}, we know that as long as training goes on, the neural network will learn the desired features with sparse representations. As a complement of the conditions in \myref{induct-3}{Induction Hypothesis}, we notice that at some \(t \geq \Theta(\frac{d\log d}{\eta \log\log d})\), we have for all \(j \in [d]\), \(\F_j^{(t)} \geq \Omega(\tau \log d)\), using \myref{lem:signal-lowerbound}{Lemma}. This can be combined with \myref{induct-3}{Induction Hypothesis} (1) and (4) to show (a) \myref{thm:convergence}{Theorem}. Now we prove that it actually converge to the desired solutions, rather than bouncing around. Denote \(w^{(t)} = (w_1^{(t)},\dots,w_m^{(t)})\), since our update is \(w^{(t+1)} = w^{(t)} - \nabla_w\Obj(f_t) + \frac{1}{\poly(d_1)}\), we have 
    \begin{align*}
        \eta \vbrack{\nabla_w\Obj(f_t),w^{(t)} - \theta^{\star}} &= \frac{\eta^2}{2}\|\nabla\Obj(f_t)\|_F^2 + \frac{1}{2}\|w^{(t)} - \theta^{\star}\|_F^2 - \frac{1}{2}\|w^{(t+1)} - \theta^{\star}\|_F^2 + \frac{\eta^2}{\poly(d_1)}\\
        &\leq \eta^2 \poly(d)  + \frac{1}{2}\|w^{(t)} - \theta^{\star}\|_F^2 - \frac{1}{2}\|w^{(t+1)} - \theta^{\star}\|_F^2 + \frac{\eta^2}{\poly(d_1)}
    \end{align*}
    Now we will use the tools from online learning to obtain a loss guarantee: define a pseudo objective for parameter \(\theta\)
    \begin{displaymath}
        \widetilde{\Obj}_t(\theta) := \widetilde{L}(f_{t,\theta},f_t) + \frac{\lambda}{2}\sum_{i\in[m]}\|\theta_i\|_2^2 = \E\left[ - \tau \log\left( \frac{e^{\vbrack{f_{t,\theta}(x_p^+),f_t(x_p^{++})}/\tau}}{\sum_{x \in \Bfr}e^{\vbrack{f_{t,\theta}(x_p^+),f_t(x)}/\tau}}\right) \right] + \frac{\lambda}{2}\sum_{i\in[m]}\|\theta_i\|_2^2
    \end{displaymath}
    Which is a convex function over \(\theta\) since it is linear in \(\theta\). Moreover, we have \(\widetilde{\Obj}_t(w^{(t)}) = \Obj(f_t)\) and \(\nabla_{\theta_i}\widetilde{\Obj}_t(w_i^{(t)}) = \nabla_{w_i}\Obj(f_t)\), thus we have 
    \begin{align*}
        \eta \vbrack{\nabla_w\Obj(f_t),w^{(t)} - \theta^{\star}} &= \eta \vbrack{\nabla_{\theta}\widetilde{\Obj}_t(w^{(t)}),w^{(t)} - \theta^{\star}}\\
        &\geq \widetilde{\Obj}_t(w^{(t)}) - \widetilde{\Obj}_t(\theta^{\star})\\
        &\geq  \widetilde{\Obj}_t(w^{(t)}) - \E\left[ - \tau \log\left( \frac{e^{\vbrack{f_{t,\theta^{\star}}(x_p^+),f_t(x_p^{++})}/\tau}}{\sum_{x \in \Bfr}e^{\vbrack{f_{t,\theta^{\star}}(x_p^+),f_t(x)}/\tau}}\right) \right] + \frac{\lambda}{2}\sum_{i\in[m]}\|\theta_i^{\star}\|_2^2\\
        &\geq \widetilde{\Obj}_t(w^{(t)}) - O(\frac{1}{\log d}) - O\left(\sum_{i\in[m]}O(\lambda \|\theta_i^{\star}\|_2^2) + O(\frac{1}{d^{0.49}})\right) - \frac{O(\log d)}{\polylog(d)}\\ 
        & \geq \Obj(f_t) - O(\frac{1}{\log d})
    \end{align*}
    Now choosing \(\kappa = \Theta(\Xi_2)\leq \lambda/d\), and by a telescoping summation, we have 
    \begin{align*}
        \frac{1}{T}\sum_{t = T_3}^{T_3 + T - 1}\left(\Obj(f_t) - O\left(\frac{1}{\log d}\right)\right) \leq \frac{O(\|w^{(T_3+T)} - \theta^{\star}\|_F^2)}{T\eta} \leq O\left(\frac{m\Xi_2}{T\eta}\right)
    \end{align*}
    Since \(T\eta \geq m\Xi_2^{10}\), this proves the claim.
\end{proof}

The corollary in the main text can be proven via simple application of linear regression analysis. Because with high probability over polynomially many data \(\mathcal{Z}_{\sup} = \{x_i, y_i\}_{i\in [n_{sup}]}\) independently generated according to our definition \myref{def:sparse-coding}{Definition} and \myref{def:downstream-task}{Definition}, \(\{x_i\}_{i \in [n_{sup}]}\) form separable clusters w.r.t. their differences in the latent variables \(z_i\), which dictate their labels.

%\begin{proof}[Proof of \myref{thm:convergence}{Theorem}]
    %Due to \myref{induct-3}{Induction Hypothesis}, we know that as long as training goes on, the neural network will learn the desired features with sparse representations. Here we prove that it actually converge to the desired solutions, rather than bouncing around. Denote \(w^{(t)} = (w_1^{(t)},\dots,w_m^{(t)})\), since our update is \(w^{(t+1)} = w^{(t)} - \nabla_w\Obj(f_t) + \frac{1}{\poly(d_1)}\), we can first calculate the spectral norm of the Hessian \(\nabla^2_w\Obj(f_t)\) as follows:
    %\begin{align*}
        %\|\nabla_{w_i}^2\Obj(f_t)\|_2 &\leq \left\| \E\left[\1_{|\vbrack{w_i^{(t)},x_p^+}|\geq b_i^{(t)}}\1_{|\vbrack{w_i^{(t)},x_p^+}|\geq b_i^{(t)}}x_px_p^{\top}\right] \right\|_2 + \left\| \E\left[ \nabla_{w_i}^2 \tau\log \left[\sum_{x_{n,s} \in \Nfr}e^{\Sim_{f_t}(x_p^+,x_{n,s})/\tau} \right]\right] \right\|_2\\
        %& \leq \E[\|x_p\|_2^2] + \E\left[\max_{x_{n,s} \in \Nfr}\frac{1}{\tau^2}|\vbrack{x_{n,s},x_p}|\right] \\
        %& \leq \poly(d)
    %\end{align*}
    %So by using \(\eta \leq \frac{1}{\poly(d)}\), we have 
    %\begin{align*}
        %\Obj(f_{t+1}) \leq \Obj
    %\end{align*}
%\end{proof}

\section{Results for Learning Without Augmentations}

In this section we will sketch the proof when no augmentation is applied to the input data. Indeed, the analysis is similar but much easier compared to the case when augmentations are used. We present the first lemma below. 

\begin{lemma}[gradient for features, positive]\label{lem:gd-pos-sparse-NA}
    Let \(i\in[m]\) and \(v \in \{\M_j\}_{j\in[d]}\cup\{\Mperp_j\}_{j\in[d_1]\setminus [d]}\), when bias \(b_i = 0\) we have 
    \begin{align*}
        \E\left[h_{i}(x_p)\1_{|\vbrack{w_i,x_p}|\geq b_i}\vbrack{x_p,\M_j}\right] = \vbrack{w_i,\M_j}\E\left[\vbrack{x_p,v}^2\right]
    \end{align*}
    when \(b_i > 0\), we also have 
    \begin{align*}
        \E\left[h_{i}(x_p)\1_{|\vbrack{w_i,x_p}|\geq b_i}\vbrack{x_p,\M_j}\right] = \vbrack{w_i,\M_j}\E\left[\vbrack{x_p,v}^2\1_{|\vbrack{w_i,x_p}|\geq b_i}\right]
    \end{align*}
\end{lemma}

\begin{proof}
    The proof is essentially trivial by following the approach in the proof of \myref{lem:grad-positive-1}{Lemma} and notice that no compensation term for augmentation is needed here.
\end{proof}

Indeed, since \(\E[\vbrack{x_p,\M_j}^2] = (1+\frac{\log\log d}{\sqrt{\log d}})\E[\vbrack{x_p,\Mperp_j}^2] \), for any training phase before close to convergence, the difference does not matter since the growth rate of each feature is exponential, i.e. \(\vbrack{w_i^{(t+1)},v} = \vbrack{w_i^{(t)},v}(1 + \E[\vbrack{x_p,v}^2])\). Indeed, denoting the neural network trained without augmentation by \(f_t^{\mathsf{NA}}\), setting bias \(b_i^{(t)} = 0\), we can have a simple lemma:

\begin{lemma}[The superiority of dense feature without augmentations]
    For each \(i \in [m]\), let \(w_i\) has norm \(\|w_i\|_2^2\) and orthogonal to each other, and such that \(w_i \perp \M_j\) for all \(j \in [d]\), then
    \begin{align*}
        L(f_t^{\mathsf{NA}}) &\leq \widetilde{O}\left(\frac{d}{\sum_{i\in[m]}\|w_i^{(t)}\|_2^2}\right)
    \end{align*}
\end{lemma}

\begin{proof}
    By Johnson-Lindenstrauss lemma, with high probability we have 
    \begin{align*}
        \vbrack{f_t^{\mathsf{NA}}(x_p),f_t^{\mathsf{NA}}(x_p)} &= \|f_t^{\mathsf{NA}}(x_p)\|_2^2 = \Omega (\sigma_{\xi}^2\sum_{i\in[m]}\|w_i^{(t)}\|_2^2)\\
        \vbrack{f_t^{\mathsf{NA}}(x_p),f_t^{\mathsf{NA}}(x_{n,s})} &= \|f_t^{\mathsf{NA}}(x_p)\|_2^2 = \widetilde{O} (\frac{\sigma_{\xi}^2}{\sqrt{m}}\sum_{i\in[m]}\|w_i^{(t)}\|_2^2)
    \end{align*}
    This leads to:
    \begin{align*}
        L(f_t^{\mathsf{NA}}) = &\E\left[ \tau \log\left(1 + \sum_{x_{n,s} \in \Nfr}e^{\vbrack{f_t^{\mathsf{NA}}(x_p),f_t^{\mathsf{NA}}(x_{n,s})}/\tau -\vbrack{f_t^{\mathsf{NA}}(x_p),f_t^{\mathsf{NA}}(x_p)}/\tau}\right) \right] \\
        &\leq \widetilde{O}(1)\E\left[ \left(\|f_t^{\mathsf{NA}}(x_p)\|_2^2 - \max_{x_{n,s}\in\Nfr} \vbrack{f_t^{\mathsf{NA}}(x_p),f_t^{\mathsf{NA}}(x_{n,s})} \right)^{-1} \right]\\
        &\leq \widetilde{O}\left(\frac{d}{\sum_{i\in[m]}\|w_i^{(t)}\|_2^2}\right) \tag{using above calculations}
    \end{align*}
    which is the claimed result.
\end{proof}
 
Specifically, if one use the full \(\M\) as the feature and raise bias above \(\Omega(\|w_i\|_2)\) in this setting, one would not have this superior loss property.

Now we only need to prove a norm result for all the neurons \(i \in [m]\), this can be done by similar analysis as in the proof of \myref{induct-1}{Induction Hypothesis}, which we skip here:

\begin{lemma}
    For each \(i \in [m]\), for some \(t \geq \widetilde{\Omega}(\frac{d}{\eta})\), we have \(\sum_{i\in[m]}\|w_i^{(t)}\|_2^2 = \Omega(\tau d \polylog (d))\).
\end{lemma}

Combining the results above, we can obtain the learning process of contrastive learning without data augmentations, in the presence of a large dense signal in the data. It is easy to see that the representations trained by this method has the following properties:

\begin{fact}
    At each iteration \(t\), the learned network without augmentations \(f_t^{\mathsf{NA}}\) satisfies for each \(i \in [m]\), we have \(\|\Mperp(\Mperp)w_i^{(t)}\|_2 = (1 - \frac{1}{\poly(d)})\|w_i^{(t)}\|_2\)
\end{fact}

\begin{proof}
    Similar to the proof of \myref{induct-1}{Induction Hypothesis}, which we skip.
\end{proof}

This fact directly leads to the final result that with high probability over \(x \sim \D_x\)
\begin{align*} 
    \left\vbrack{\frac{f_t^{\mathsf{NA}}(x)}{\|f_t^{\mathsf{NA}}(x)\|_2},\frac{f_t^{\mathsf{NA}}(\xi)}{\|f_t^{\mathsf{NA}}(\xi)\|_2} \right} \geq 1 - \widetilde{O}\left(\frac{1}{\poly(d)}\right)
\end{align*} 
And trivially, one cannot perform any linear regression or classification over such feature map \(f_t^{\mathsf{NA}}\), to obtain meaningful accuracy in our downstream tasks over the sparse coding data \(\D_x\).

\bibliographystyle{plainnat}
\bibliography{contrastive}

\end{document}